\documentclass[journal]{IEEEtran}
\ifCLASSINFOpdf
\else
\fi
\usepackage{array}

\usepackage{url}

\usepackage{amssymb}
\usepackage[english]{babel}
\usepackage{gensymb} 
\usepackage{subfig} 

\usepackage[dvips]{graphicx}%
\usepackage{amsfonts}
\usepackage{amsmath}
\usepackage{amsthm}
\newtheorem{proposition}{Proposition}
\usepackage{cite} 
\newcommand{\off}[1]{}

\def\IR{\relax{\rm I\kern-.18em R}}
\def\p{\partial}

\DeclareMathOperator{\Div}{div}

\graphicspath{figs/}

\begin{document}
%
\title{Separation Surfaces in the Spectral TV Domain for Texture Decomposition}
%
%
%

\author{Dikla~Horesh 
        and~Guy~Gilboa,~\IEEEmembership{Member,~IEEE}
\thanks{D. Horesh and G. Gilboa are with the Department of Electrical Engineering, Technion -– Israel Institute of Technology,
Haifa 32000, Israel e-mail: (dikla@campus.technion.ac.il , guy.gilboa@ee.technion.ac.il ).}
}

%
%

\markboth{IEEE TRANSACTIONS ON IMAGE PROCESSING,~Vol.~xx, No.~x, September~2016}%
{Horesh \MakeLowercase{\textit{et al.}}: Separation Surfaces for Texture Decomposition}
%



\maketitle

\begin{abstract}
In this paper we introduce a novel notion of separation surfaces for image decomposition.
A surface is embedded in the spectral total-variation (TV) three dimensional domain and encodes a spatially-varying separation scale.
The method allows good separation of textures with gradually varying pattern-size, pattern-contrast or illumination. 
The recently proposed total variation spectral framework is used to decompose the image into a continuum of textural scales.
A desired texture, within a scale range, is found by fitting a surface to the local maximal responses in the spectral domain. A band above and below the surface, referred to as the \textit{Texture Stratum}, defines for each pixel the adaptive scale-range of the texture. 
Based on the decomposition an application is proposed which can attenuate or enhance textures in the image in a very natural and visually convincing manner.
\end{abstract}


%
\IEEEpeerreviewmaketitle

\section{Introduction}
%
%
%
%
\IEEEPARstart{D}{ecomposing} an image into meaningful components is an important and challenging inverse problem in
image processing.
The general concept of structure-texture decomposition is that an image can be regarded as composed of a structural part, corresponding to the main large objects in the image, and a textural part, containing fine details, usually with some periodicity and oscillatory nature.
Image decomposition is a preprocessing stage, which can be essential for many image processing and computer vision tasks such as segmentation \cite{casaca2013spectral}, content based image retrieval \cite{singha2012content} , feature extraction and classification \cite{dua2012wavelet} and restoration and analysis of ancient documents \cite{AncientDocs2013}.
We first briefly recall the main approaches related to image decomposition (focusing on variational methods).

\subsection{Structure-texture and Multiscale Decomposition}
{\bf $\bf u+v$ Model.}
An image $f$ can be decomposed as $f = u+v$, where $u$ represents image cartoon or geometric (piecewise-smooth) component and $v$ represents the oscillatory or textured component of $f$.
This motivated \cite{Meyer[1]} to suggest the $TV-G$ variational model where the minimization yields $u$ with a low total-variation energy and $v$ with a low integral norm, referred to as a $G$-norm, which favors
oscillatory signals. Suggestions to implement Meyer's model were given in \cite{Luminita[1],Aujol[3]}. Many extensions and variations to the model with alternative norms adapted for textures were proposed, such as
 \cite{Luminita[2],agco06,MaurelAujolPeyre_texture2011,Nikolova,DuvalAujolVese_decomposition2009}. The use of nonconvex regularizers was recently proposed in \cite{Atto_nonconvex_decomp2015}.
Employing sparse representation methods for decomposition was first suggested in \cite{Starck_etal2005}.
Two simplistic ways of revealing the textural parts in images are still used for some computer vision tasks.
The most basic one is linear - using a smoothing kernel, such as a Gaussian, and subtracting the smoothed image from the input image.
Naturally edges and textures are mixed. A somewhat more reliable method is to apply edge-preserving denoising, such as bilateral filtering \cite{tom_man_98}, and subtract the result from the input image.


{\bf $\bf u+v+w$ Model.}
In \cite{AujolChambolle} a model decomposing an image into structure $u$, texture $v$, and noise $w$ was proposed using dual norms (negative Sobolev and Besov norms) for the texture and noise parts.
An analysis of the three-part decomposition can be seen in \cite{Gilles_Meyer2010} and a recent approach using non-linear PDE's for structure-texture-edge decomposition is described in \cite{Moreno_decomp2015}.


{\bf Multiscale Model.}
It was realized quite early on that several textures of different scales can appear in an image and should be decomposed separately.
Multiscale decomposition using several $TV-L^2$ (ROF \cite{rof92}) decompositions, was first suggested in \cite{Tadmor_vese}.
Some features interpreted as ``texture'' in a given scale can be interpreted as ``structure'' at a finer scale. As conventional ROF was used the separation was not optimal, mixing some structure in the texture.
Gilles \cite{gilles} combined Meyer's decomposition model \cite{Meyer[1]} with a Littlewood-Paley filter, to extract a certain class of textures in an image. While this works well for synthetic images, it is not ideal for some real world images.
Zhang et al. \cite{zhang2014rolling} proposed a new framework called Rolling Guidance filter. This technique consists of an iterated improved variant of the bilateral filter which is controlled by a larger support linear smoothing kernel. High quality results were shown in \cite{zhang2014rolling}.
We will compare our work also to this state-of-the-art technique and show its limitations, especially when
there are gradual changes in pattern size or contrast.

{\bf Continuous Model.}
The spectral TV decomposition, explained in details below, can be seen as a generalization to the continuum of
multiscale decomposition, with infinitesimal scale precision which can be related to the eigenvalue of the nonlinear eigenvalue problem induced by the regularizer (see details hereafter).
In this case the input image is an integral over all scales. In practice the scale (time) step is finite and a summation of quantized scales is performed.

%
%



\subsection{Contribution and paper outline}
In this paper we present two essential contributions. A new approach of scale-separation is introduced which is based on the TV transform \cite{Gilboa_spectv_SIAM_2014}.
Using a geological analogy, the texture is encoded by a stratum, with a surface as center-line, in the 3D spectral TV domain, see visualization in Fig. \ref{fig:surface}.
It is well adapted to the image and can cope with gradually changing textures with respect to many parameters such as size, contrast and illumination.
Having defined a general desired scale-range, the method is automatic. A second contribution is a texture processing approach which can enhance or attenuate textures in an
easy manner with very vivid and natural looking results.

The outline of the paper is as follows: in Section \ref{sec:pre} we describe the convex nonlinear eigenvalue problem and the spectral TV approach.
Section \ref{sec:multi} presents multiscale decomposition and orientation analysis based on spectral TV. This part
was first presented in a conference \cite{horesh2015multiscale}. We then proceed to the main contribution of the paper, Section \ref{sec:surf},
where a surface-based decomposition is introduced. In Section \ref{sec:app} a texture processing application is proposed,
illustrated by natural color image examples.

\begin{figure}
\begin{tabular}{cc}
\subfloat[Separation surface and stratum visualization]{\includegraphics[height=1.4in]{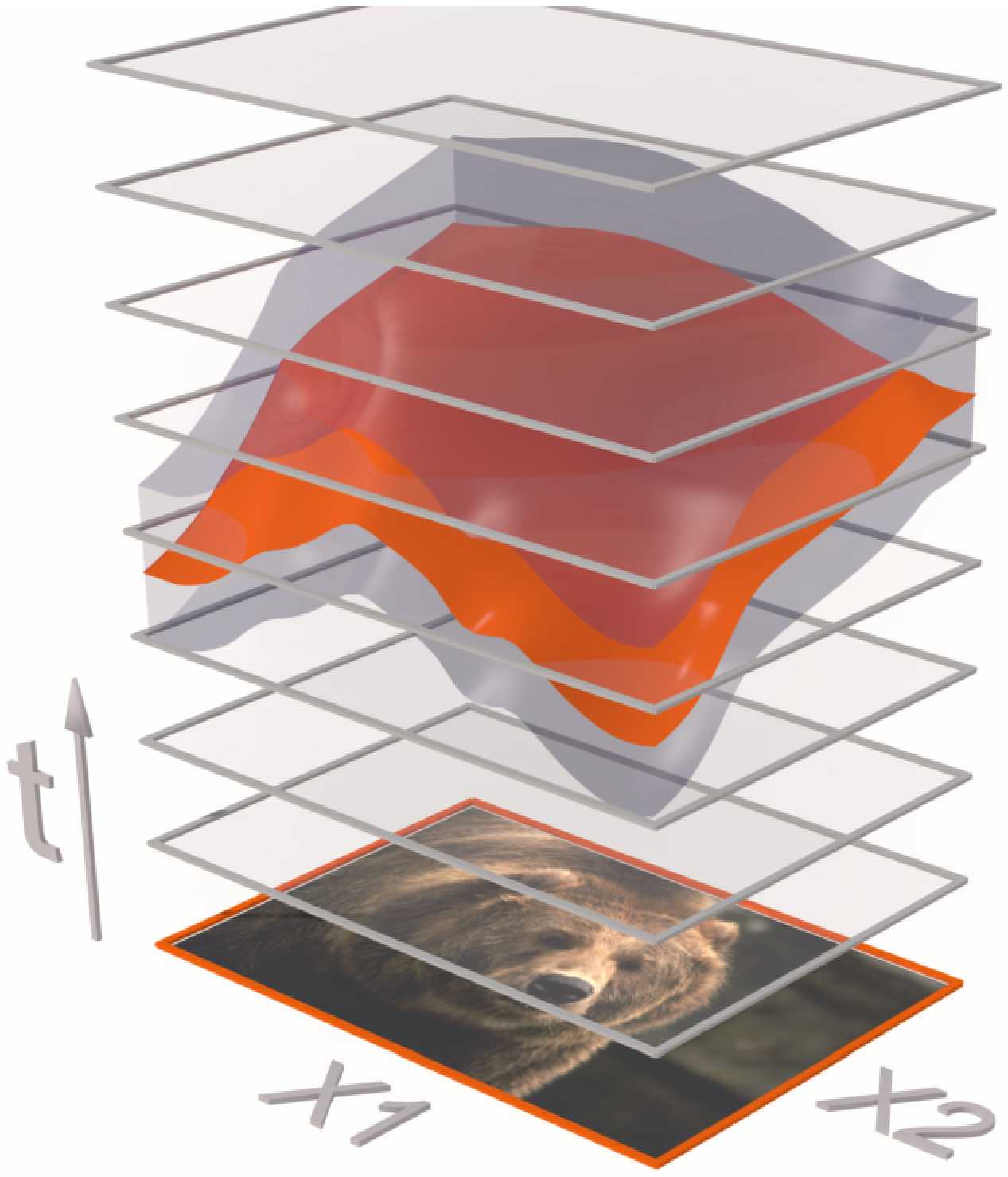}} &
\subfloat[Natural rock strata]{ \includegraphics[height=1.4in]{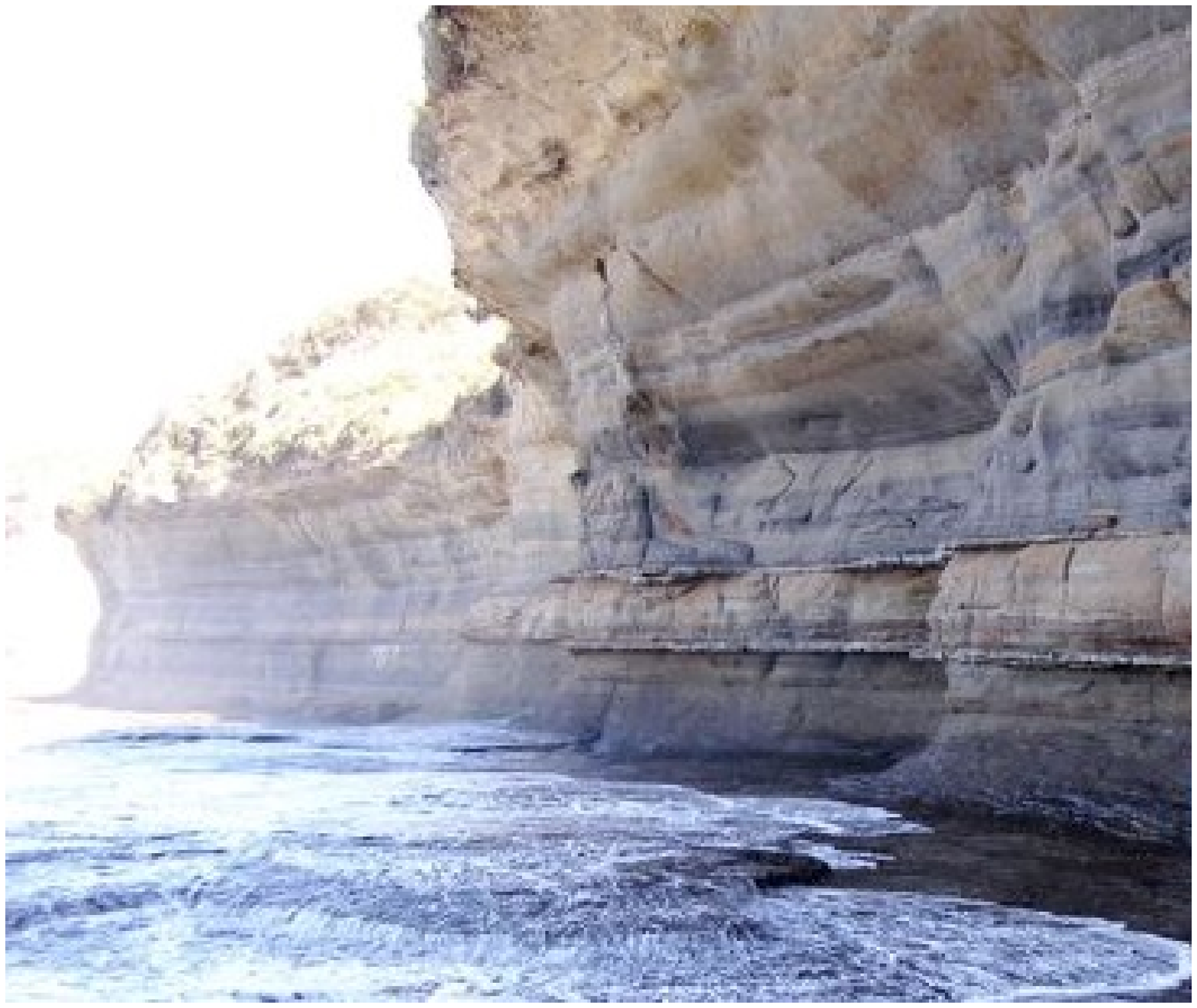}} \\
\end{tabular}
   \caption{Geological analogy of the texture encoded as a stratum in the spectral TV domain.}
   \label{fig:surface}
\end{figure}

\section{Preliminaries}
\label{sec:pre}
In this section we summarize the essential theory concerning non-linear eigenfunctions induced by convex functionals, the spectral TV framework and Gabor filters (used as scale-orientation descriptors).

\subsection{Nonlinear Eigenfunctions}
Classical linear eigenfunction analysis has shown to provide many state-of-the-art algorithms in
signal processing, computer vision and machine-learning. Some examples are segmentation \cite{shi_malik00}, clustering \cite{Ng_Spectral_clustering_2002},
subspace clustering \cite{Liu_subspace_LRR_2013} and dimensionality reduction \cite{Laplacian_eigenmaps_belkin_niyogi_2003}.
Eigenfunctions of an operator can be viewed as the operator's inherent atoms with an intrinsic scale represented by the respective eigenvalue.
Recent studies \cite{Benning_Burger_2013,Gilboa_spectv_SIAM_2014} indicate that a generalized theory can be developed for the convex nonlinear case.

Nonlinear eigenfunctions induced by a convex functional emerge by the following \emph{nonlinear eigenvalue problem}:
\begin{equation}
\label{eq:ef_problem}
 \lambda u \in \p J(u),
\end{equation}
where $J(u)$ is a convex functional and $\p J(u)$ is its subdifferential.
A function $u$ admitting Eq. \eqref{eq:ef_problem} is referred to as an eigenfunction with a corresponding eigenvalue $\lambda$.
We can briefly study the linear case, to get some intuition.

\subsubsection*{A Linear Example}
Let us examine the functional $$J(u) = \frac{1}{2}\int_{\Omega}|\nabla u(x)|^2 dx,$$
where $\nabla$ is the gradient. The convex functional induces an operator through its subgradient.
Here the subgradient (in this case single valued) is $p(u) = -\Delta u$ ($\Delta$ denotes the Laplacian).
The corresponding eigenvalue problem is
$$ -\Delta u = \lambda u .$$
In the one-dimensional case, with appropriate boundary conditions, functions of the form $u = \sin(\omega x)$
are eigenfunctions with corresponding eigenvalues $\lambda = \omega^2$.

Thus Fourier frequencies naturally emerge as solving an eigenvalue problem related to a quadratic smoothing convex functional.

In~\cite{Gilboa_SSVM_2013_SpecTV,Gilboa_spectv_SIAM_2014} an image decomposition and
filtering framework was suggested.
It presents a notion of generalized nonlinear eigenfunctions which are used to define forward and inverse TV transforms.
This can be used to decompose the image into well defined scales and allows a new variety of filtering methods.

\subsection{Spectral TV}
In~\cite{Gilboa_spectv_SIAM_2014} a non-conventional way of defining a transform through a
partial-differential-equation (PDE) is suggested, based on the total-variation (TV) functional:
\begin{equation}
\label{eq:tv}
J(u) = \int_\Omega |D u|,
\end{equation}
where $Du$ denotes the distributional gradient of $u$.
The corresponding gradient descent of the functional, known as total-variation flow \cite{tv_flow},
is formally written as:
\begin{equation}
\label{eq:tv_flow}
\begin{array}{ll}
\frac{\partial u}{\partial t}  = \Div \left(\frac{Du}{|Du|}\right), & \textrm{in } (0,\infty)\times \Omega  \\
\frac{\partial u}{\partial n}=0, & \textrm{on } (0,\infty)\times \partial\Omega  \\
u(0;x)=f(x), & \textrm{in } x \in \Omega,
\end{array}
\end{equation}
where $\Omega$ is the image domain (a bounded set in $\IR^N$ with Lipschitz continuous boundary
$\partial \Omega$).
%
The TV transform is defined by:
\begin{equation}
\label{eq:phi}
\phi(t;x) = u_{tt}(t;x)t,
\end{equation}
where $u_{tt}$ is the second time derivative of the solution $u(t;x)$ of \eqref{eq:tv_flow}.
The inverse transform is:
\begin{equation}
\label{eq:tv_recon}
f(x) = \int_0^\infty \phi(t;x) dt + \bar{f},
\end{equation}
where $\bar{f} = \frac{1}{\Omega}\int_\Omega f(x)dx$ is the mean value of the initial condition.
Filtering is performed using a transfer function $H(t)\in \IR$:
\begin{equation}
\label{eq:tv_filt}
f_H(x) := \int_0^\infty \phi(t;x)H(t) dt + \bar{f}.
\end{equation}
The spectrum $S^f(t)$ of the input signal $f(x)$ corresponds to the $L^1$ amplitude of each scale:
\begin{equation}
\label{eq:S}
S^f(t) = \|\phi(t;x)\|_{L^1} = \int_\Omega |\phi(t;x)|dx.
\end{equation}

Two significant results were shown in~\cite{Gilboa_spectv_SIAM_2014} for this transform:
\begin{itemize}
\item {\bf Eigenfunctions as Atoms:} Let $f(x)$ be a function which admits the nonlinear eigenvalue problem \eqref{eq:ef_problem}, ($f=u$), for the TV functional. Then the transform yields a measure (single impulse), multiplied by $f(x)$, at time $t=1/\lambda$ and is zero for all other $t$: $\phi(t;x)=\delta(t-1/\lambda)f(x)$, where $\delta(\cdot)$ is the Dirac delta function.
\item {\bf Relations to TV-flow:} The TV flow solution $u(t)$ is given by:
\begin{equation}
\label{eq:H_tvflow}
u(t) = \int_0^\infty H^t(\tau)\phi(\tau;x) d\tau + \bar{f};\,\,\,
\end{equation}
$$H^t(\tau) =
\left\{
\begin{array}{ll}
0,& 0 \le \tau < t \\
\frac{\tau-t}{\tau}, & t \le \tau < \infty
\end{array}
\right.
.$$
\end{itemize}

The first result relates to nonlinear spectral theory, which has attracted
increasing interest lately, see e.g.~\cite{Benning_Burger_2013,Nonlin_Lap_spectral_2012} and \cite{Bresson_Tai_Chan_Szlam_Cheeger_2014} in the segmentation and learning context.

The second result shows that the framework is a generalization of standard TV filters and that many other filters related to the functional
can be designed.

An example of different spectral components and of spectral image filtering can be seen in Fig.~\ref{fig:zebra_tex}. A  zebra image is shown with its spectrum in different colors to demonstrate the integration intervals of the $\phi$'s, using \eqref{eq:tv_filt} with $H=1$ in the desired interval and 0 otherwise, appearing respectively in the filtered images. The contrast is enhanced for better visualization.

\begin{figure}
\includegraphics[height=1.1in]{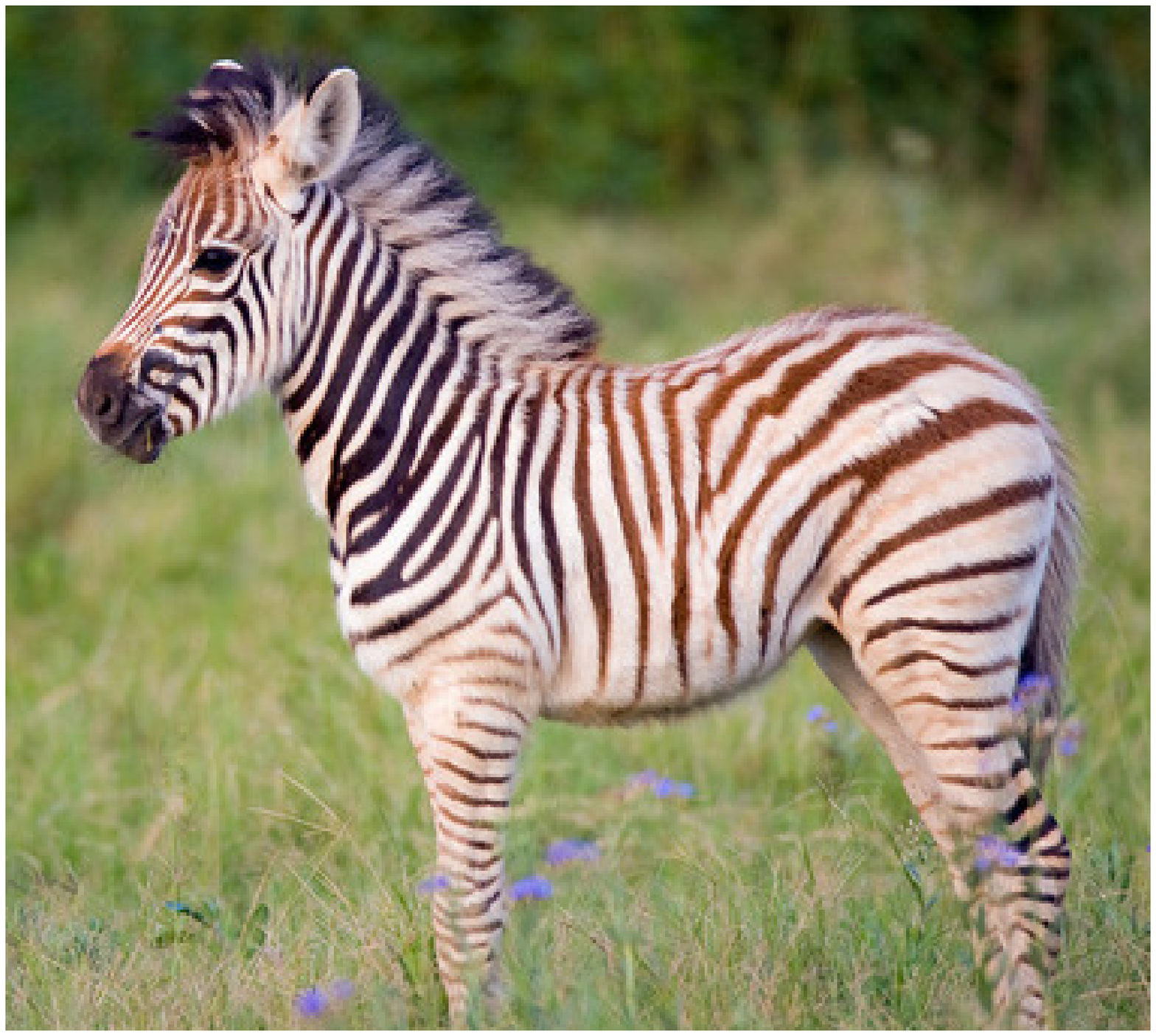}
\includegraphics[height=1.1in]{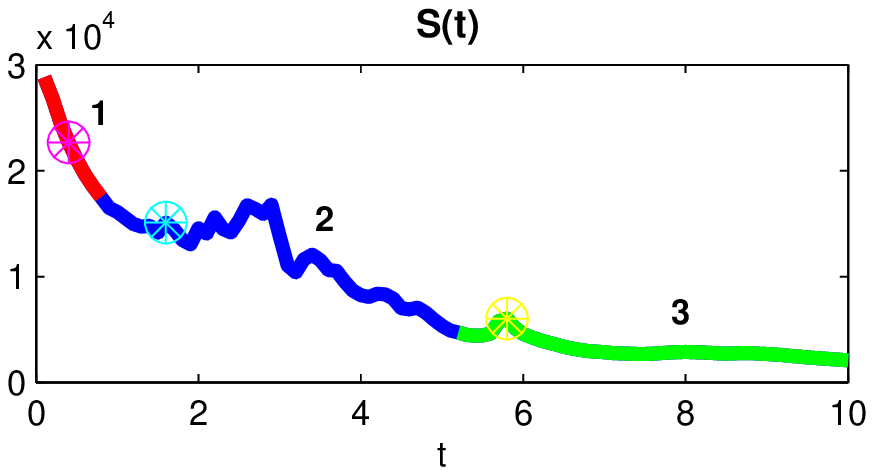}\\
\centering
\begin{tabular}{ccc}
\subfloat[High-pass]{\includegraphics[width=1in]{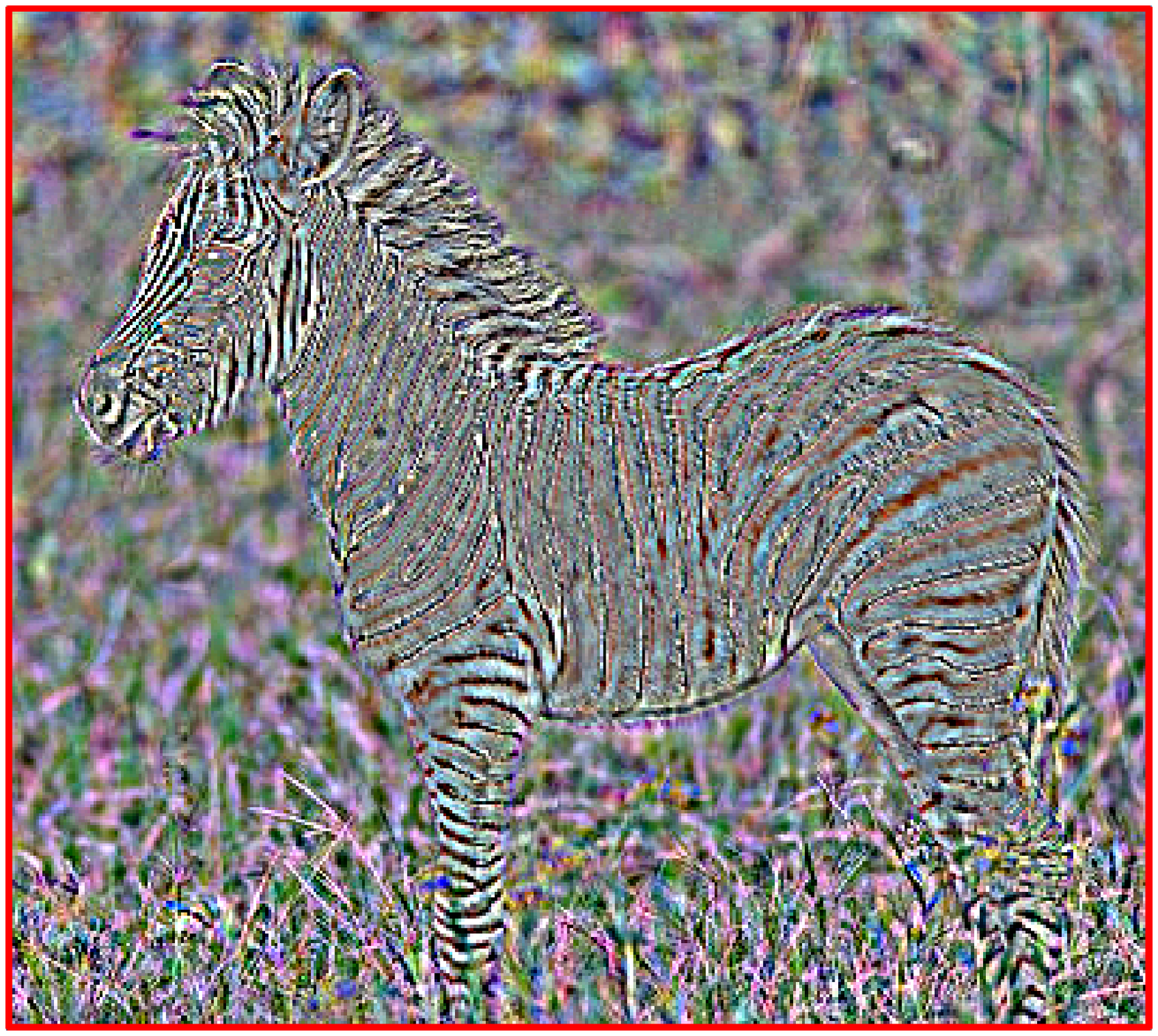}} &
\subfloat[Band-pass]{\includegraphics[width=1in]{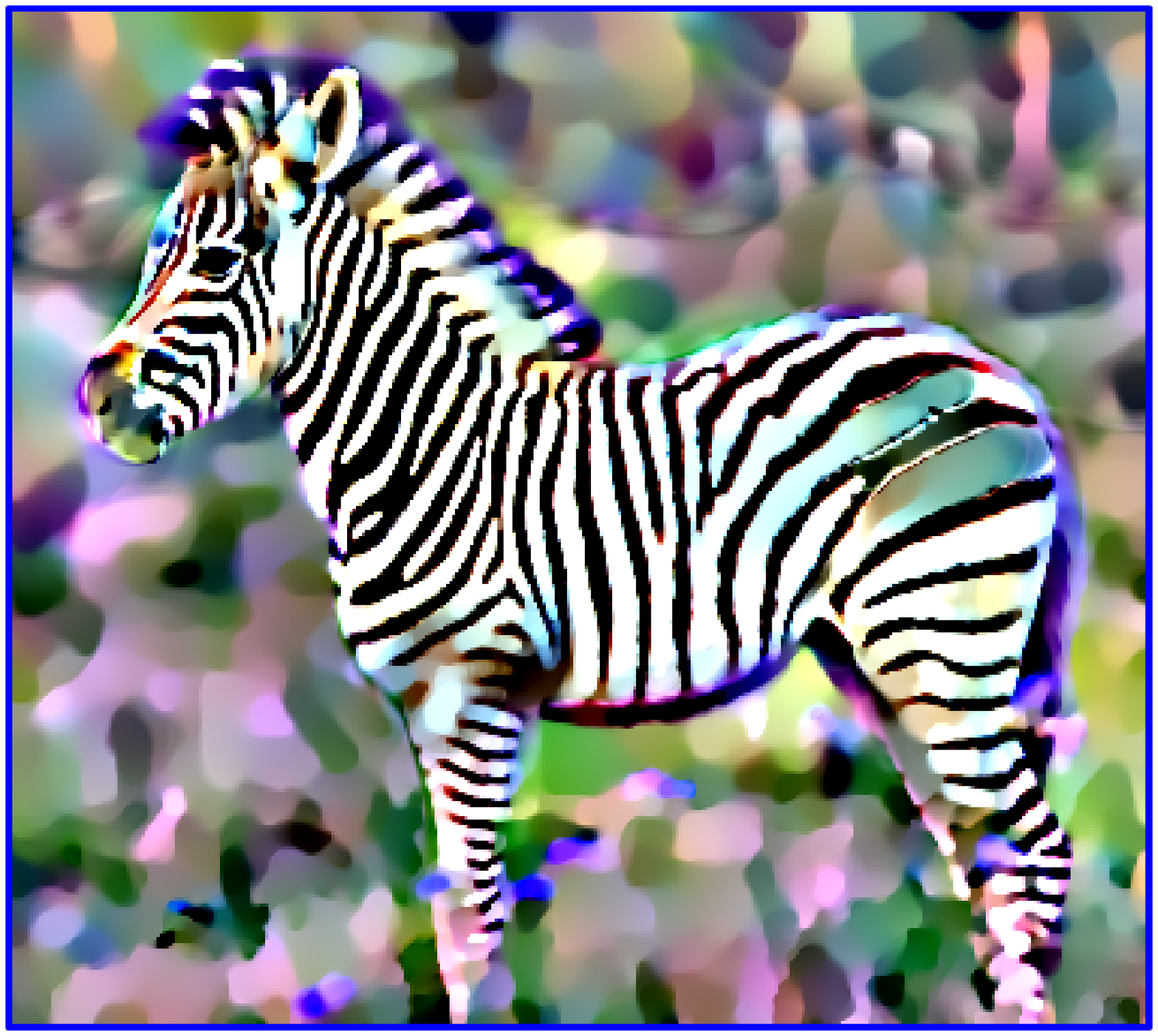}}  &
\subfloat[Low-pass]{\includegraphics[width=1in]{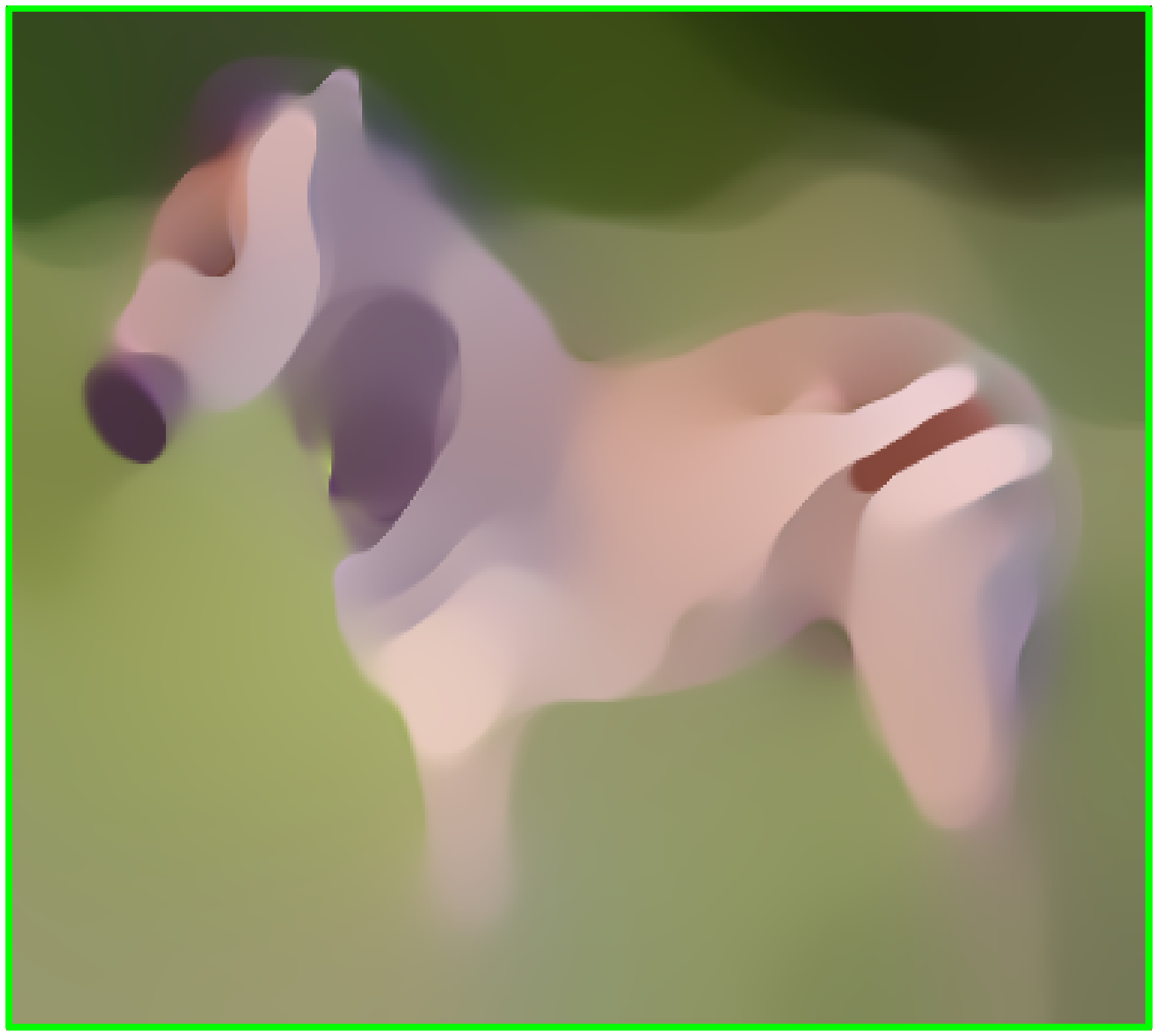}} \\
\subfloat[$t=0.4$]{\includegraphics[width=1in]{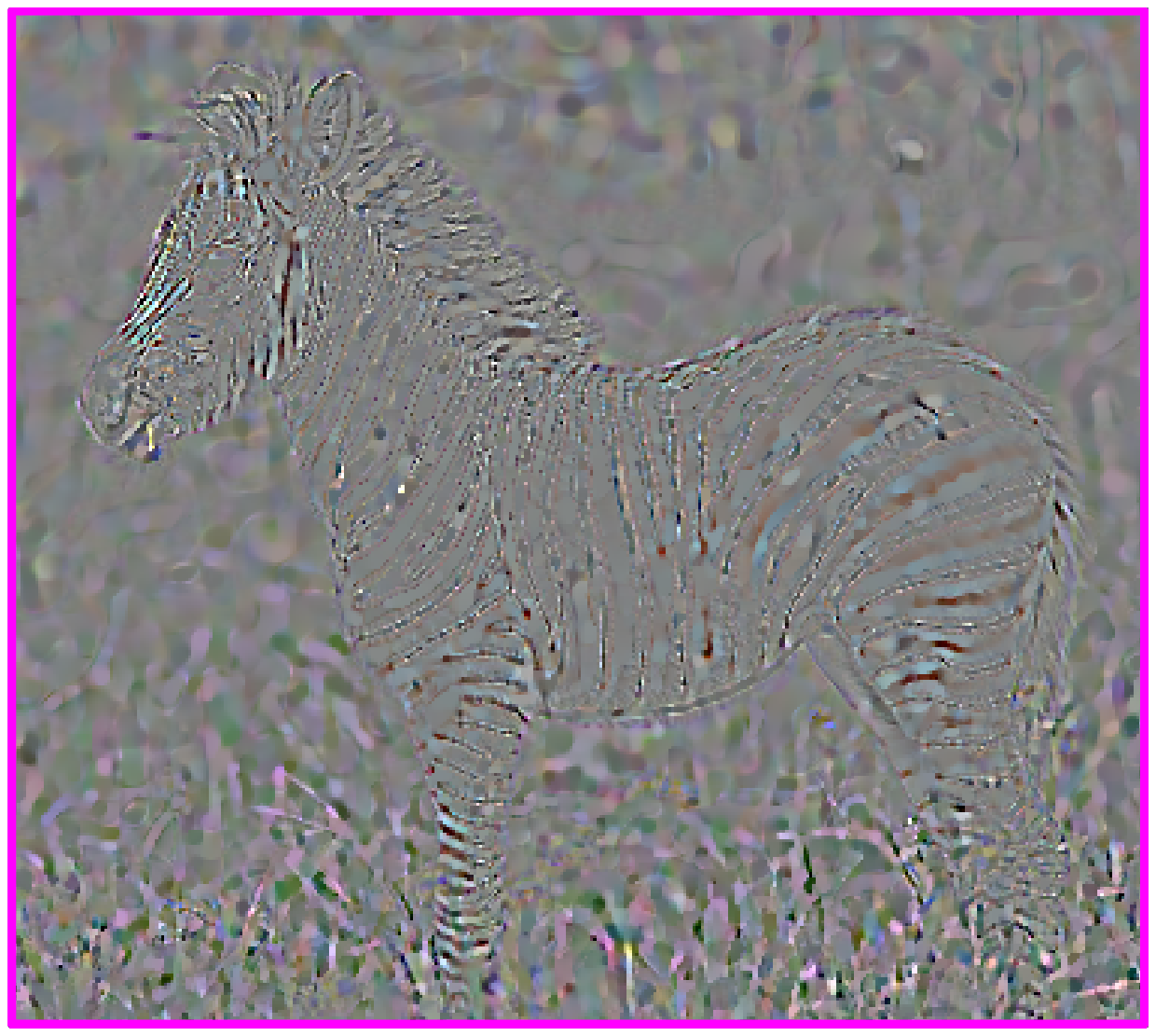}} &
\subfloat[$t=1.6$]{\includegraphics[width=1in]{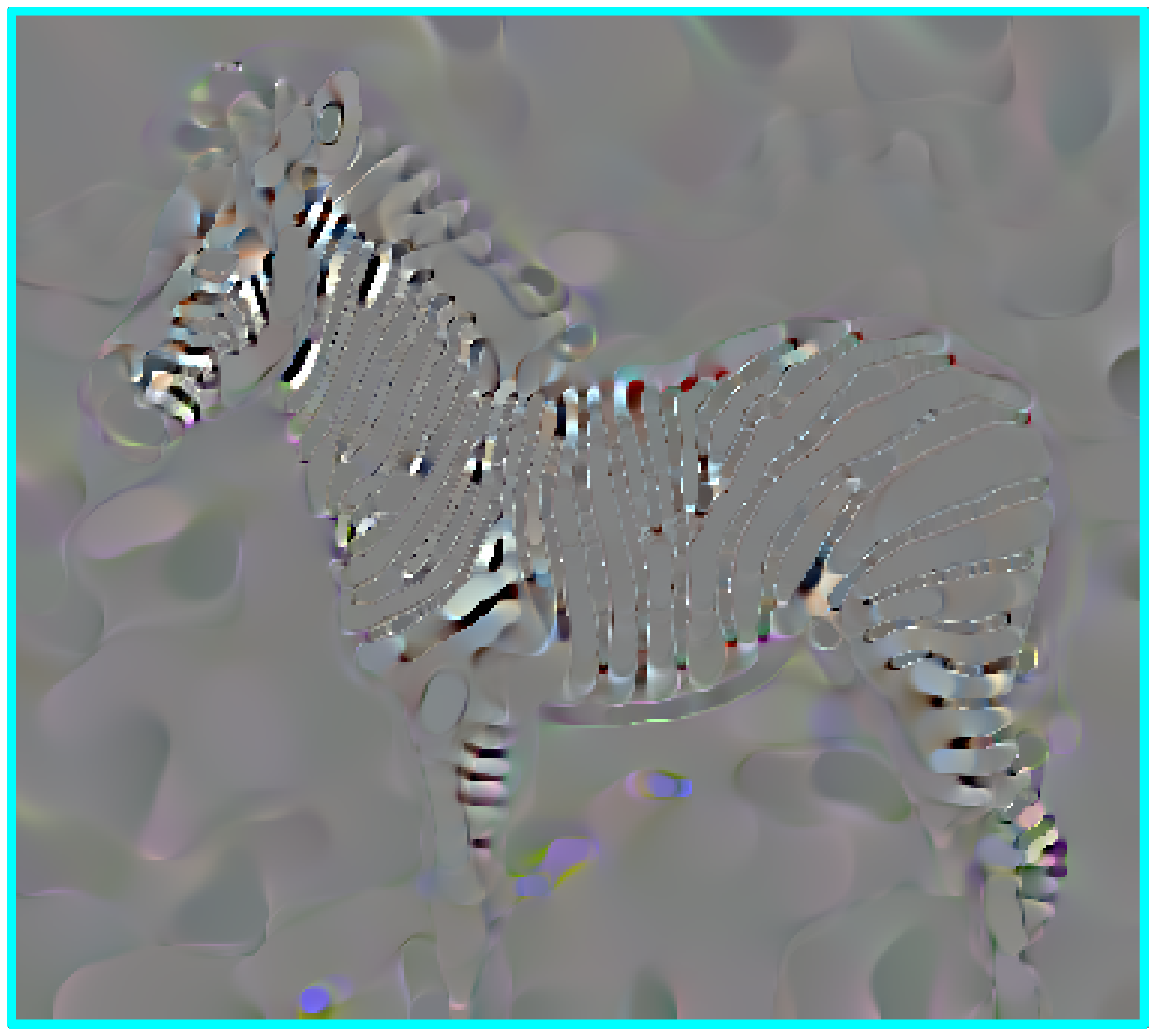}} &
\subfloat[$t=5.8$]{\includegraphics[width=1in]{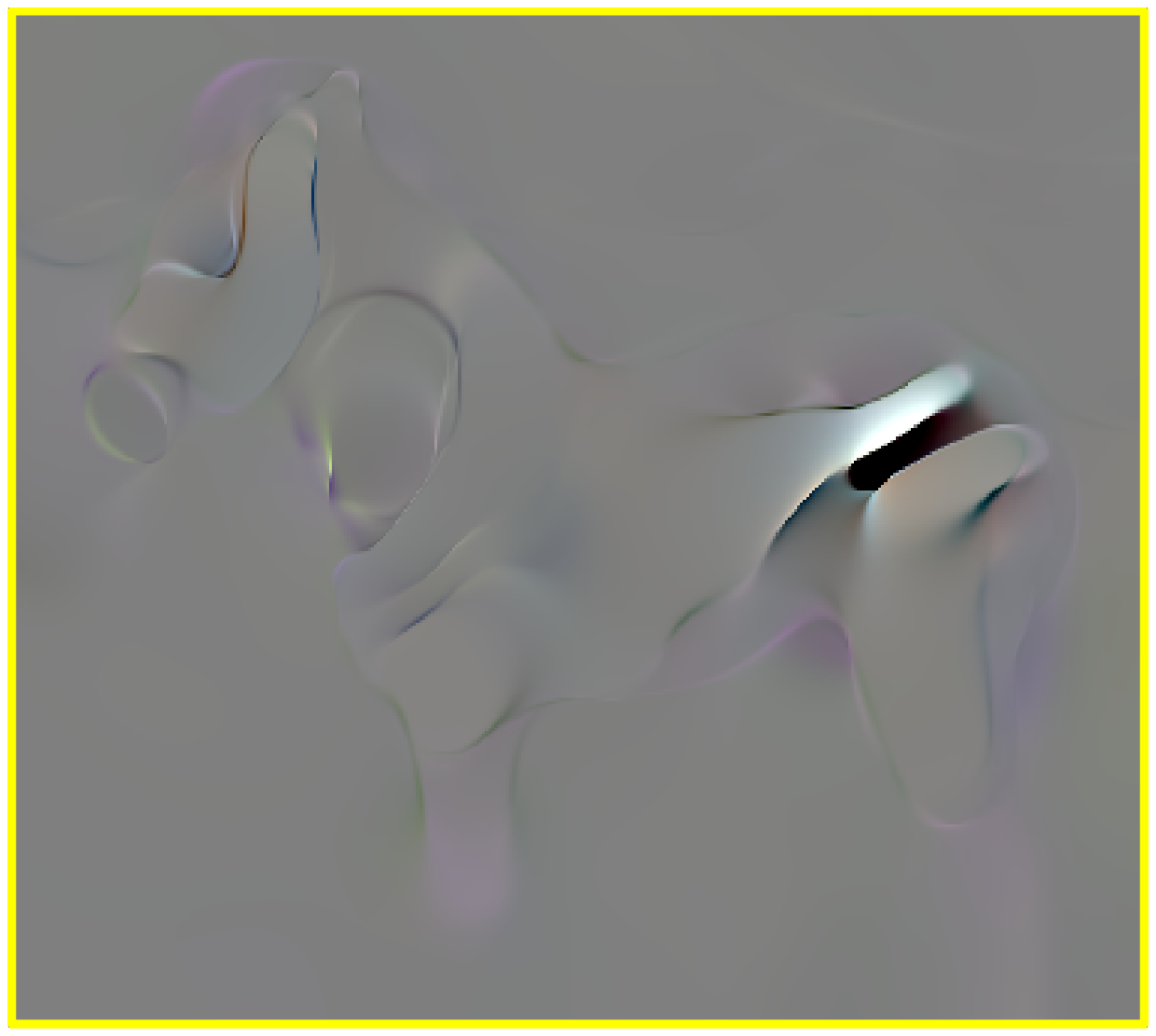}} \\
\end{tabular}
\caption {Zebra image (top left), TV spectrum of the image with separated textures marked in different colors (top right), and spectral decomposition  of the zebra image to textures, displaying integration over time (middle) and certain $\phi$'s (bottom).}
\label{fig:zebra_tex}
\end{figure}

\subsection{Gabor Filters}

A Gabor filter bank is a set of regularly spaced filters that roughly mimic the behavior of the human visual system (HVS) for texture detection.  According to this model, the HVS perceives the image through a set of filtered images, so that each image contains some unique visual information over a narrow range of orientation channel. In that manner, Gabor filtering has been shown to be a good fitting to this model, providing optimal localization of image details in a joint spatial and frequency domain \cite{Jain,daugman1980two}.
The Gabor wavelet definition is
\begin{equation}
g (x,y)=\frac{1}{2 \pi \sigma_x \sigma_y}exp[-\frac{1}{2}(\frac{\tilde{x}^{2}}{{\sigma_x} ^{2}}+\frac{\tilde{y}^{2}}{{\sigma_y} ^{2}})+2\pi j W\tilde{x}]
\end{equation} 
\begin{equation}
\tilde{x} ={x\cos (\frac{\mu \pi}{M})+y\sin(\frac{\mu \pi}{M})} ,
\tilde{y}={-x\sin(\frac{\mu \pi}{M})+y\cos(\frac{\mu \pi}{M})}
\end{equation}
$\mu$ controls the orientation of the filters, with M being the total number of different orientations and $W$ scales the center of the filter in the frequency domain.
\begin{figure}
\centerline{
\includegraphics[width=3.4in]{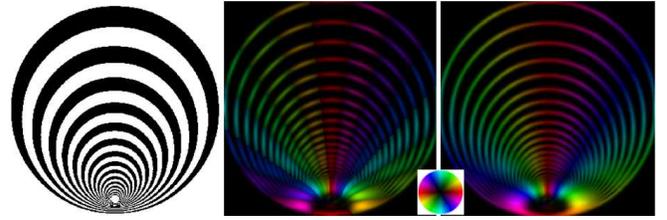}}
\caption {Concentric circles image (left), orientation map of 8 Gabor filters response (middle) and orientation map of 30 Gabor filters response (right).}\label{fig:circles_gabor}
\end{figure}


\section{Multiscale Decomposition Using Spectral TV}
\label{sec:multi}
In this section we show how one can use spectral TV for multiscale decomposition and multiscale orientation analysis.
This part was first presented by the authors in a conference \cite{horesh2015multiscale}.
In this work we extended the common structure and texture decomposition to multi-scale texture separation in order to get all textures, the coarse and fine-scaled. The orientation of the different texture layers was separately characterized using the Gabor filters bank, generating the scale-orientation descriptor. Precise orientation mapping can be useful for
analysis and inner texture actions and synthesis of image with complex textures content.

\subsection{A Necessary Condition for Perfect Separability}
In \cite{spec_one_homog2015} an orthogonality relation between $\phi$, Eq. \eqref{eq:phi}, and $u$ is established:
\begin{equation}
\label{eq:phi_u_orth}
\langle u(t),\phi(t) \rangle = 0, \forall t \in (0,\infty),
\end{equation}
where $\langle \cdot, \cdot \rangle$ denotes the $L^2$ inner product over the domain $\Omega$.
Using the above relation and the one given in \eqref{eq:H_tvflow} a necessary condition for
perfect separability of eigenfunctions can be shown:
\begin{proposition}
\label{prop:sep_cond}
Let $f_1(x)$, $f_2(x)$ admit the eigenvalue problem \eqref{eq:ef_problem}, with $J$ the TV functional \eqref{eq:tv}, and
$\lambda_1$, $\lambda_2$ the corresponding eigenvalues ($\lambda_1>\lambda_2$).
Then for $f = f_1 + f_2$ a necessary condition to have
\begin{equation}
\label{eq:phi_prop}
\phi(t) = \delta(t-\frac{1}{\lambda_1})f_1(x) + \delta(t-\frac{1}{\lambda_2})f_2(x),
\end{equation}
is
$$ \langle f_1, f_2 \rangle = 0. $$
\end{proposition}
\begin{proof}
Let us assume Eq. \eqref{eq:phi_prop} holds and $\langle f_2(x),f_1(x) \rangle \ne 0$.
We express $u(t_1)$ using \eqref{eq:H_tvflow}, with $t_1:=\frac{1}{\lambda_1}$ (note that $H^{t_1}(t_1)=0$):
$$ u(t_1) = \int_{t_1}^\infty H^{t_1}(\tau)\phi(\tau;x) d\tau =
H^{t_1}(t_2) f_2(x). $$
From Eq. \eqref{eq:phi_u_orth} we have $$ \langle u(t_1),\phi(t_1) \rangle = 0, $$
and therefore $$ H^{t_1}(t_2)\delta(t=0)\langle  f_2(x),f_1(x) \rangle = 0, $$
which contradicts our assumption.
\end{proof}
\begin{figure}
\begin{tabular}{cc}
\subfloat[Signals]{\includegraphics[width=1.5in]{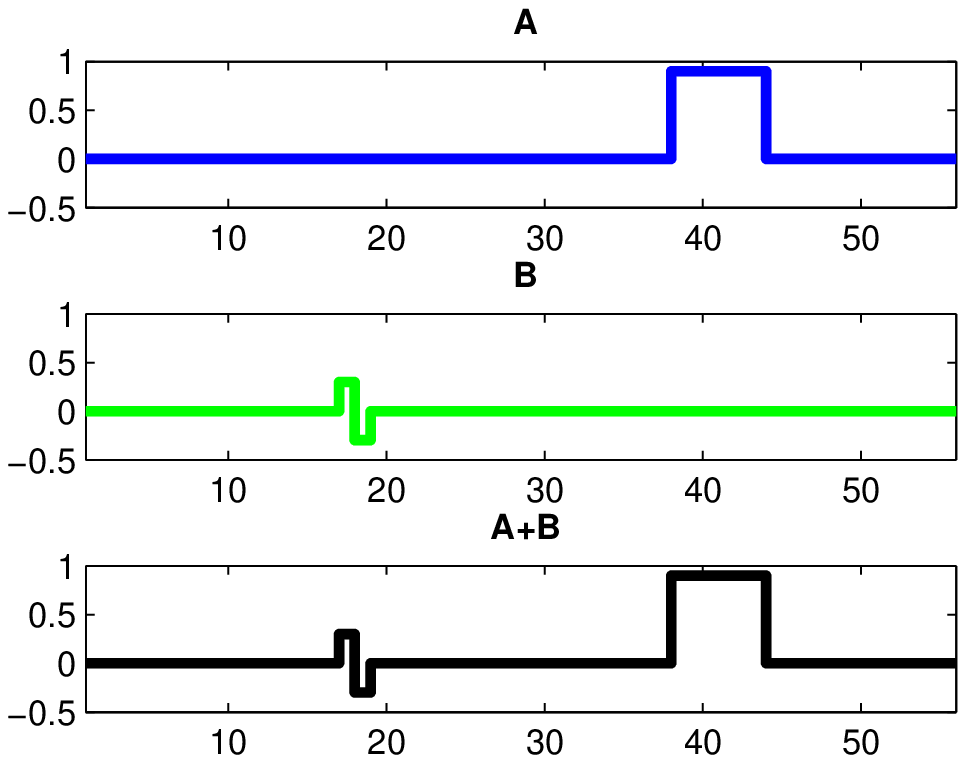}} &
\subfloat[TV spectrums]{\includegraphics[width=1.5in]{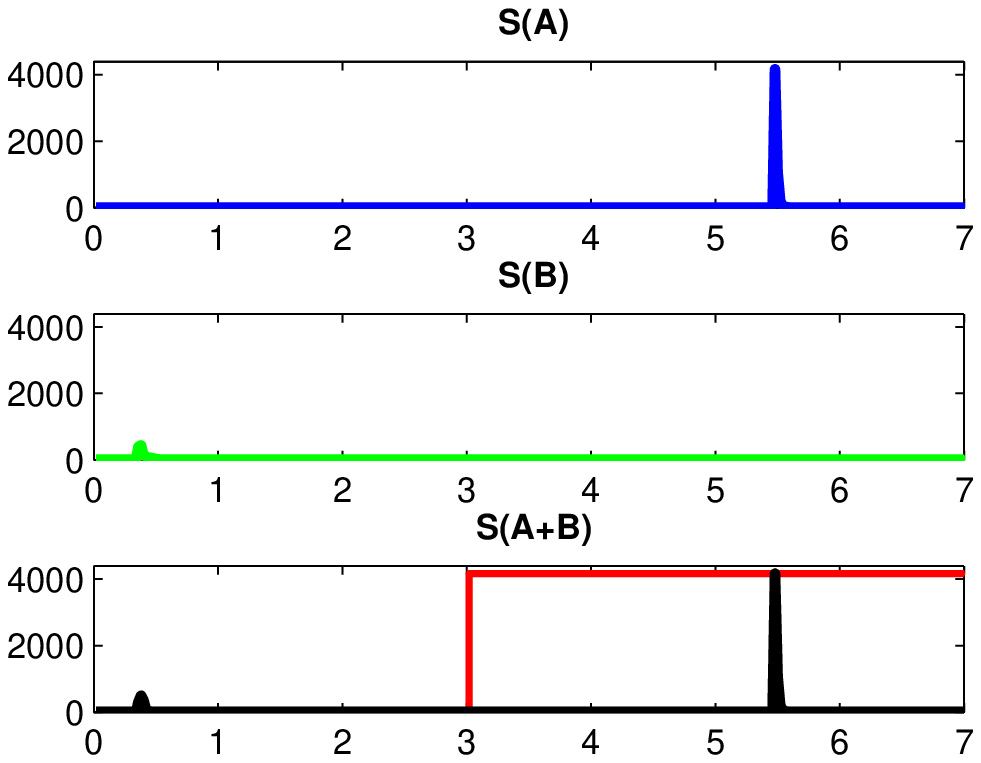}} \\
\subfloat[Spec TV (proposed) result]{\includegraphics[width=1.5in]{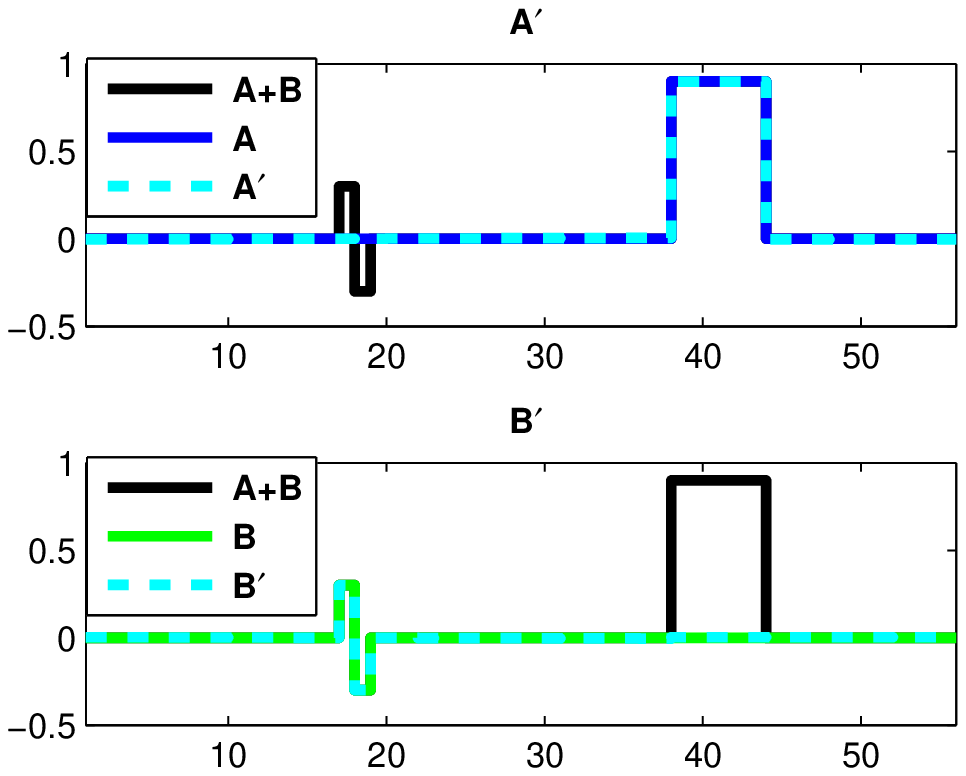}}&
\subfloat[TV result]{\includegraphics[width=1.5in]{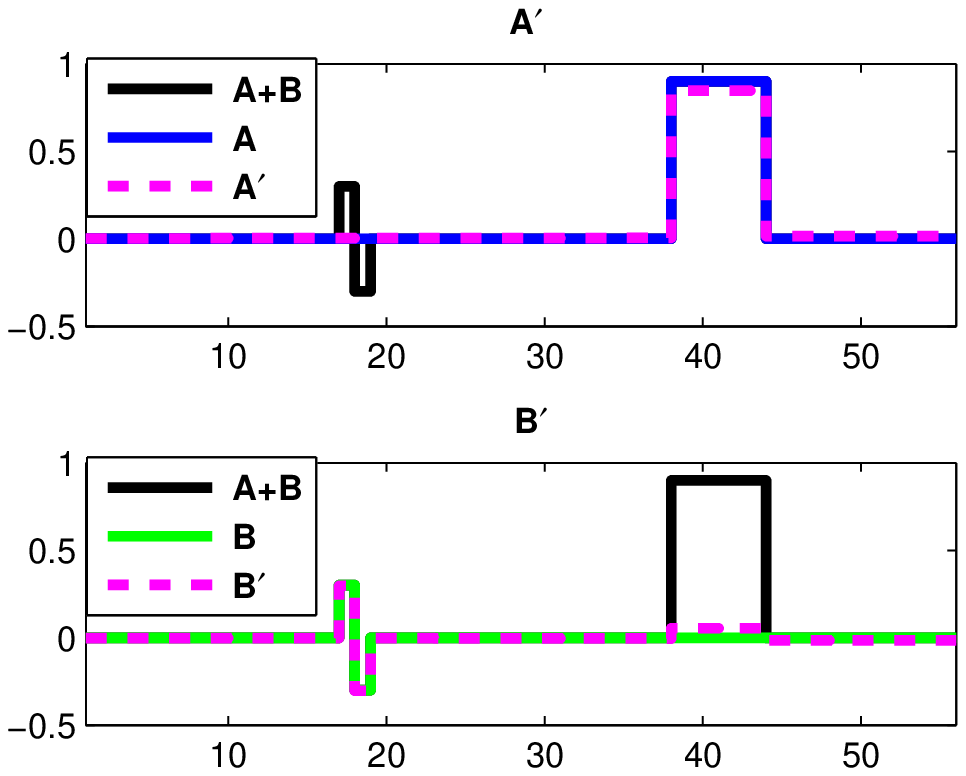}}\\
\end{tabular}
\caption { Separating 1D orthogonal signals.}
\label{fig:simp1}
\end{figure}

\begin{figure}
\centering
\begin{tabular}{cc}
\subfloat[Signals]{\includegraphics[width=1.5in]{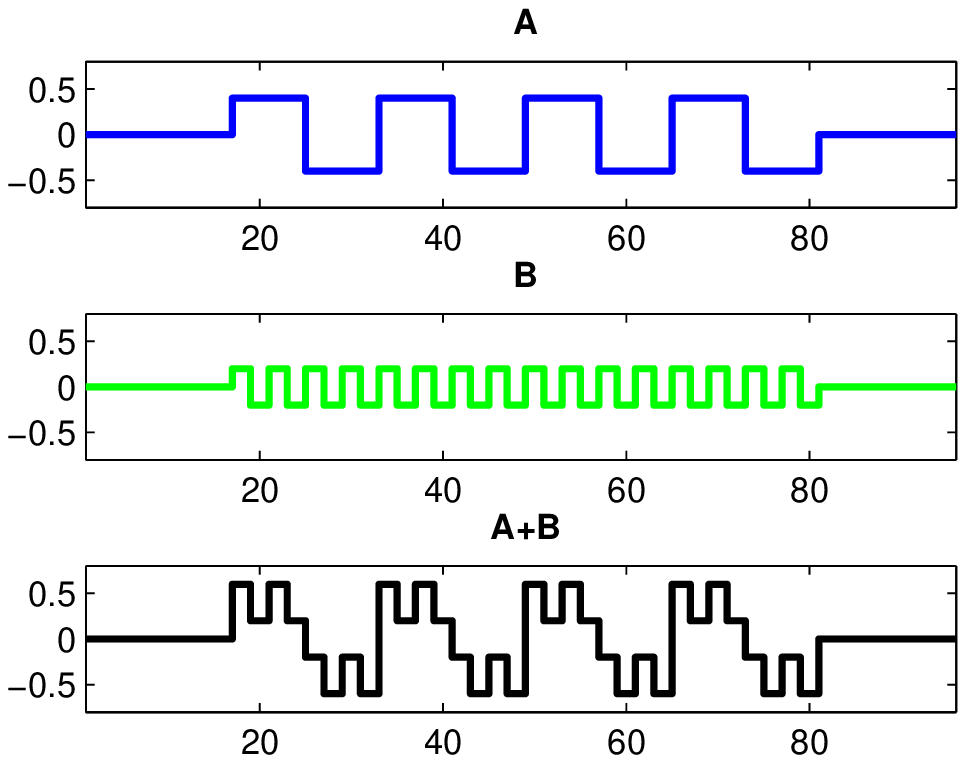}} &
\subfloat[TV spectrums]{\includegraphics[width=1.5in]{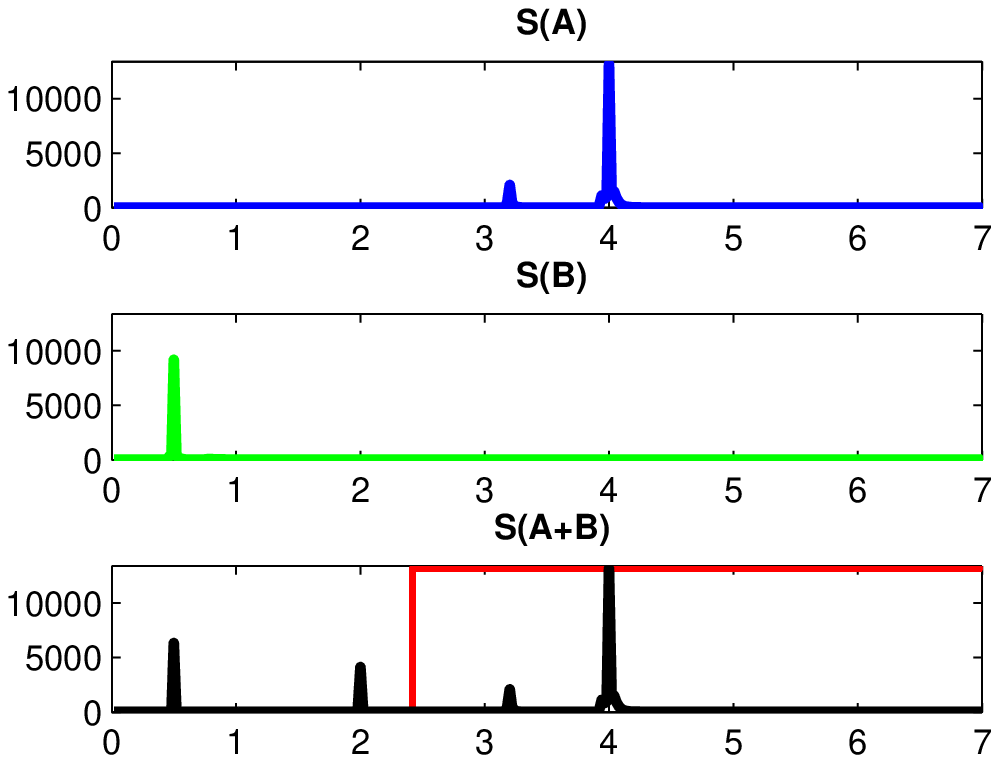}} \\
\subfloat[Spec TV (proposed) result]{\includegraphics[width=1.5in]{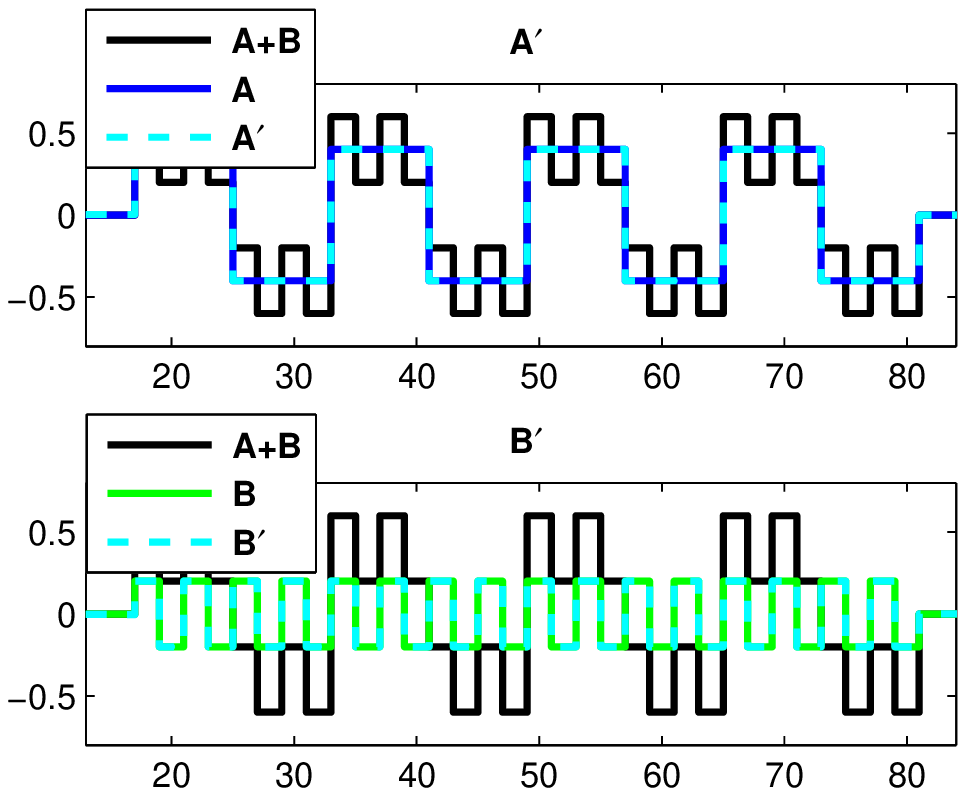}}&
\subfloat[TV result]{\includegraphics[width=1.5in]{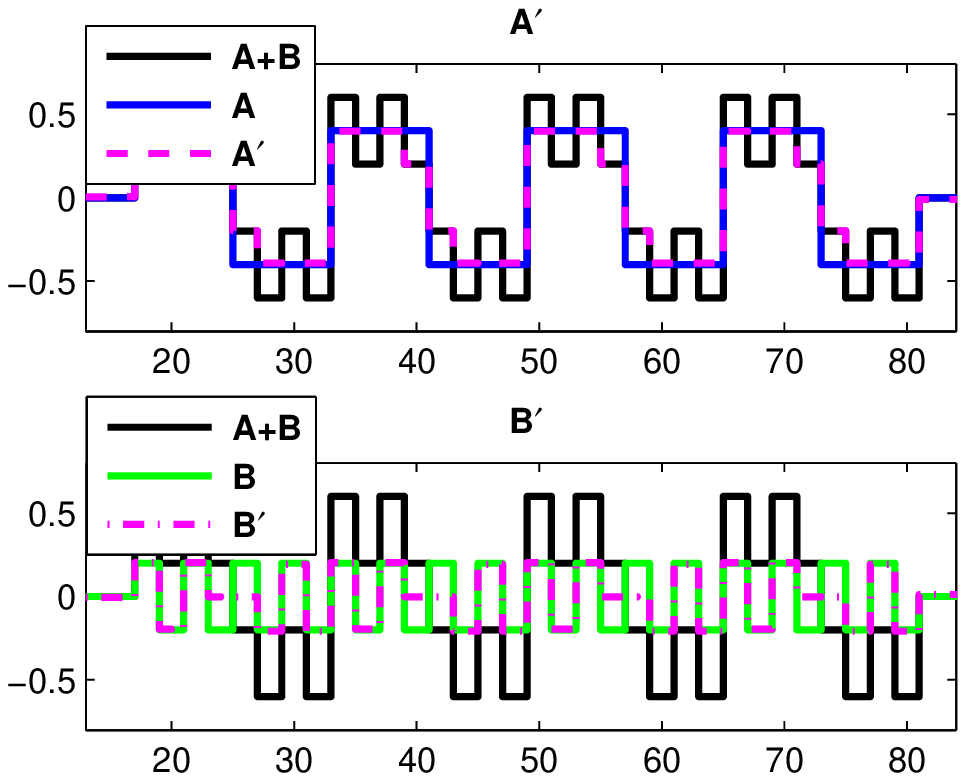}}\\
\end{tabular}
\caption {Separating  1D oscillating uncorrelated signals}\label{fig:os1}
\end{figure}
Note that in the case of perfect separability, Eq. \eqref{eq:phi_prop}, simple spectral filtering \eqref{eq:tv_filt} with $H(t)=1$ for $t<t_c$ and $0$ otherwise, where $t_c \in (1/\lambda_1,1/\lambda_2)$, can perfectly separate $f_1$ and $f_2$.
Examples demonstrating the signals' separability in 1D are shown in Figs. \ref{fig:simp1}-\ref{fig:os1}. Separation of the larger scale signal is performed using \eqref{eq:tv_filt} with $H(t)$ a step signal ($\{0,1\}$ values) as superimposed in red on the combined spectrum plot (bottom of Figures \ref{fig:simp1}-\ref{fig:os1} (b)). In Fig. \ref{fig:simp1} two well-separated orthogonal signals are shown, their spectrum, Eq. \eqref{eq:S}, has one peak for each signal, and the combined spectrum $S^{(A+B)}(t)$ is the sum $S^A(t)+S^B(t)$. Decomposition using spectral TV yields a perfect separation. We can see also the optimal possible result of standard TV regularization for comparison. Note that the standard TV does not yield perfect separation and the decomposition mixes both signals.
In Fig.~\ref{fig:os1}, two orthogonal oscillating signals are combined. Their combined spectrum slightly changes, compared to the original isolated signals, as a result of the overlap. However they can still be perfectly separated using spectral TV, while the standard TV has significant artifacts and the signals are not well separated.
We use the projection algorithm of Chambolle \cite{Chambolle[1]} to implement the TV-flow time steps \cite{Gilboa_spectv_SIAM_2014}.
For more details, see \url{http://guygilboa.eew.technion.ac.il/code/}.

\subsection{Image Decomposition}
Decomposition using spectral TV was applied to textured images. The decomposition can be done to as many different layers as  required, limited numerically only by the chosen time step (the theoretical formulation is continuous in time). An example of such image decomposition can be seen in Fig. \ref{fig:chess_in}. A game board image is shown with its spectrum in different colors to demonstrate the separation points of the decomposition (or integration intervals of the $\phi$'s, using \eqref{eq:tv_filt} with $H=1$ in the desired interval and 0 otherwise) appearing respectively in Fig. \ref{fig:chess_tex}: the wood pattern , the game board lines and the structure with the round game pieces.
In this example, the separation points of the decomposition were manually chosen. However, in \cite{agco06} it was done automatically for structure-texture decomposition using correlation between the texture and the residual structure in each level. In the next section of this work we explore the automatic texture separation.

\subsection{Scale-Orientation Descriptor}
After decomposing the image we generate the scale-orientation descriptor (SOD) for each texture level by fully mapping its orientation, creating a multi-valued orientation descriptor for each pixel. This was achieved using the Gabor filter bank.
The Gabor filter response was calculated in 30 orientations ($M=30, \mu =0, 1, ..., 30$), spanning 180\degree  \, and 4 spatial frequencies for finer and coarser scales. In Fig. \ref{fig:circles_gabor} presented a concentric circles image (left) and the colored orientation map of 8 (middle) and 30 (right) Gabor filters response.
In each scale, the orientation giving the maximum response was selected, and the maximum of the 4 scales was taken to describe the orientation of each pixel.

An example of such scale-orientation descriptor can be seen in Fig. \ref{fig:chess_tex}. In this figure, decomposition of the game board image (Fig. \ref{fig:chess_in}) to three scales is shown, together with the corresponding orientation mapping images. We can observe the great separation of the different layers as seen in the colored orientation figures as processed from the Gabor filters response. The orientation in each pixel is taken from the Gabor scale which gave the strongest response, usually corresponding to the texture coarseness. The result of Gilles decomposition  \cite{gilles} for this image are on Fig. \ref{fig:chess_Gilles}, as can be seen, the textures nicely appear there in different scales but the different textures are not separated.

\begin{figure}
\centerline{
\includegraphics[width=1.3in]{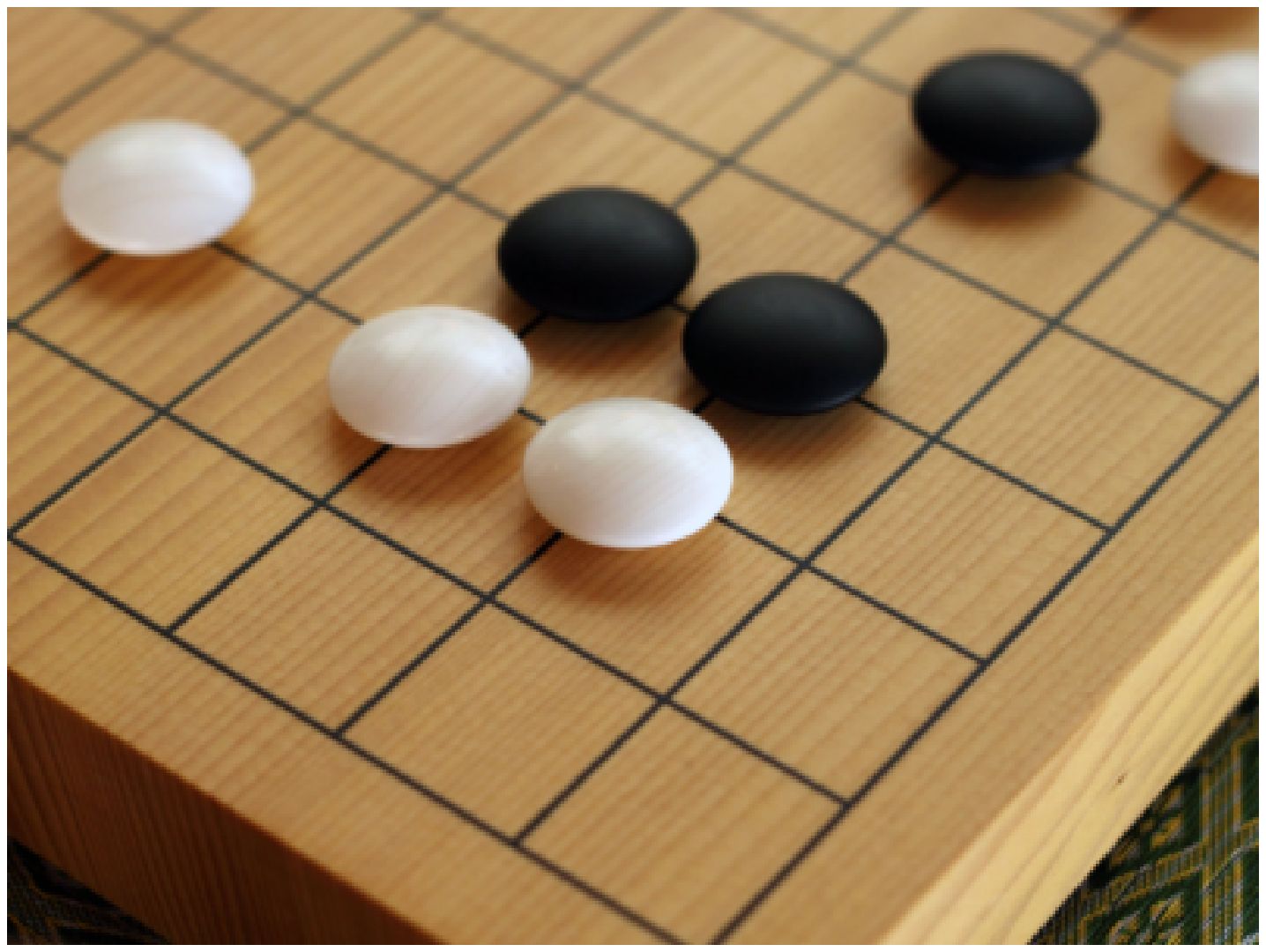}
\includegraphics[height=0.9in,width=2.1in]{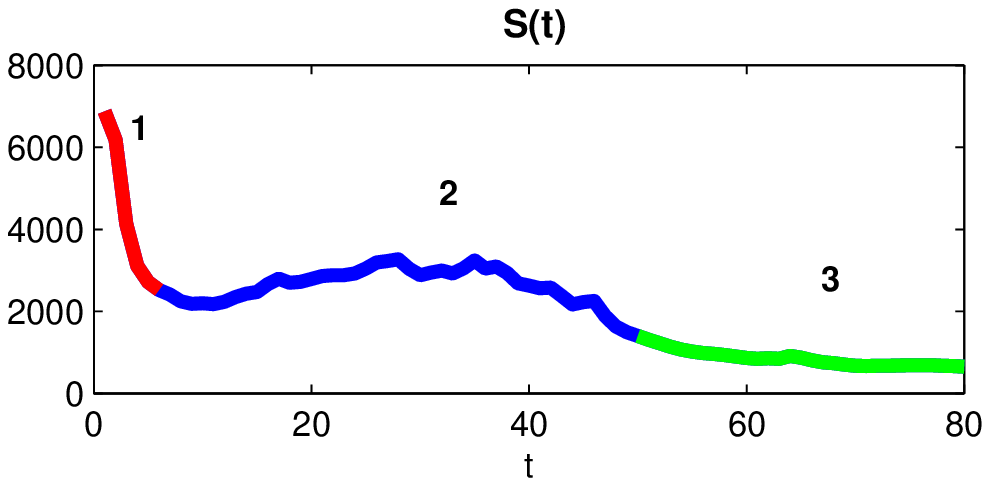}}
\caption {Game board image (left) and TV spectrum of the image with separated textures marked in different colors (right) }\label{fig:chess_in}
\end{figure}
~\begin{figure}
\centerline{
\includegraphics[width=1.1in]{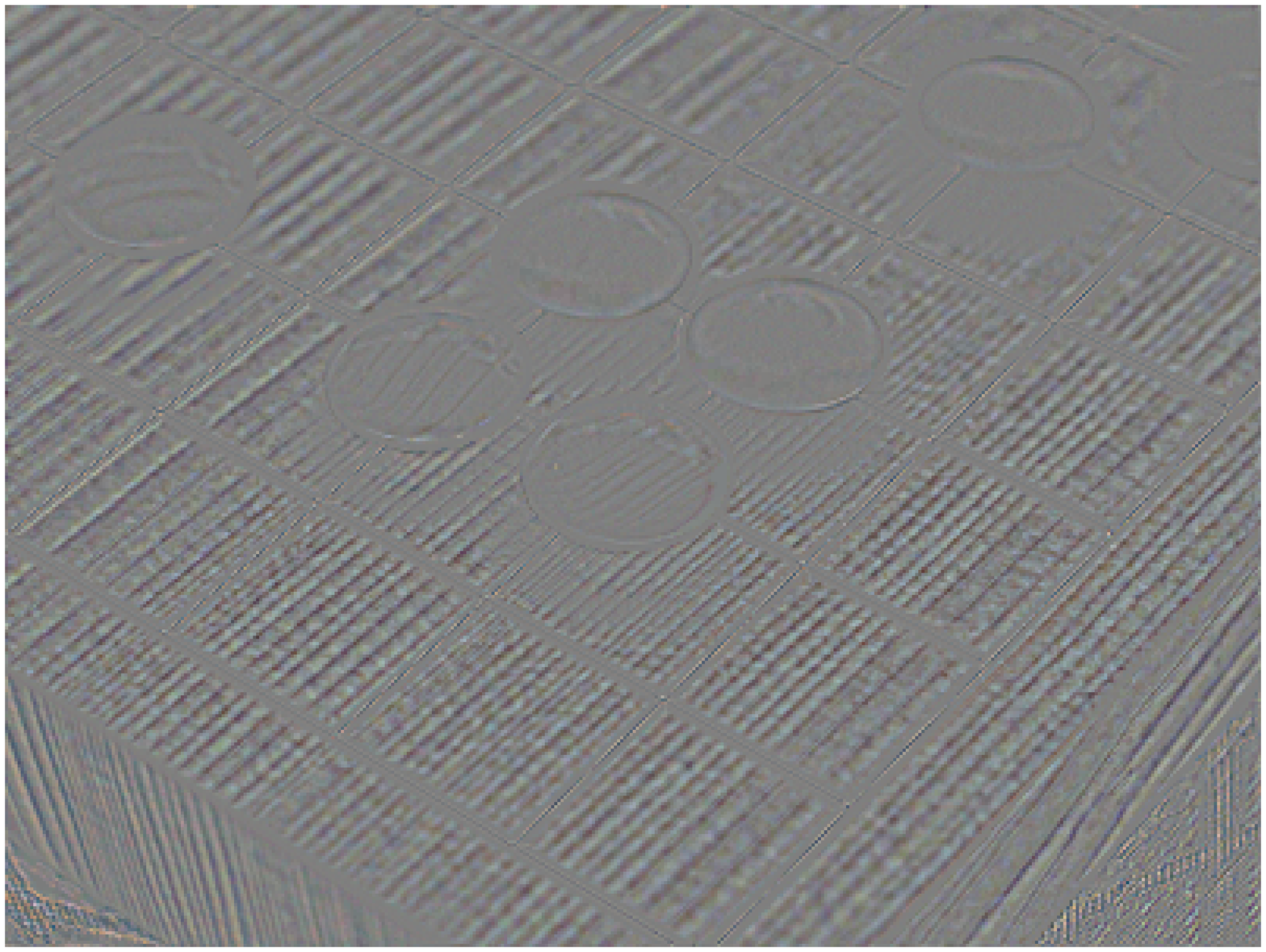}
\includegraphics[width=1.1in]{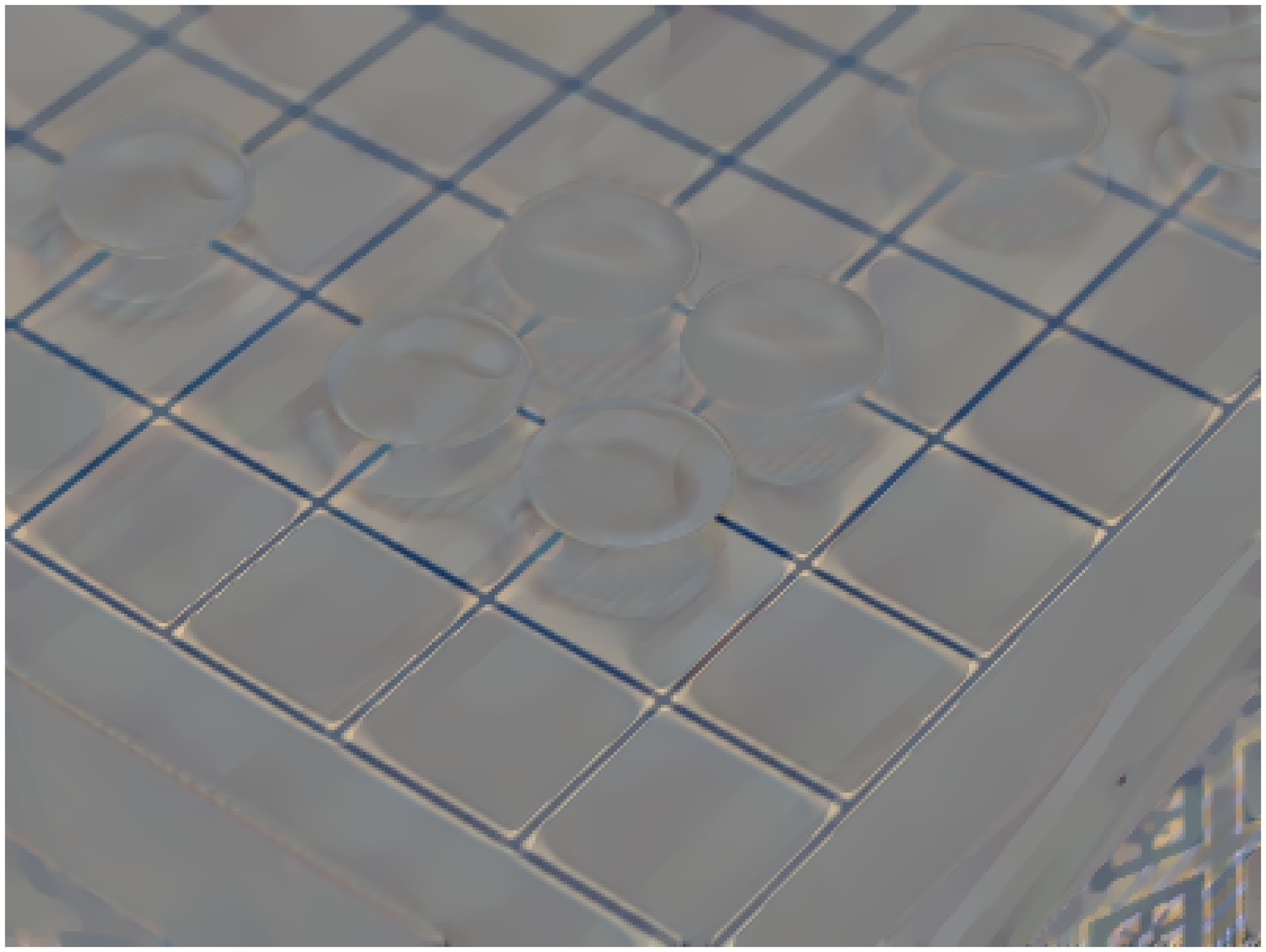}
\includegraphics[width=1.1in]{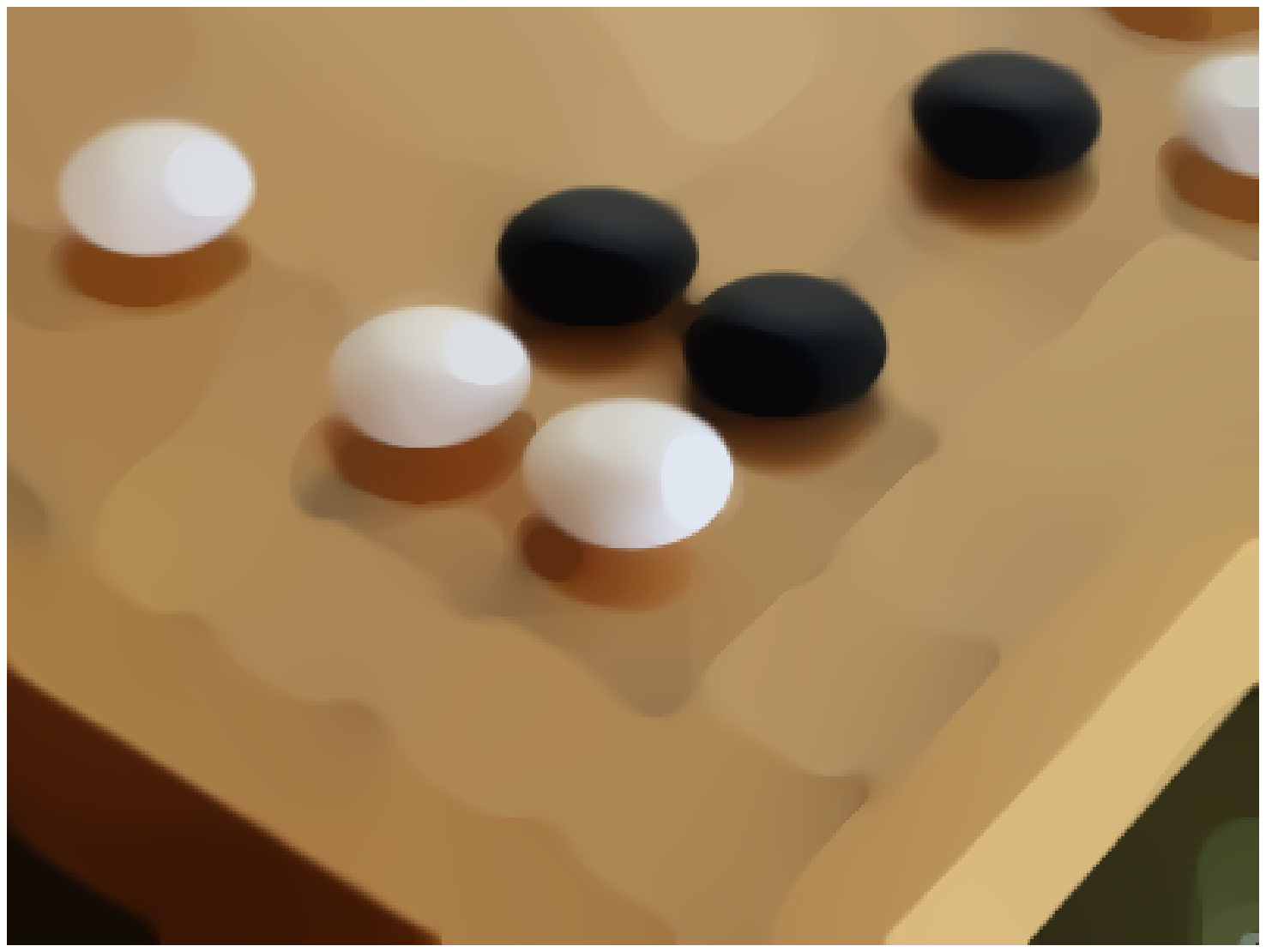}}
\vspace{5pt}
\centerline{
\includegraphics[width=1.1in]{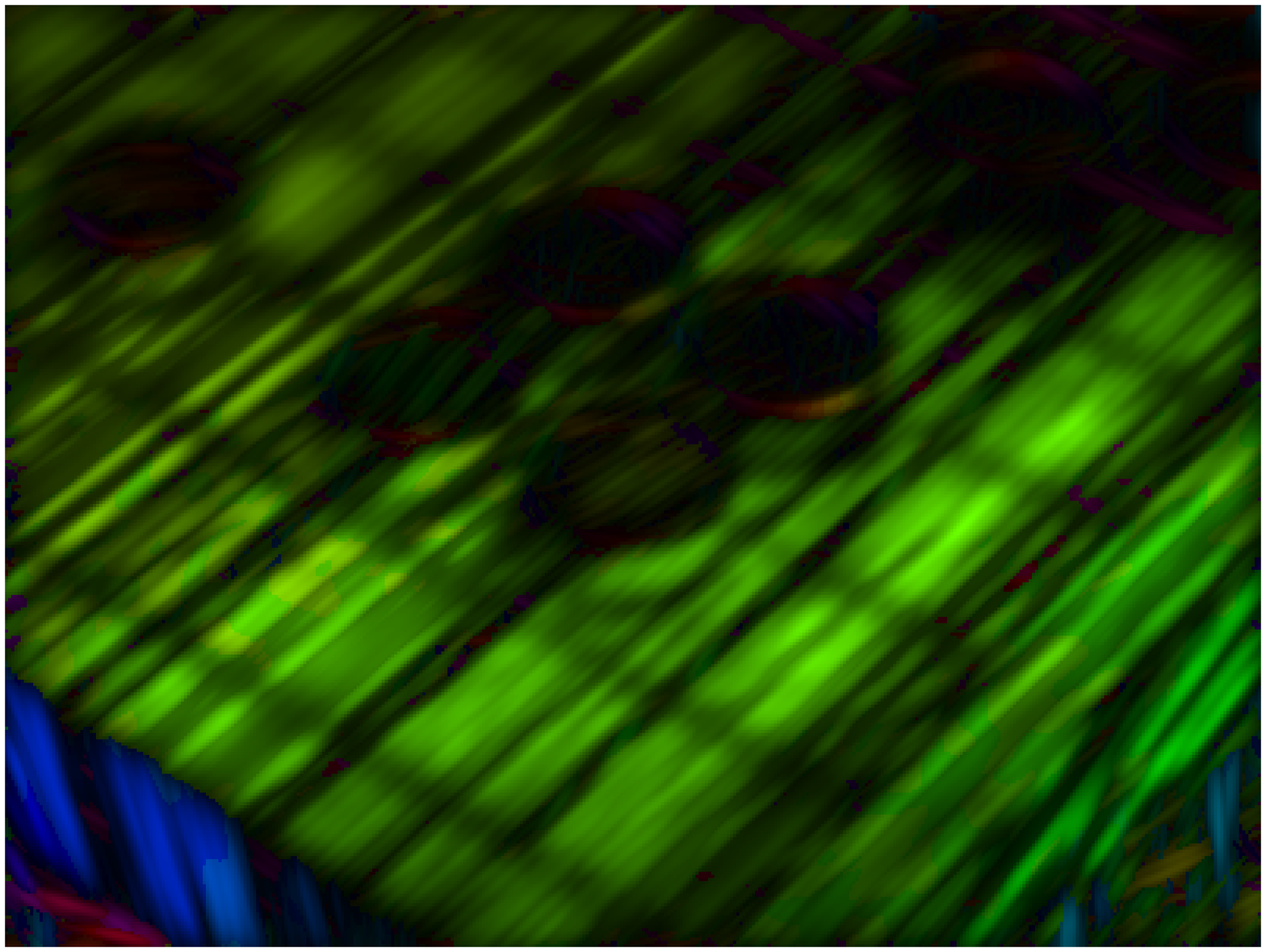}
\includegraphics[width=1.1in]{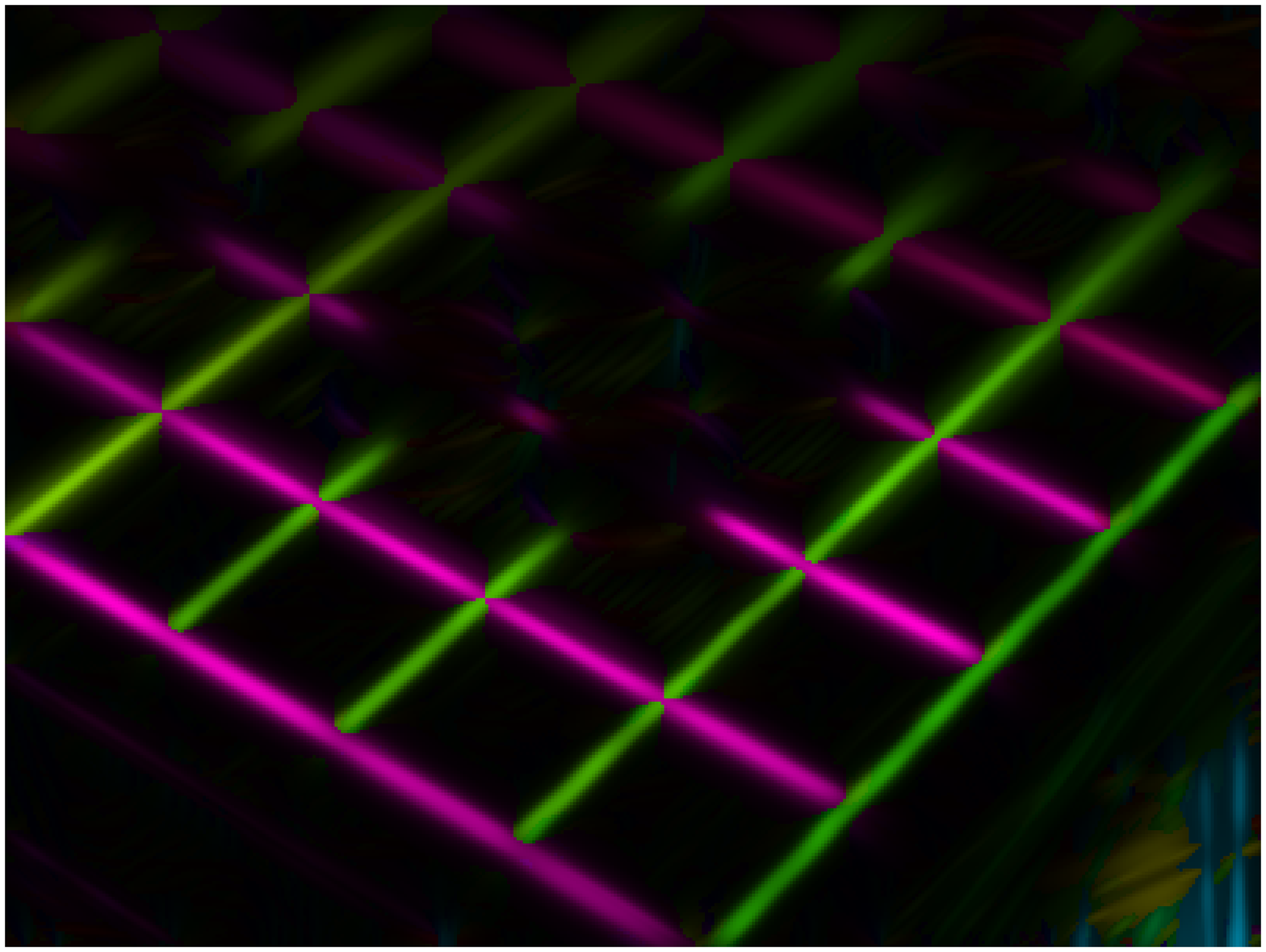}
\includegraphics[height=0.82in,width=1.1in]{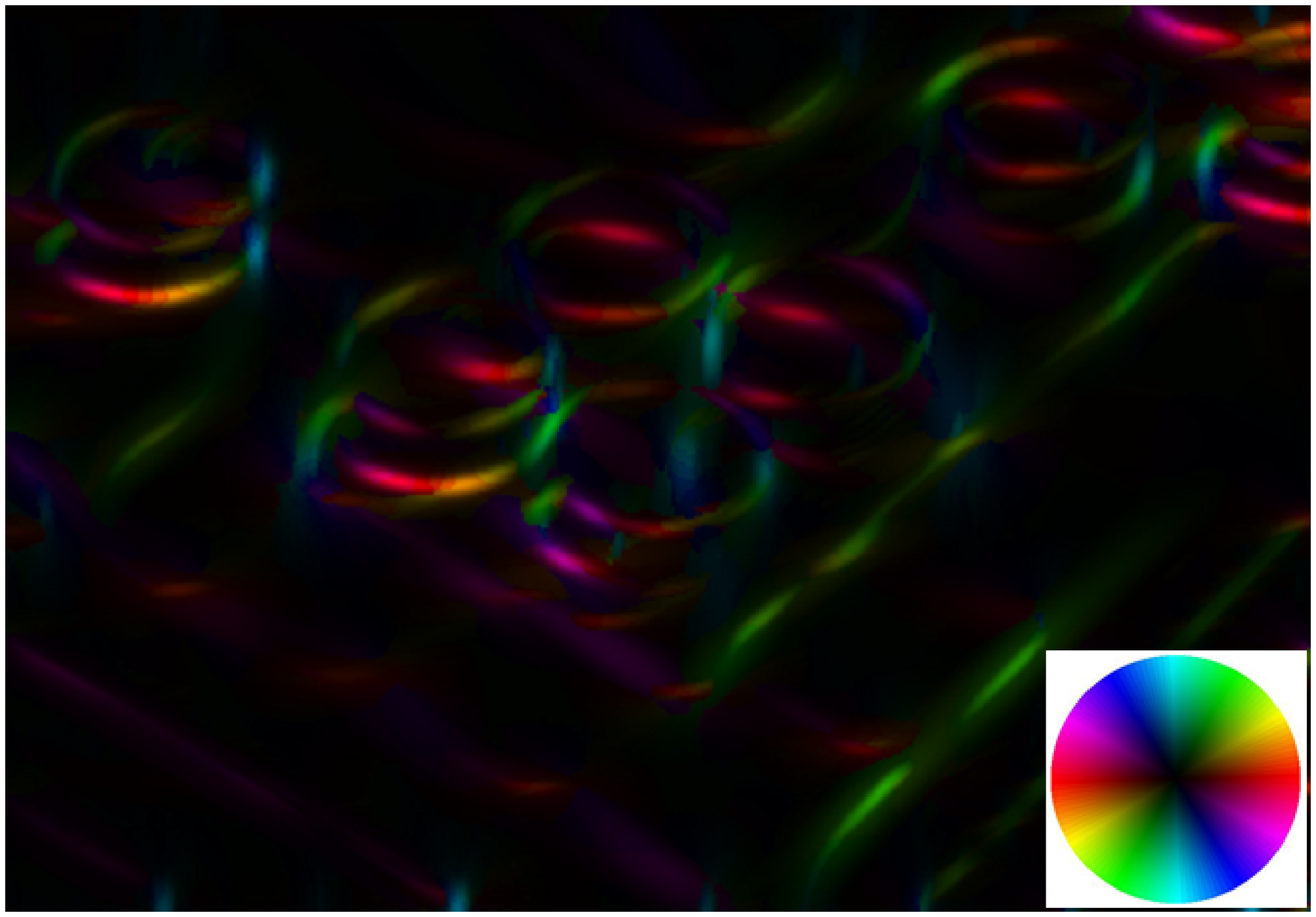}}
\caption {SOD of Game board image (Fig. \ref{fig:chess_in}). Our multi scale decomposition  of the image (top) and the corresponding orientation maps (bottom). (The contrast is enhanced for better visualization.)}\label{fig:chess_tex}
\end{figure}
\vspace{-15pt}
~\begin{figure}
\begin{tabular}{c}
\includegraphics[width=3.4in]{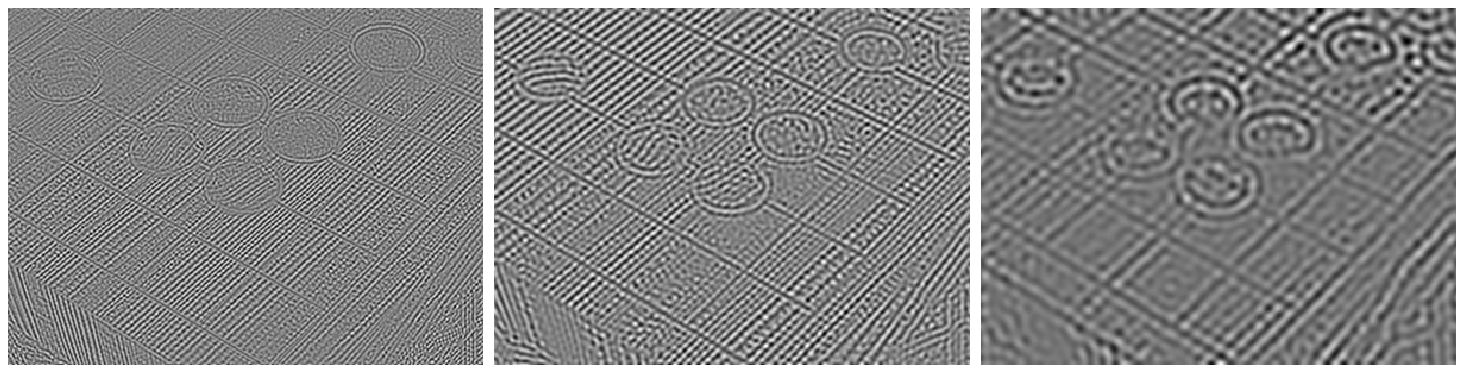}\\
\includegraphics[width=3.4in]{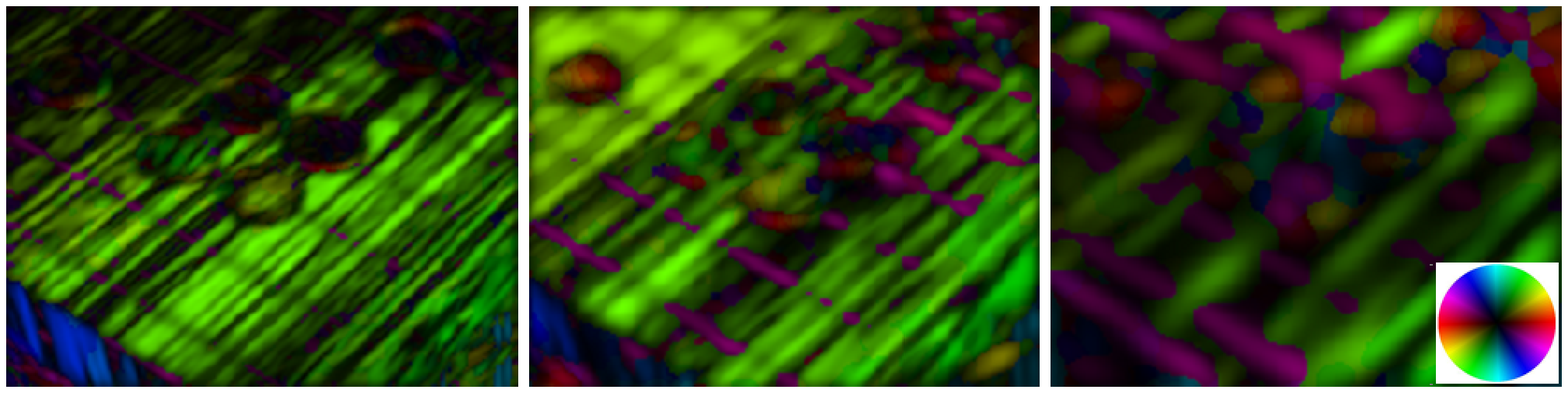}
\end{tabular}
\caption {Compared method \cite{gilles} decomposition result for game board image to 3 scales (top) and Gabor orientation visualization (bottom)}
\label{fig:chess_Gilles}
\end{figure}

\section{Separation Surface}
\label{sec:surf}

We now present the novel notion of the separation surface. The texture decomposition so far was done assuming that our texture is homogeneous and can be separated at a certain configuration which fits the whole image. However, that is not the situation in many natural images, in which, due to changing texture, lighting conditions, or perspective, the desired texture can not be decomposed in the same configuration for  the entire image. For that purpose we introduce the separation surface. It is a decomposition configuration, changing in a continuous manner in the image according to the texture.
We first characterize the different textures in the image to find the spectral time band of the desired texture. We then determine the exact spectral layer ($\phi$) for each pixel in a global manner to ensure smoothness of the extracted texture, under the assumption of texture being smooth and continuous. We define a band surrounding the surface to capture the entire features of the desired texture. We call it \textit{stratum}. In geology and related fields, a stratum is a layer of sedimentary rock or soil with internally consistent characteristics that distinguish it from other layers. Similarly, in our case, the stratum defines a certain texture, distinguished from other by its features. It consists of the spectral layers forming the desired texture for each pixel.
A figure simulating the separation surface can be seen in Fig. \ref{fig:surface}, as well as an image of natural rock stratum.




\subsection{Analytic Example}
We begin by examining an analytic example to understand the surface fitting process. It is well established that disks are elementary structures for the TV functional. They satisfy the eigenvalue problem \eqref{eq:ef_problem} which implies their shape stays the same during a TV gradient descent evolution (TV-flow \cite{tv_flow}). Moreover, the eigenvalue of a disc of radius $r$ and height $h$ is inverse proportional to those measures \cite{Meyer[1],tv_flow}:
\begin{equation}
\label{eq:disc}
\lambda \propto \frac{1}{h r},
\end{equation}
where in the spectral TV domain we have the discs appear at scale $t= \frac{1}{\lambda}  \propto hr$.
We would like to analyse the image in  Fig. \ref{fig:circles}, containing synthetic discs in varying size and contrast. In order to do so, we will recall two properties of spectral TV  \cite{Gilboa_spectv_SIAM_2014}:\\
\textit{Contrast change.}  
$$ f(x) \to af(x),\,\, \phi(t,x) \to \phi(t/a,x),\,\, S(t)\to S(t/a).$$
\textit{Spatial scaling (2D)}.
$$ f(x) \to f(ax),\,\, \phi(t,x) \to a\phi(at,ax),\,\, S(t)\to a^{-1}S(at).$$
For example, for an image which is half the size of the original image, we have a = 2. 

 The effect of the contrast change in our example is that the darker the discs (lower contrast), their scale is lower. The effect of the scale change  is that the smaller the discs, their scale is lower. In this example one can see how the scale and the maximal $\phi$ value gradually changes with size and contrast.
 Following this analytic example we can understand the spectral TV behaviour of different images and textures.

\begin{figure}
\centering
\begin{tabular}{ccc}
\subfloat[Input image]{\includegraphics[width=1.1in, height=0.7in]{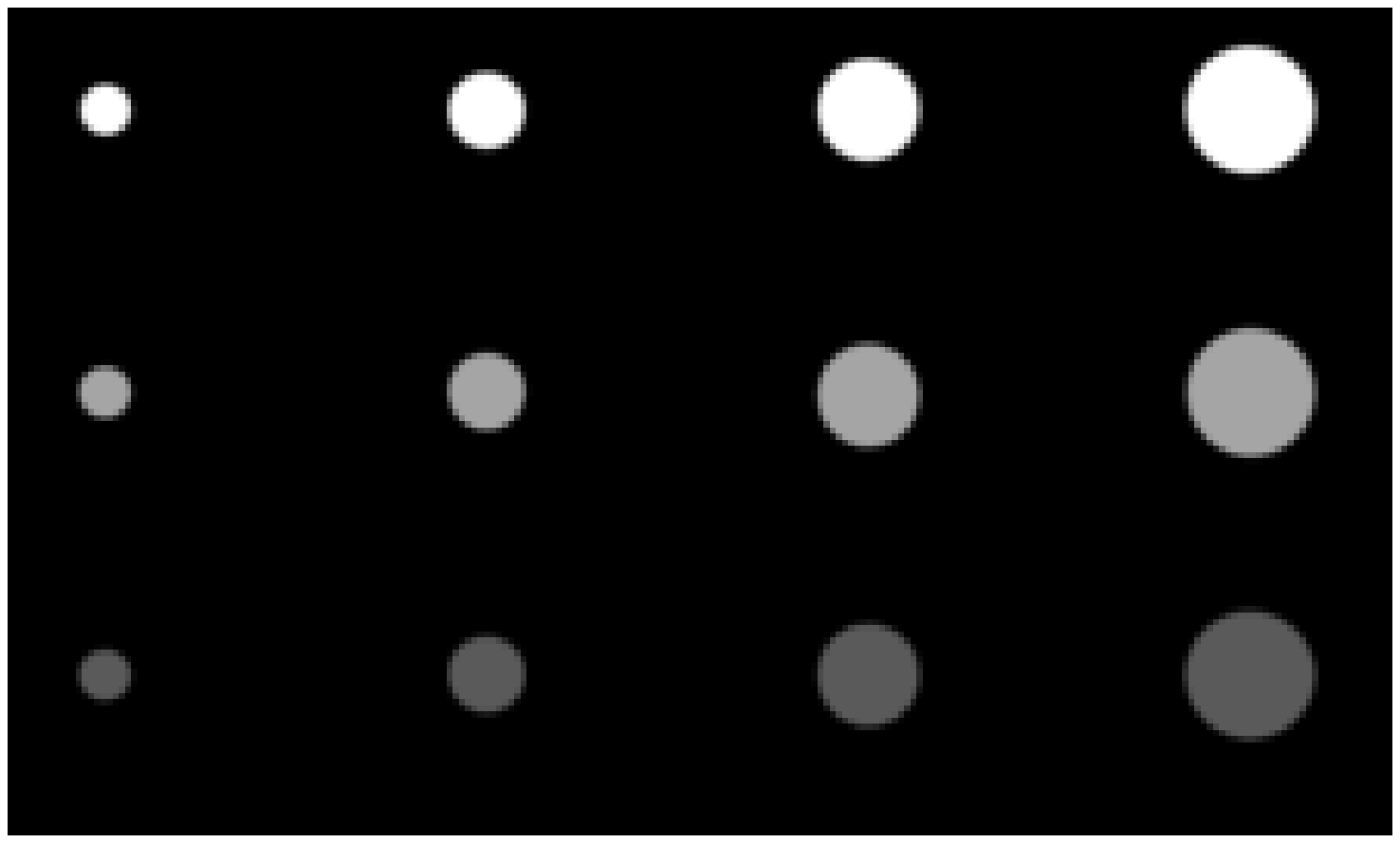}} &
\subfloat[Max Phi Value]{\includegraphics[width=1in]{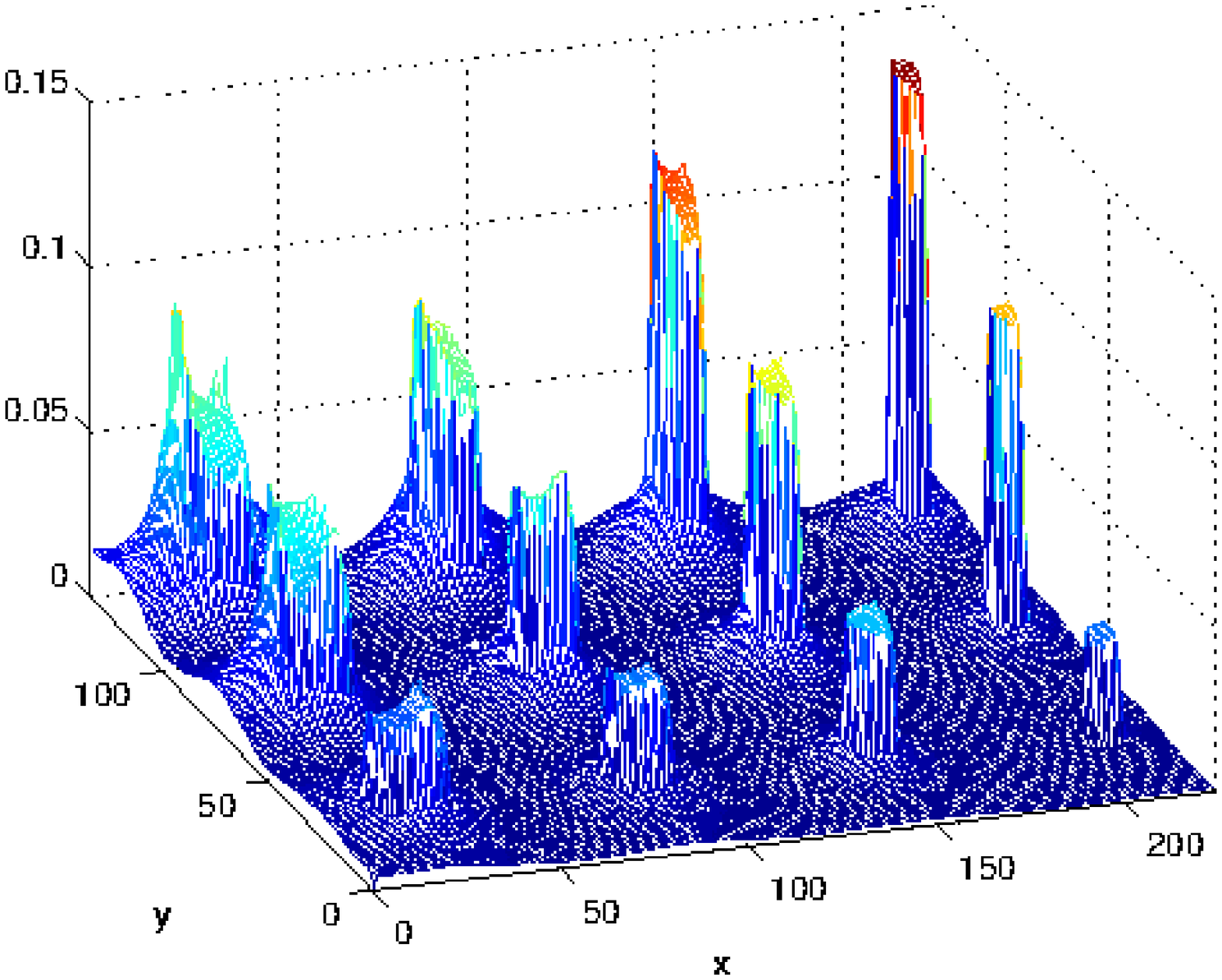}} &
\subfloat[Max. Time]{\includegraphics[width=1.1in]{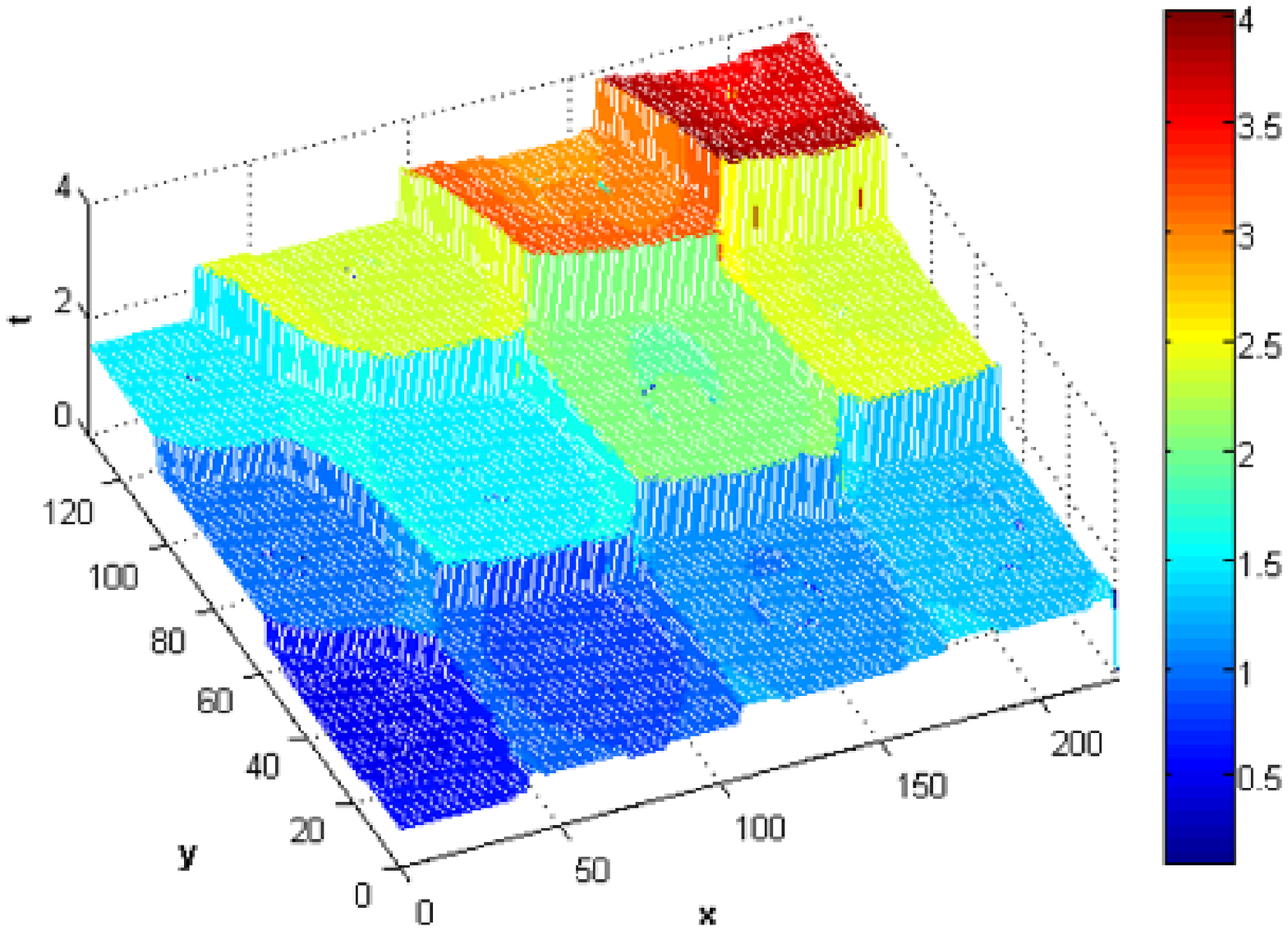}} \\
\end{tabular}
\caption {Input image of discs varying in size and contrast (a), the maximal $\phi$ value (b) and the evolution time  (c)}
\label{fig:circles}
\end{figure}

\begin{figure*}
\centering
\includegraphics[width=6.6in]{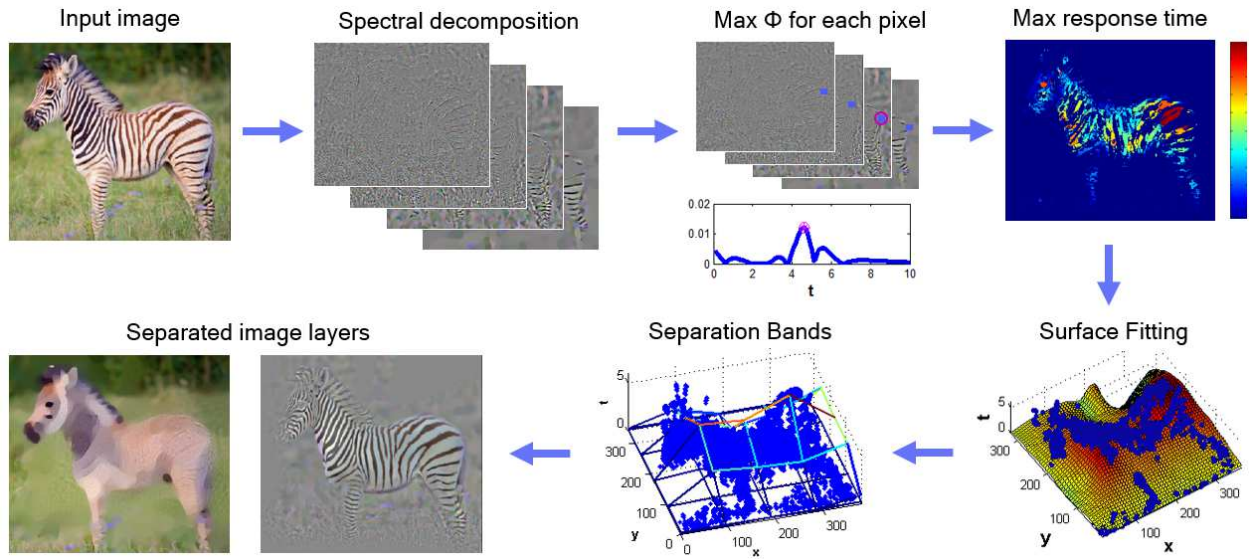}
\caption {Algorithm flow of the separation surface}
\label{fig:algorithm}
\end{figure*}

\subsection{Algorithm}
We present a separation surface algorithm using the TV spectral framework. The idea is that natural textures are not homogeneously distributed in scale and space, instead they continuously vary in the image. Under that assumption we suggest a decomposition configuration, changing in a continuous manner in the image according to the texture.
The spectral framework allows a straightforward way to achieve such complex decomposition.
The general approach can be described as follows:
\begin{enumerate}
\item Compute the spectral components $\phi(t)$ and the spectrum $S(t)$.
\item Manually or automatically, analyze the spectrum and choose the spectral components $\phi(t)$, $t_1 \leq t \leq t_2$, among which the desired texture stretches.
\item Create a salient time map $T(x)$: for each pixel $x$, take $t$ which maximizes $\phi(t,x)$ within the range $t \in [t_1,t_2]$.    
\item Filter $T(x)$ to enable good surface fitting.
\item Fit a surface using regression to the corresponding filtered time mapping.
\item Calculate the bands below and above the surface to get the integration times of the desired texture in each pixel.  
\item Reconstruct the desired texture layer by integrating over the texture stratum, using:
\begin{equation}
\label{eq:stratum}
f_H(x)=\int_{0}^\infty{}H(t; x) \phi (t; x)dt,\\
\end{equation}
$$H(t; x) = \left\{\begin{matrix}
1 &  (t;x) \in stratum \\
0 & else
\end{matrix}\right.$$

\end{enumerate}
\textit{Notes}
\begin{itemize}
\item We take the maximum $\phi$ among the selected spectral components under the assumption that in the scale range of the desired texture, it is dominant and therefore its response on those spectral components is high, and at most times, higher than the other patterns in those scales.
\item Filtering of the maximal $\phi$ can include omitting values on image boundaries, taking only $\phi$ values at limited percentiles (we used percentile range of 85-95) etc..
\item The band width is set according to the scale of the surface at each image location. 
The band is wider
at higher spectral scales, due to smear effect with time (Fig. \ref{fig:text1504_result}(c)).
\end{itemize}
Figure \ref{fig:algorithm} illustrates the proposed algorithm of the separation surface. Zebra image was taken, to extract the stripes texture. It can be vividly seen that for coarse stripes texture, later times are taken, and larger separation band accordingly. The output is the two separated image layers, of the zebra stripes and the image residual. One can see that all the stripes are extracted, coarse and fine, with high and low contrast, while successfully preserving the edges in the residual image.

\subsection{First Order Surface Examples}
Favourably, we work with first order surfaces, since it allows us to better regulate the data and dismiss outliers, and due to its robustness. Later we show that it can be generalized to surfaces of higher order. Let us examine a synthetic example. A texture was taken of the Brodatz texture database, and was added a pattern of circles horizontally varying in contrast, Fig. \ref{fig:text1504_result}. In this example, we can see in the separation band image that the base texture and the circles are not separated in scale, and the fitted separation plane and the band taken, manage to capture the entire circles features while including  just a bit of the base texture, and only on the left of the image, where the contrast of the circles is low and their spectral scale is similar to that of the base texture. Our separation result is very good and highly resembles the original textures. 
Comparisons to TV-G and to RGF are shown where in both cases the separation is not good and residual of the textures 
 vividly remains in both the layers. Another example is shown on Fig. \ref{fig:tiles} of street tiles, one can see the pattern is homogeneous in nature but in the image it is linearly changing due to perspective point of view. We decomposed that image using the spectral TV filtering, and took the max $\phi$ time mapping. We can see it captures the linear change in the vertical direction of the image. The upper side is far from the photographer, thus details are small and the evolution time is short. The bottom of the image is close, thus details are coarser and the resulting evolution time is longer. Taking the stratum accordingly, we capture the entire texture in one image and the structure in another image, with sharp edges, and no tiles remainder. We compared this result to the state-of-the-art RGF result  \cite{zhang2014rolling} which suggest a multi-scale decomposition scales, we show here the 2nd and 3rd scales. In both of them, there can be seen a difference between the lower and upper sides of the image. In the 2nd scale there are many tiles edges, mostly on the lower side of the image, and in the 3rd scale, the structure edges are already smeared, especially at the upper side of the image.

\begin{figure}
\vspace{-10pt}
\centering
\begin{tabular}{cc}
\subfloat[Synthesized input image]{\includegraphics[width=1.6in]{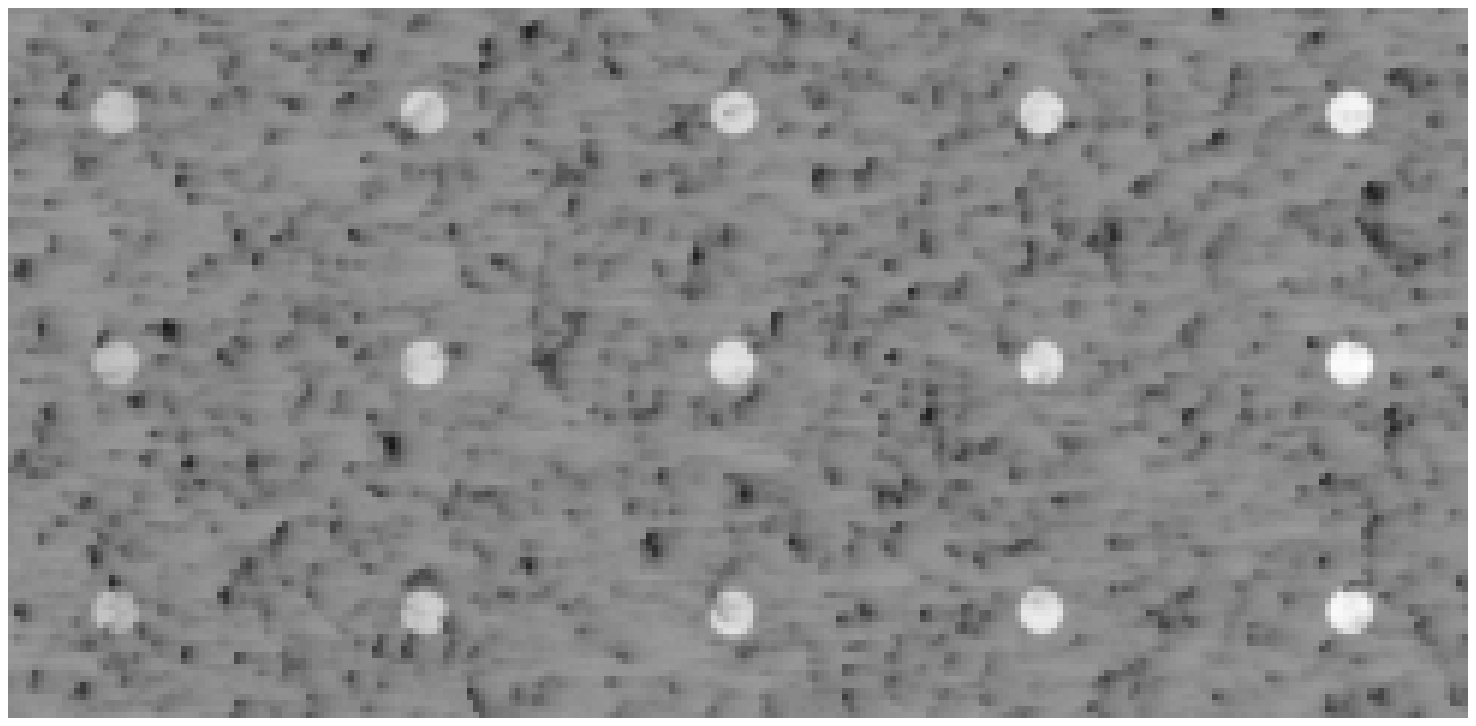}} &
\subfloat[Max. $\phi$ time]{\includegraphics[width=1.6in]{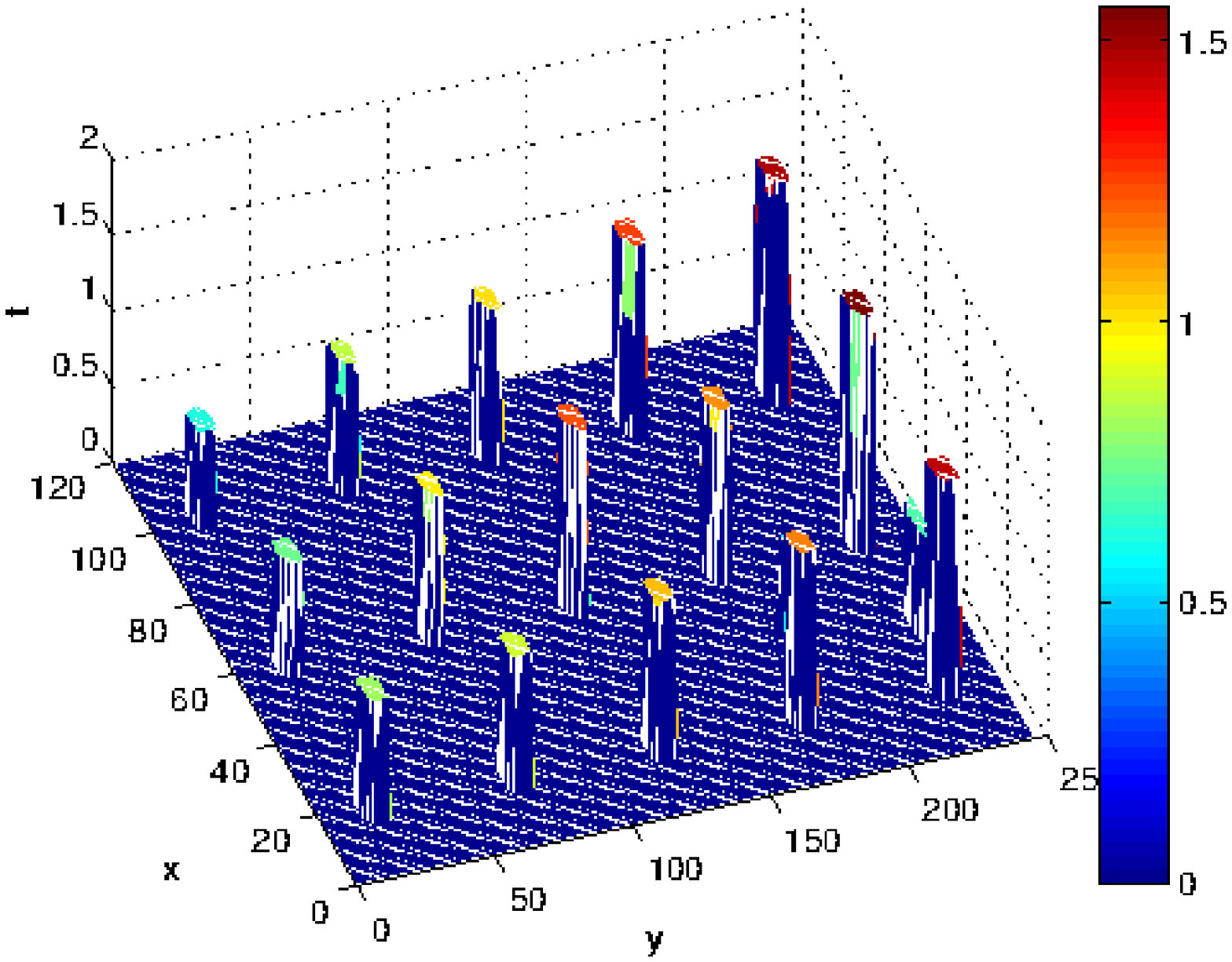}} \\
\subfloat[Band width]{\includegraphics[width=1.6in]{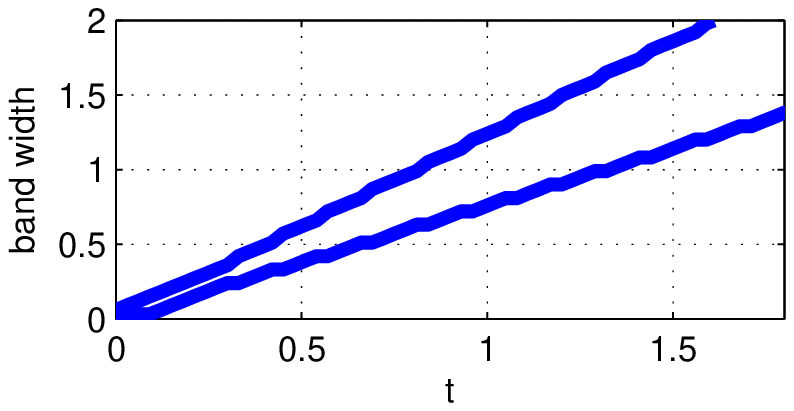}} &
\subfloat[Separation band]{\includegraphics[width=1.6in]{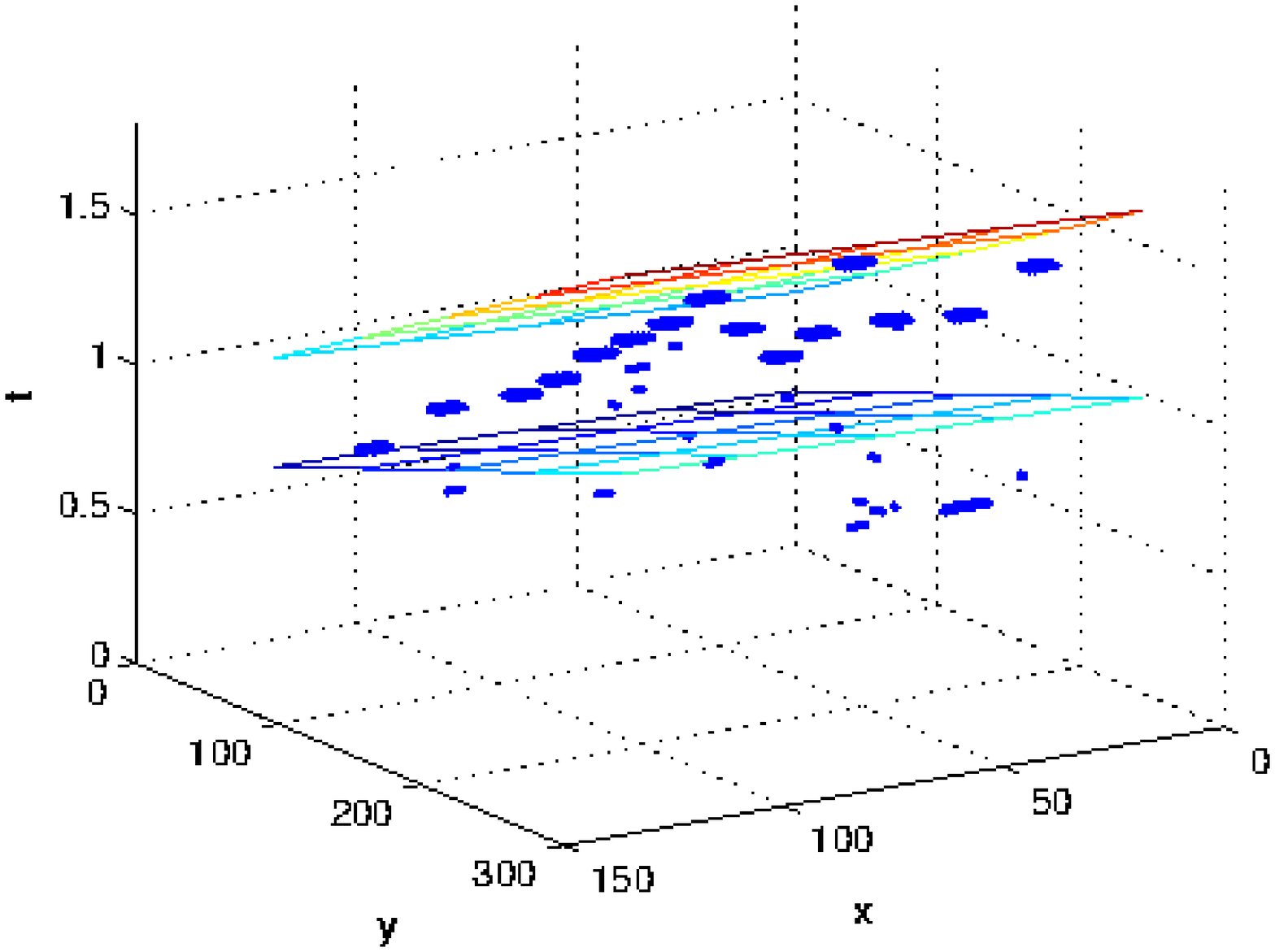}} \\
\subfloat[Texture 1 - original]{\includegraphics[width=1.6in]{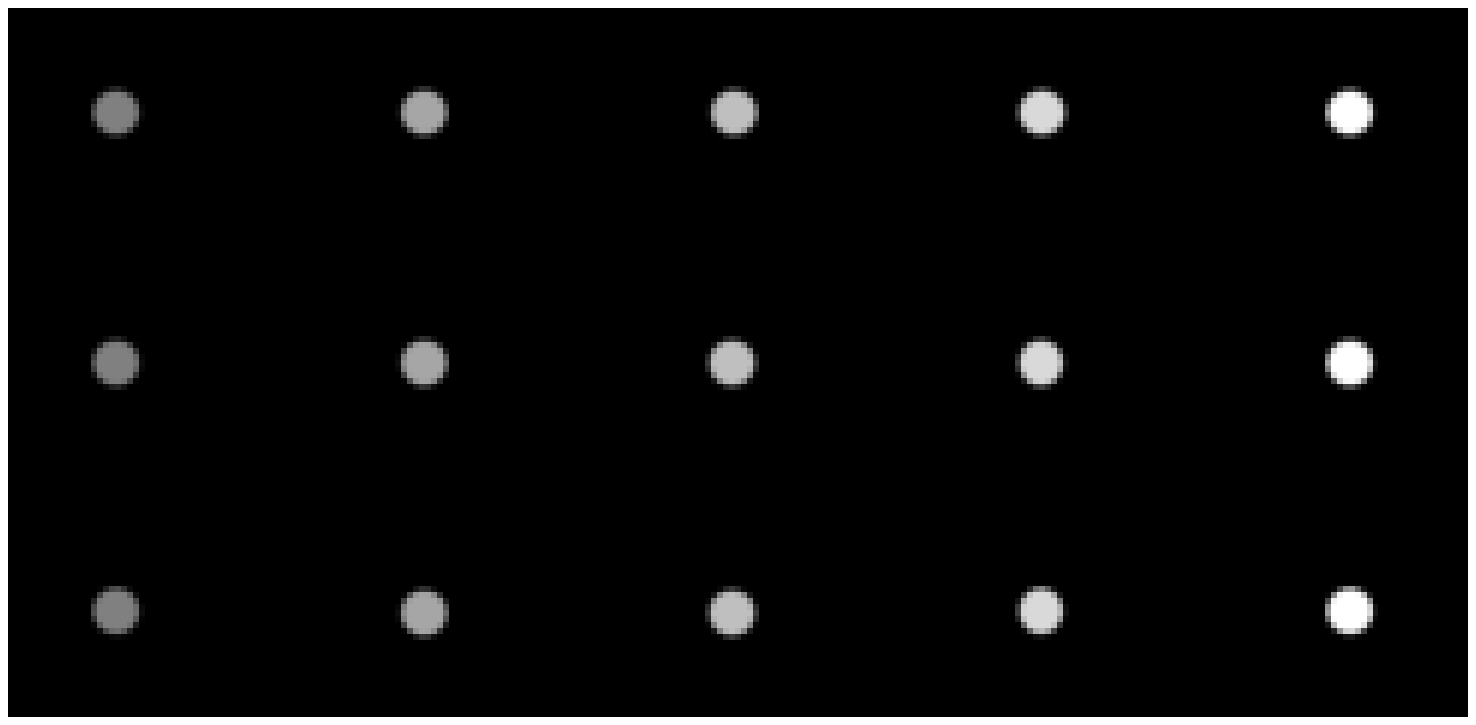}} &
\subfloat[Texture 2 - original]{\includegraphics[width=1.6in]{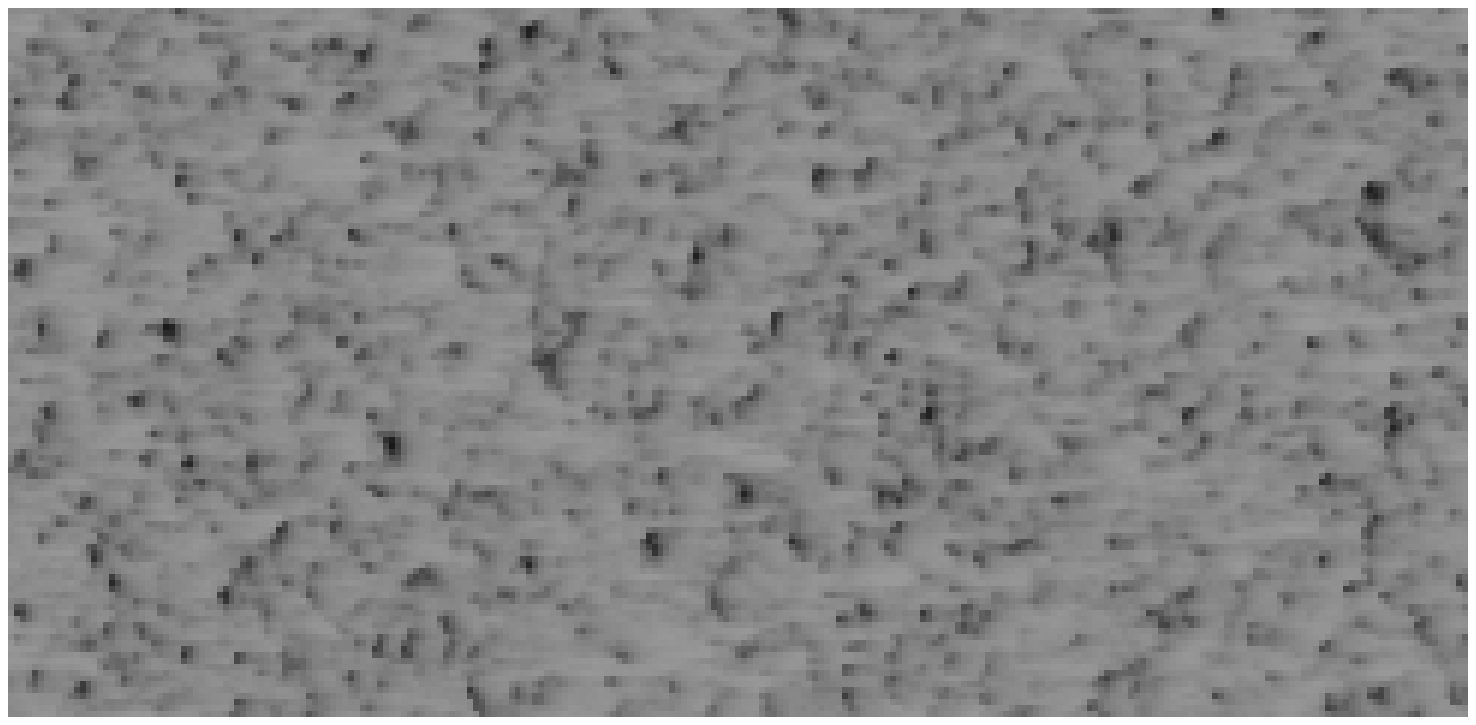}} \\
\subfloat[Decomposed layer 1 - proposed]{\includegraphics[width=1.6in]{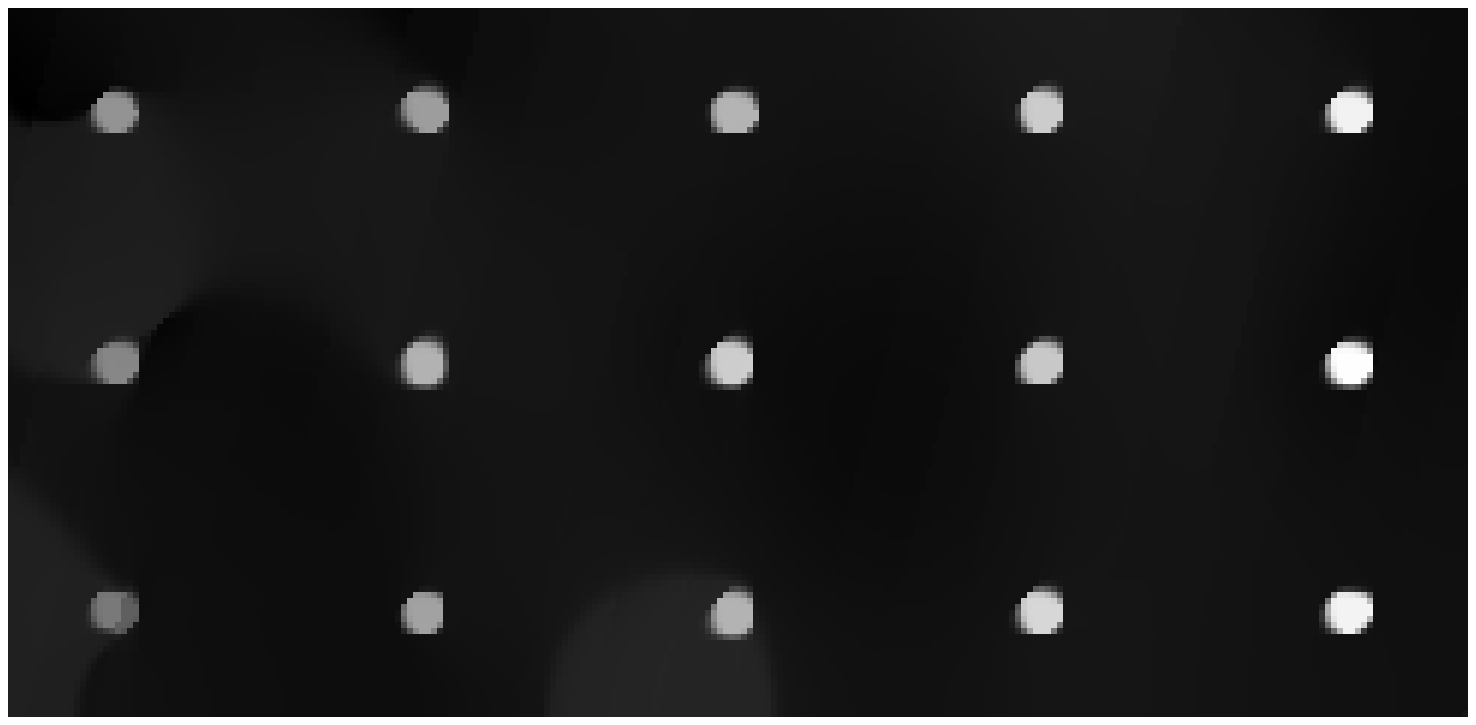}} &
\subfloat[Decomposed layer 2 - proposed]{\includegraphics[width=1.6in]{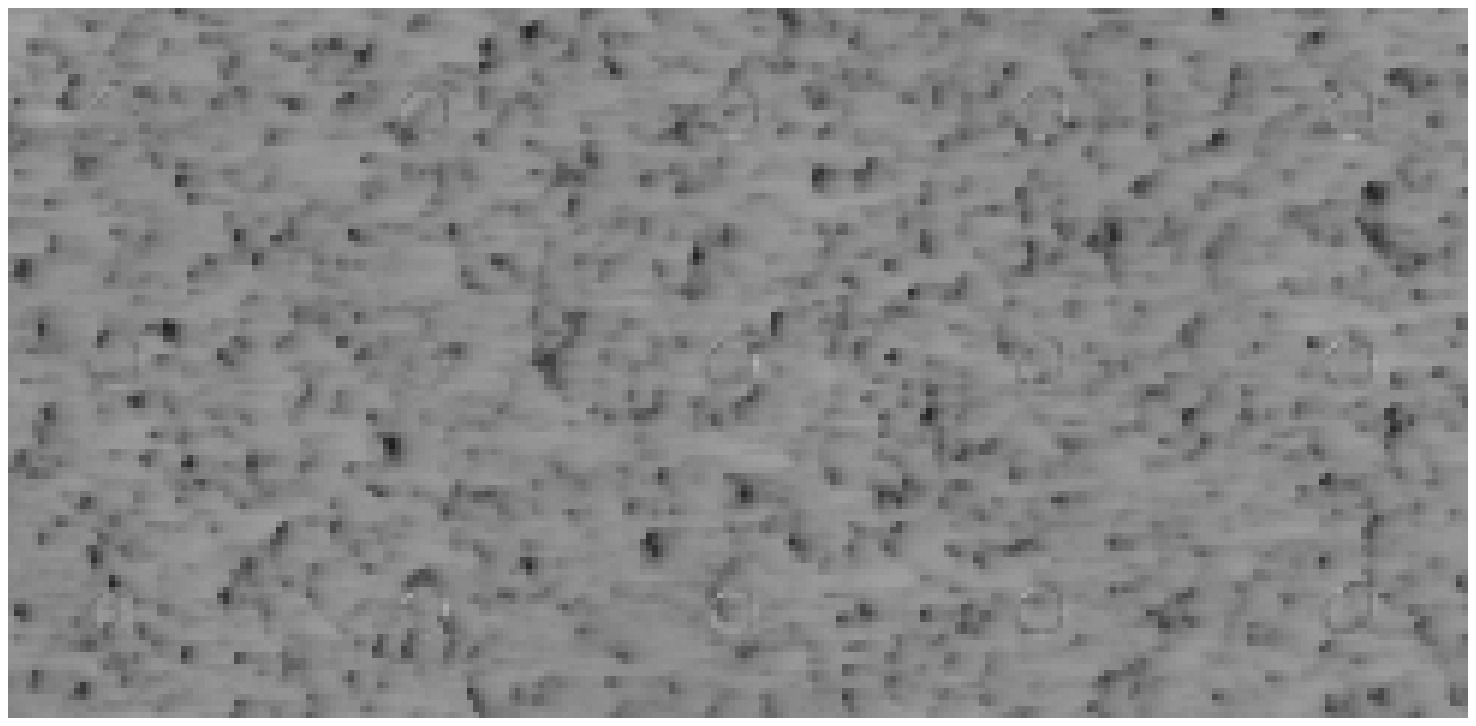}} \\
\subfloat[TV-G, layer 1]{\includegraphics[width=1.6in]{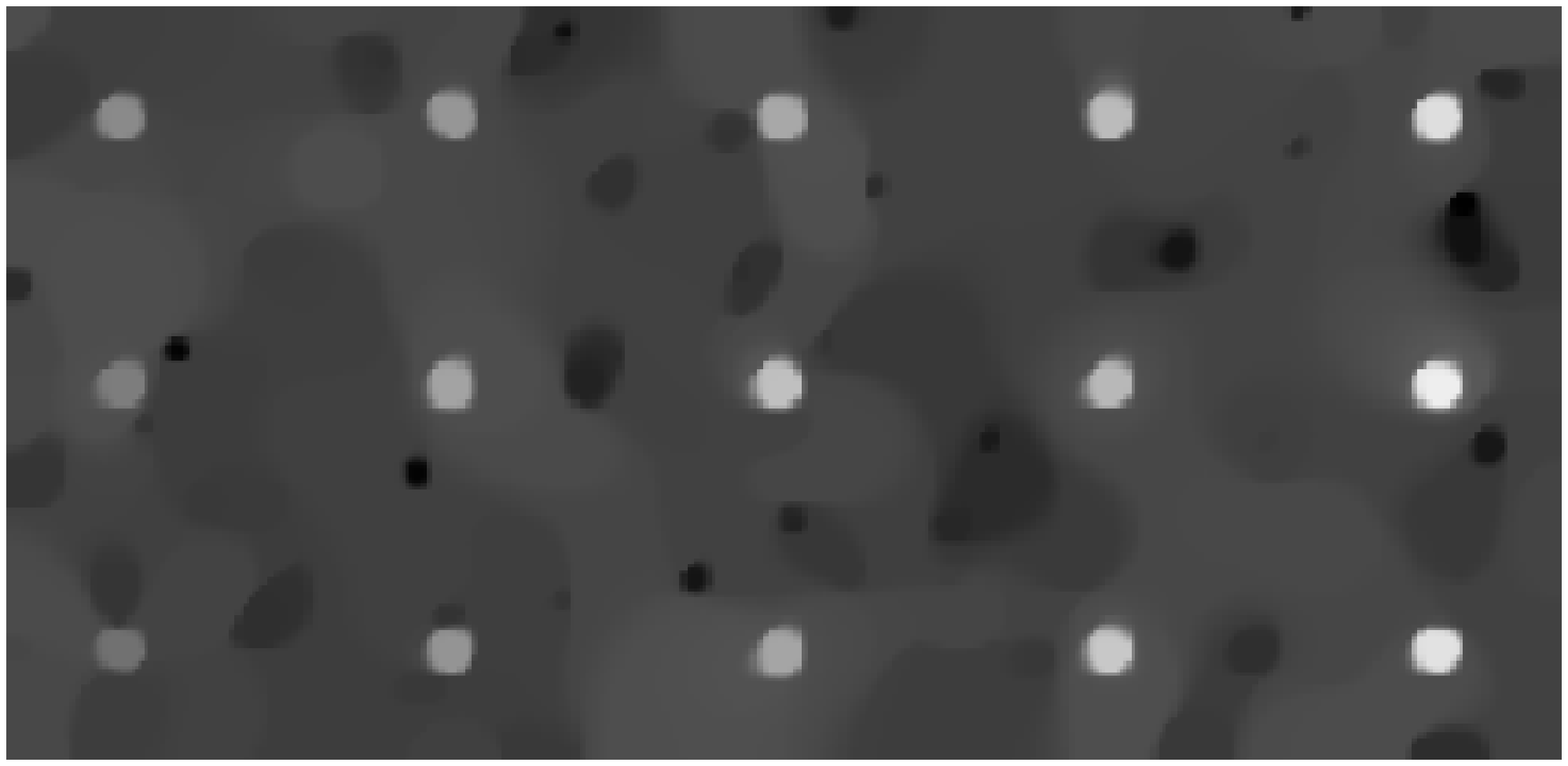}} &
\subfloat[TV-G, layer 2]{\includegraphics[width=1.6in]{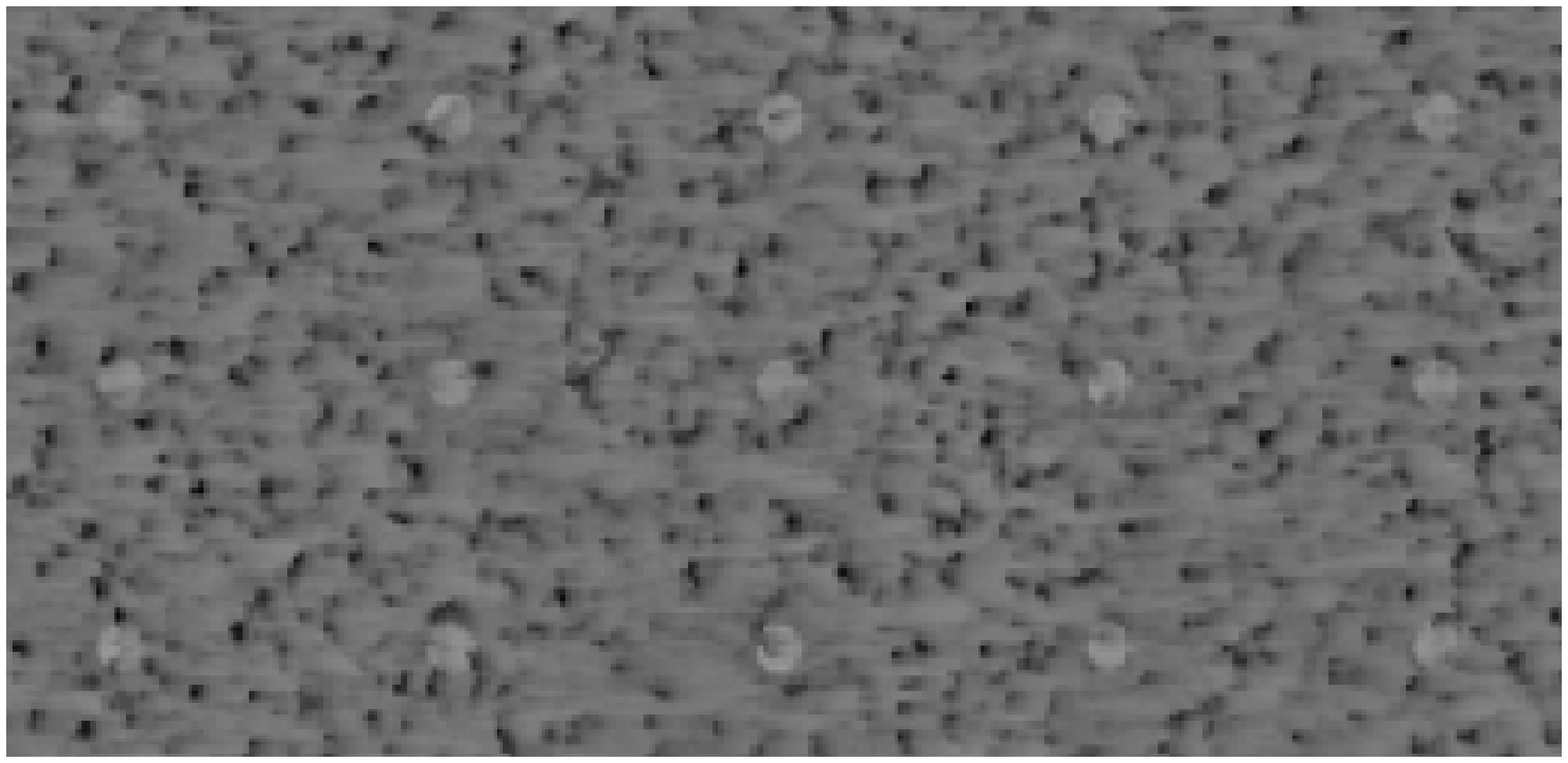}} \\
\subfloat[RGF, layer 1]{\includegraphics[width=1.6in]{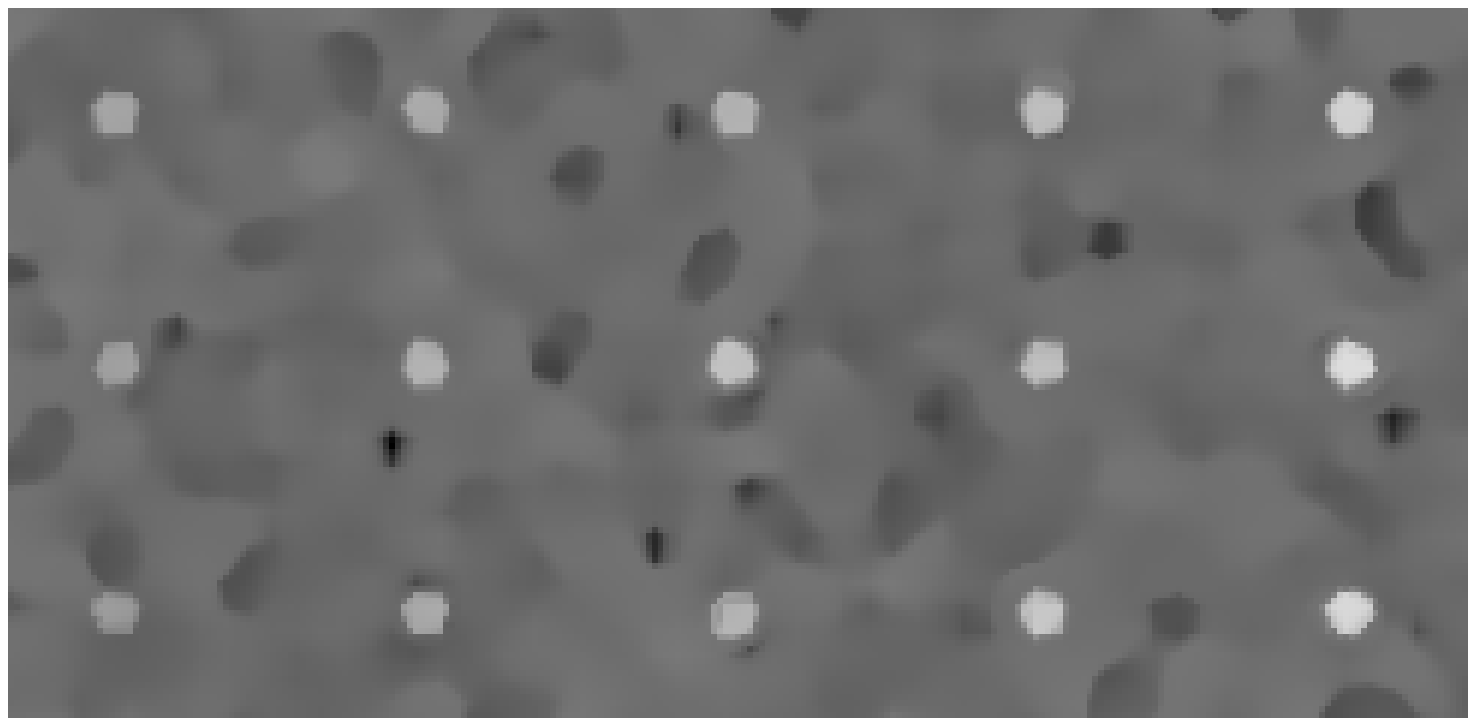}} &
\subfloat[RGF, layer 2]{\includegraphics[width=1.6in]{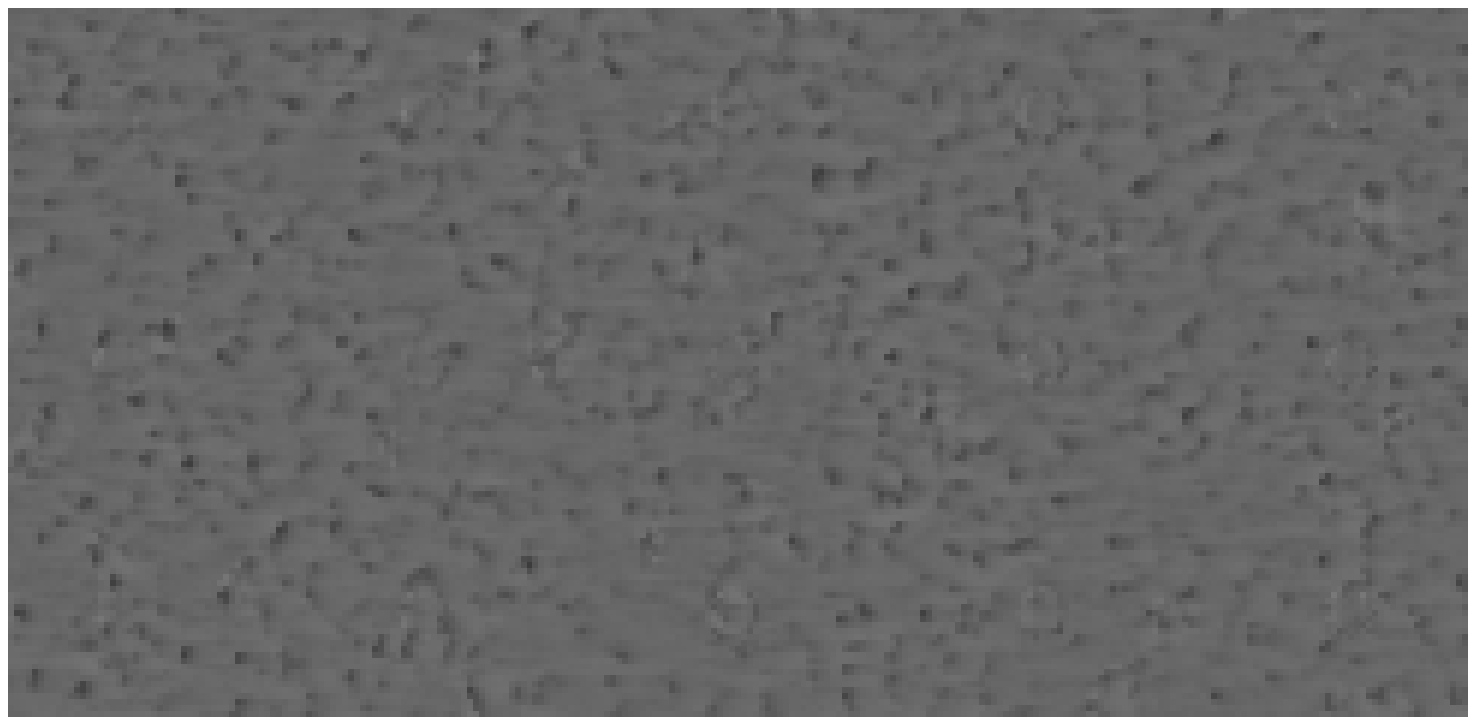}} \\
\end{tabular}
\caption {Separation plane of a synthetic example}
\label{fig:text1504_result}
\vspace{-10pt}
\end{figure}

\begin{figure}
\centering
\begin{tabular}{cc}
\subfloat[Input image]{\includegraphics[width=1.6in]{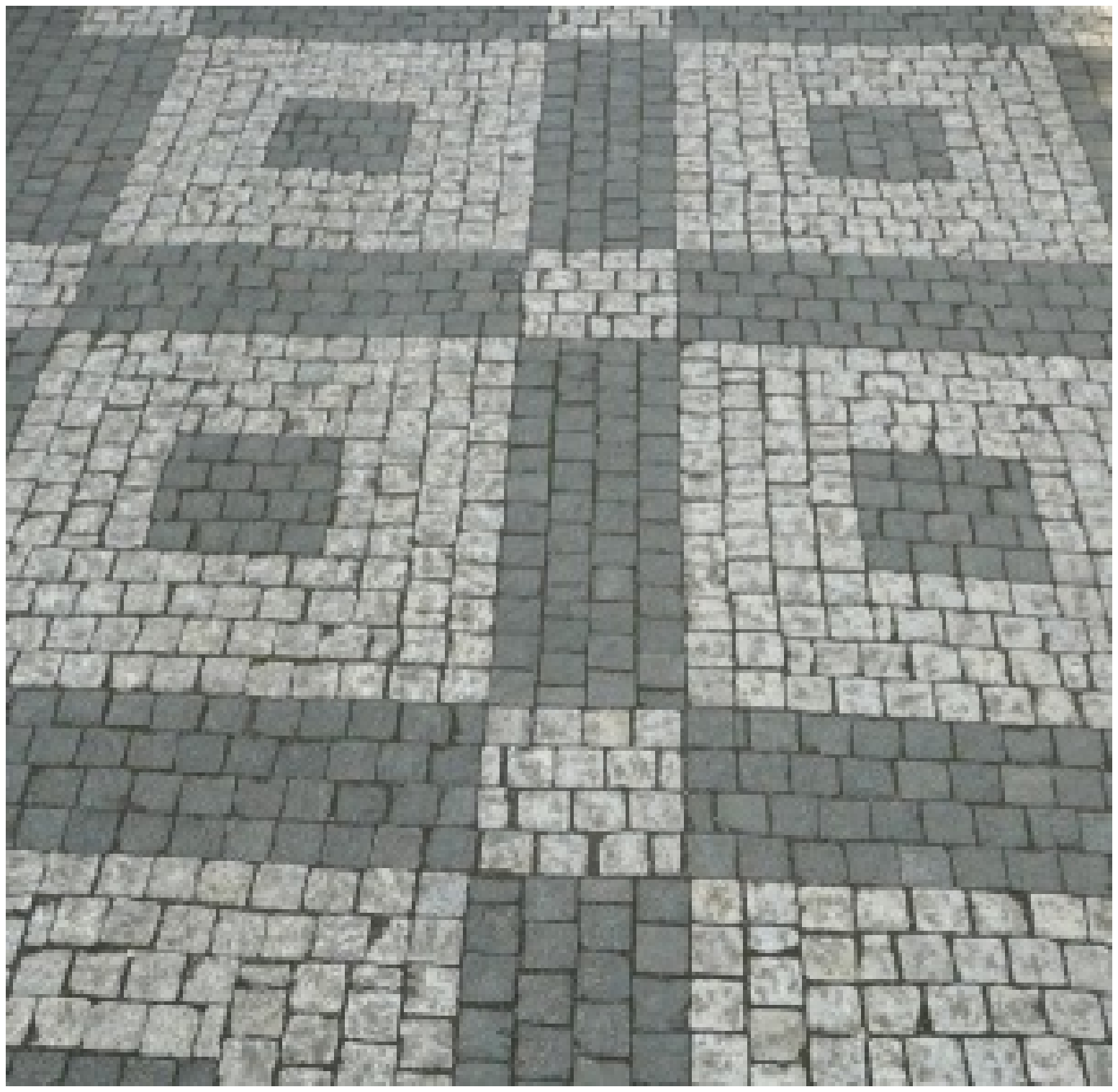}} &
\subfloat[Max. Phi time with separation band]{\includegraphics[width=1.6in]{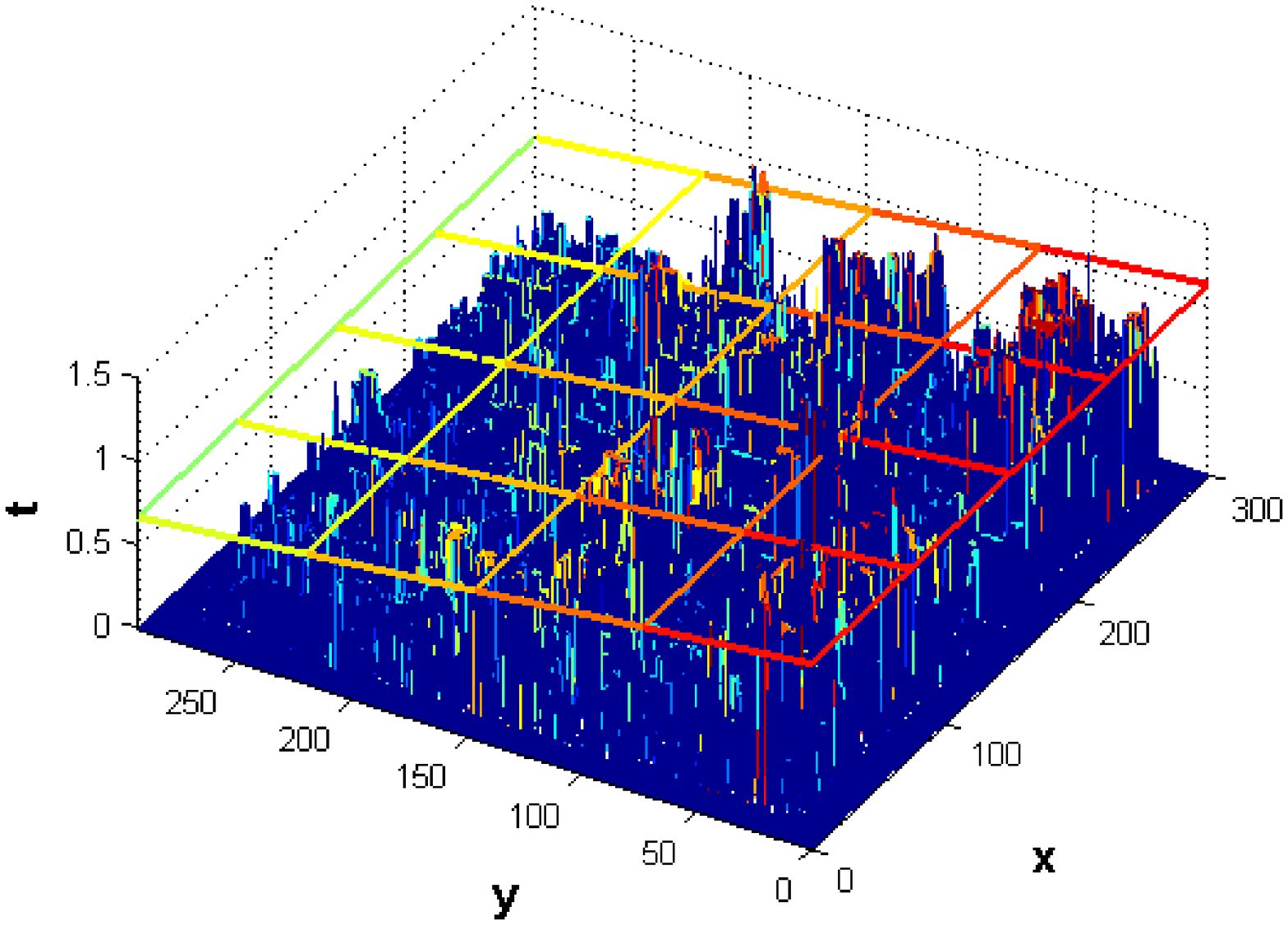}} \\

\subfloat[Proposed texture]{\includegraphics[width=1.6in]{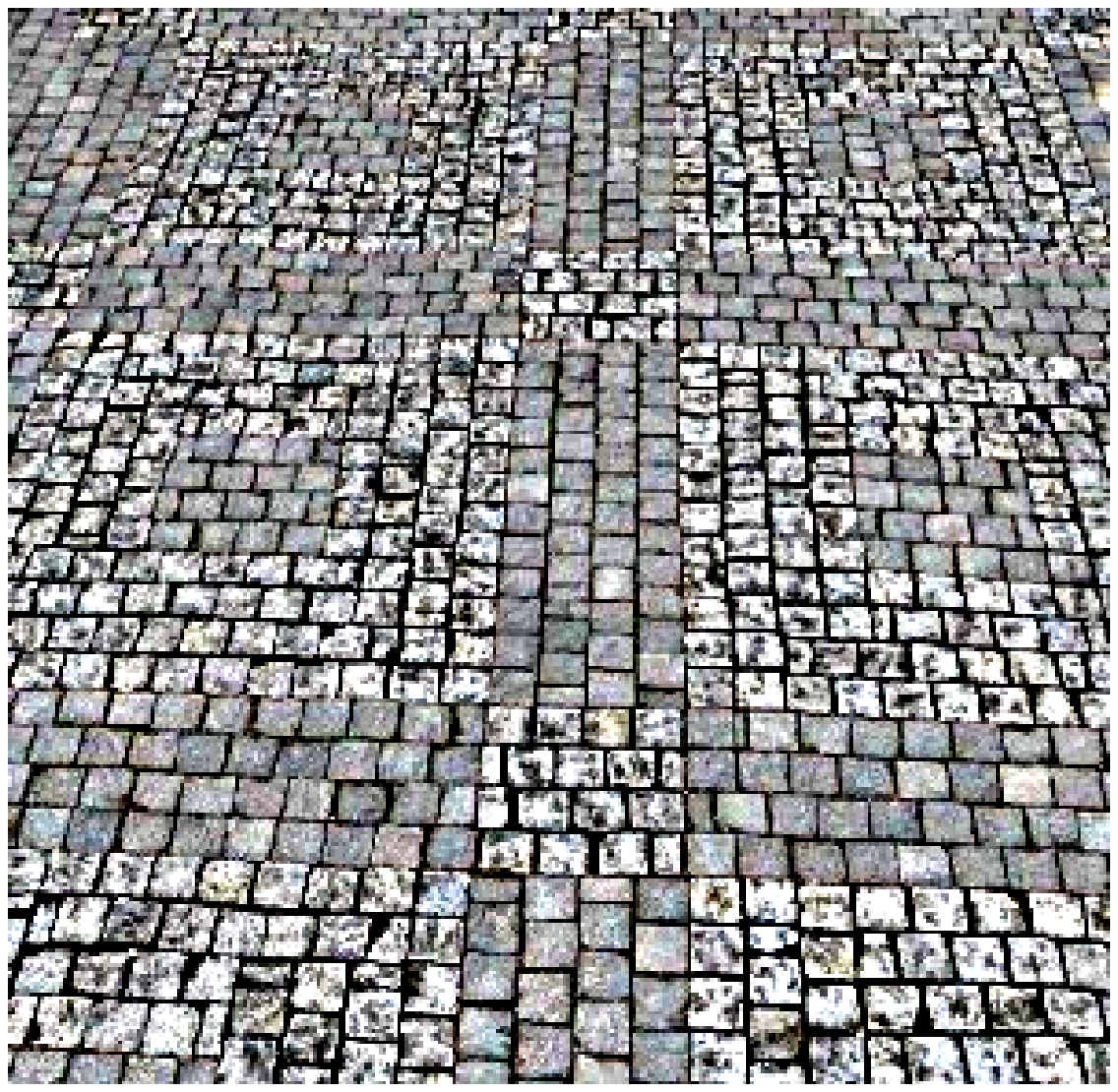}} &
\subfloat[Proposed structure]{\includegraphics[width=1.6in]{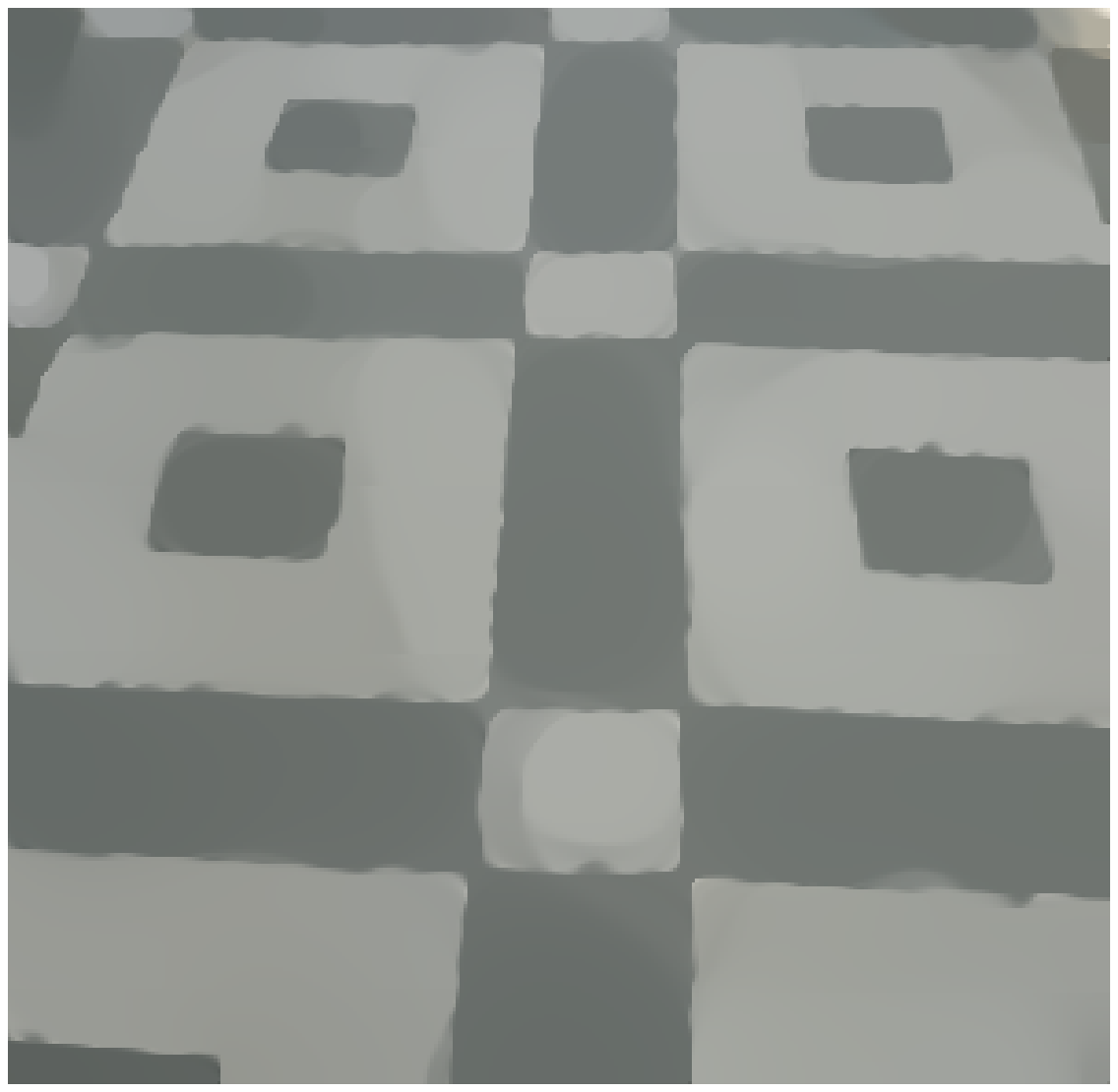}} \\

\subfloat[RGF decomposition: scale 2]{\includegraphics[width=1.6in]{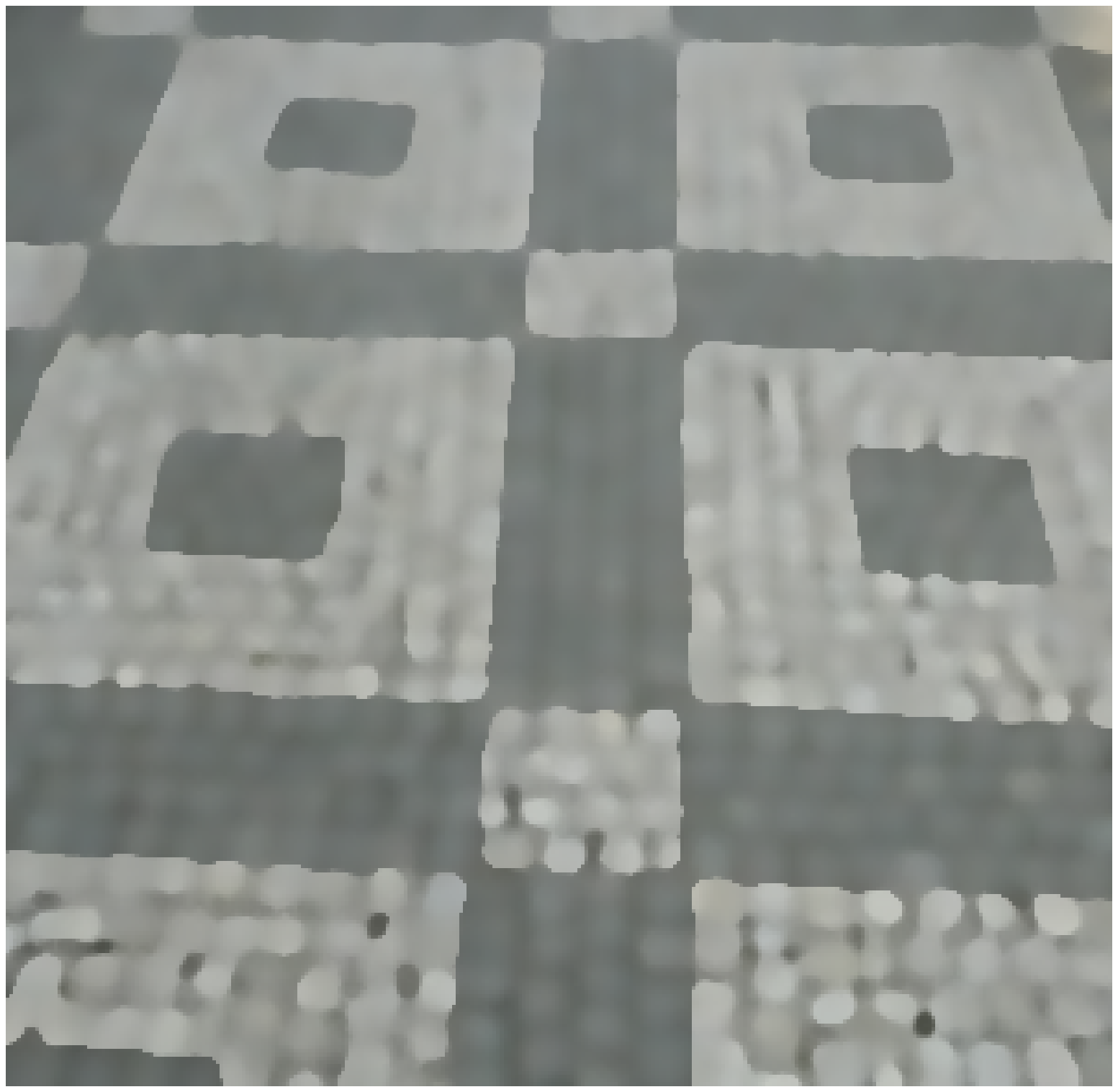}} &
\subfloat[RGF decomposition: scale 3]{\includegraphics[width=1.6in]{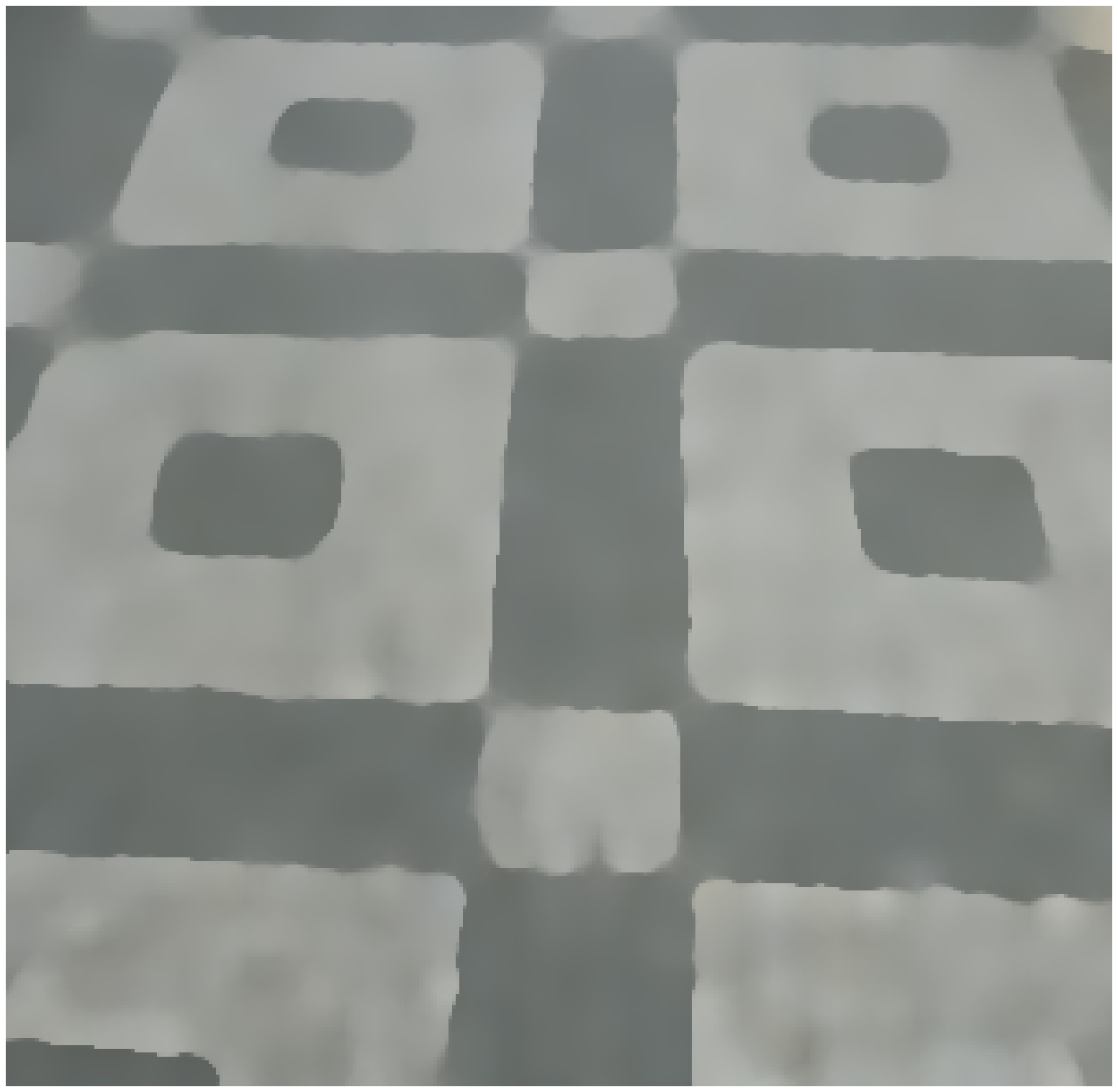}} \\
\end{tabular}
\caption {Multiscale separation of Street tiles image and comparison to RGF method}
\label{fig:tiles}
\end{figure}

\subsection{Generalization to Surfaces}

\begin{figure}
\centering
\begin{tabular}{cc}
\subfloat[Input image]{\includegraphics[height=1.1in]{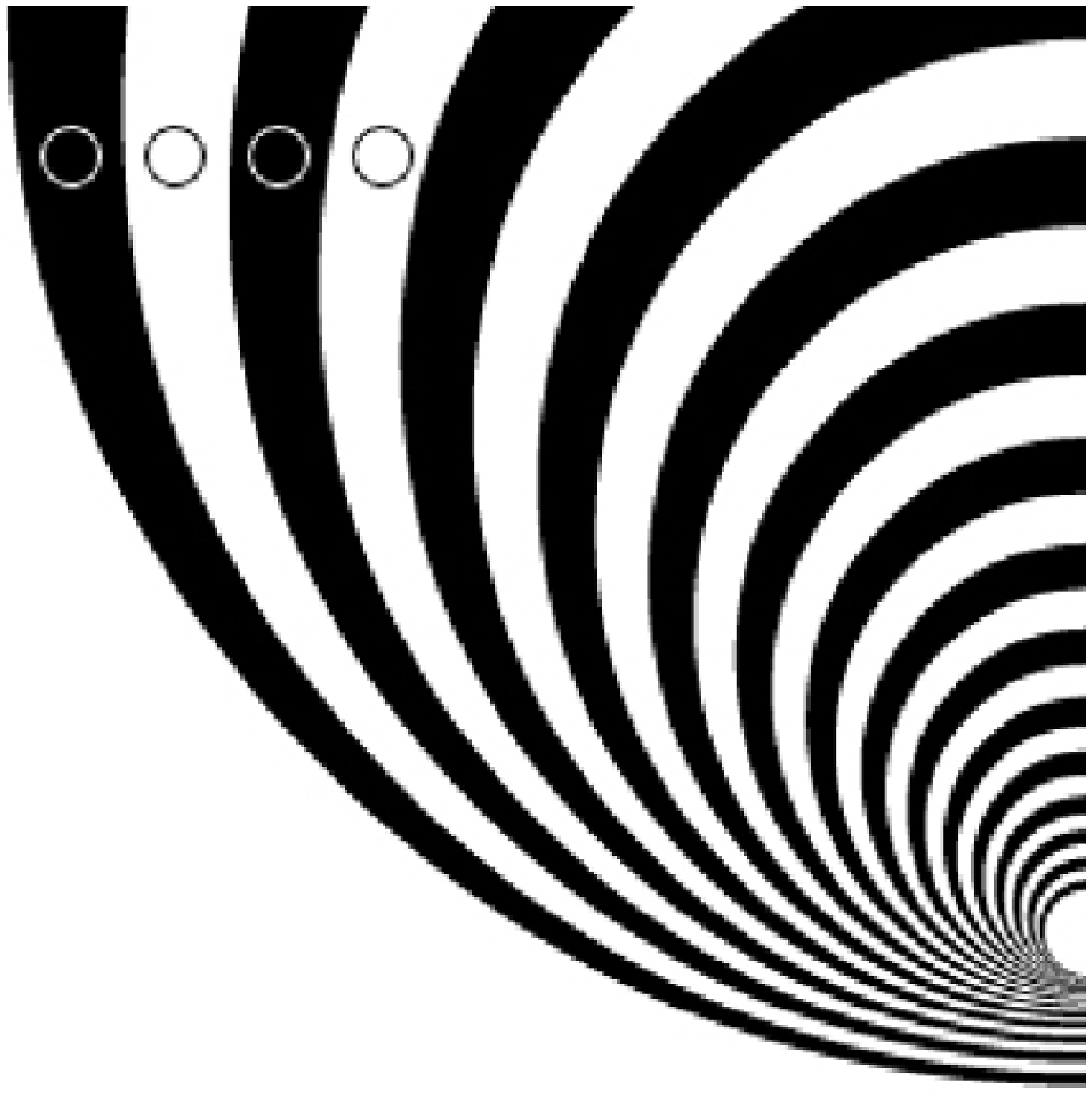}} &
\subfloat[Max. $\phi$ time]{\includegraphics[height=1.2in]{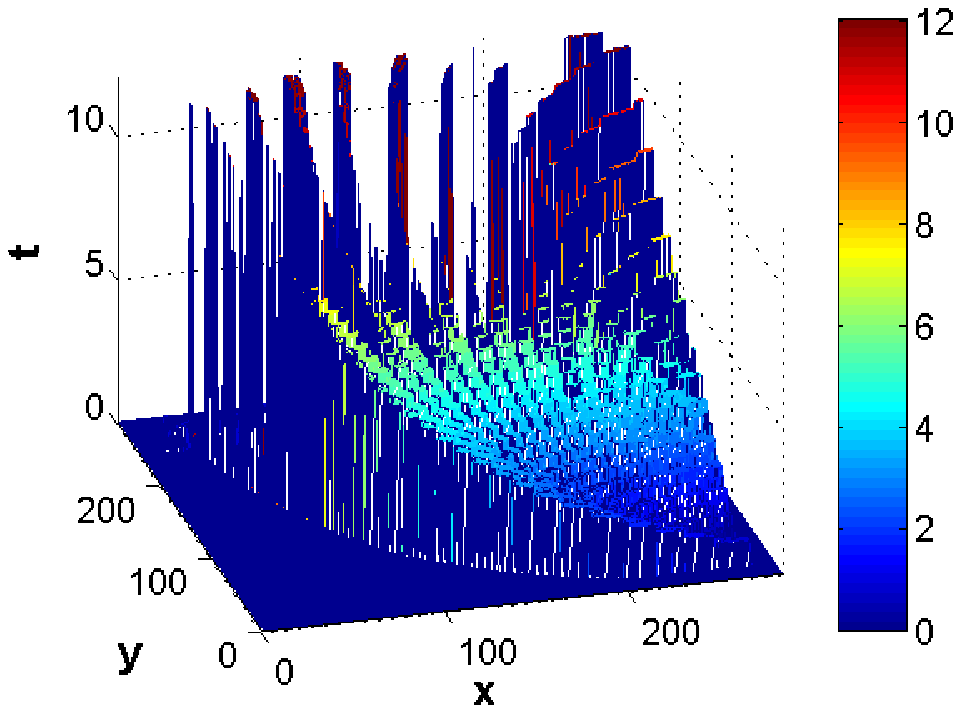}} \\
\subfloat[Texture surface]{\includegraphics[height=1.2in]{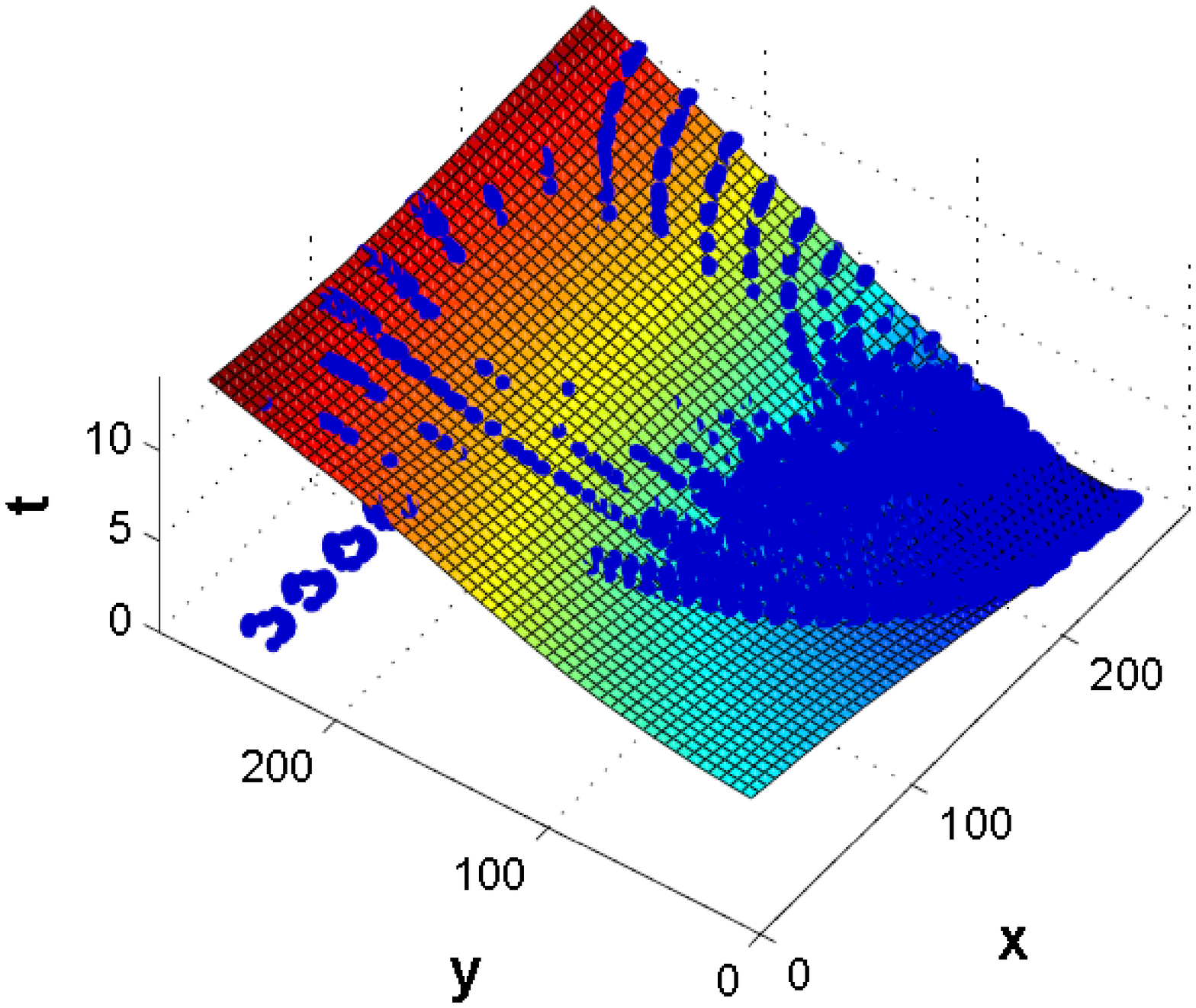}} &
\subfloat[Separation band]{\includegraphics[height=1.2in]{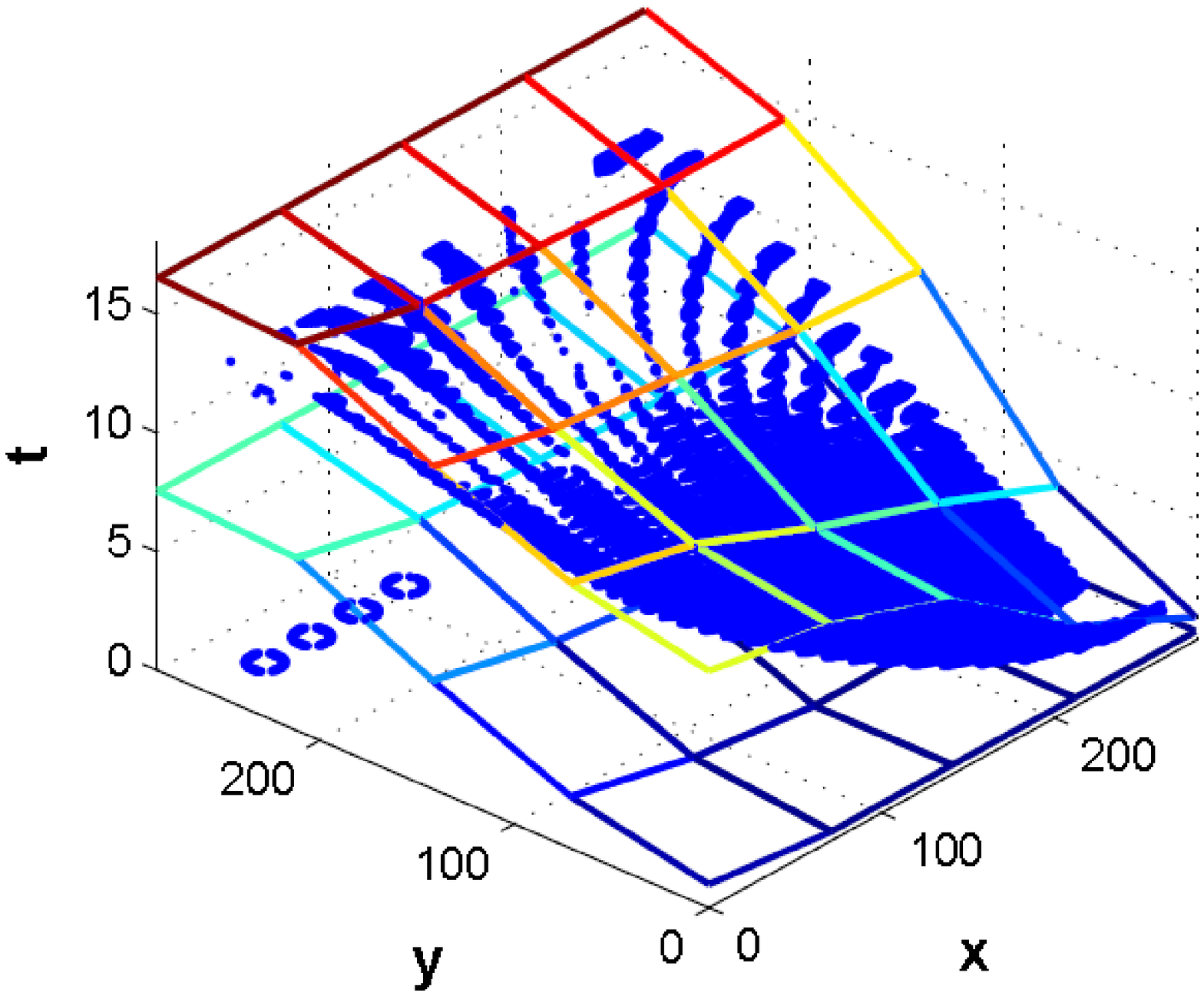}} \\
\subfloat[Decomposed layer 1 - proposed]{\includegraphics[height=1.1in]{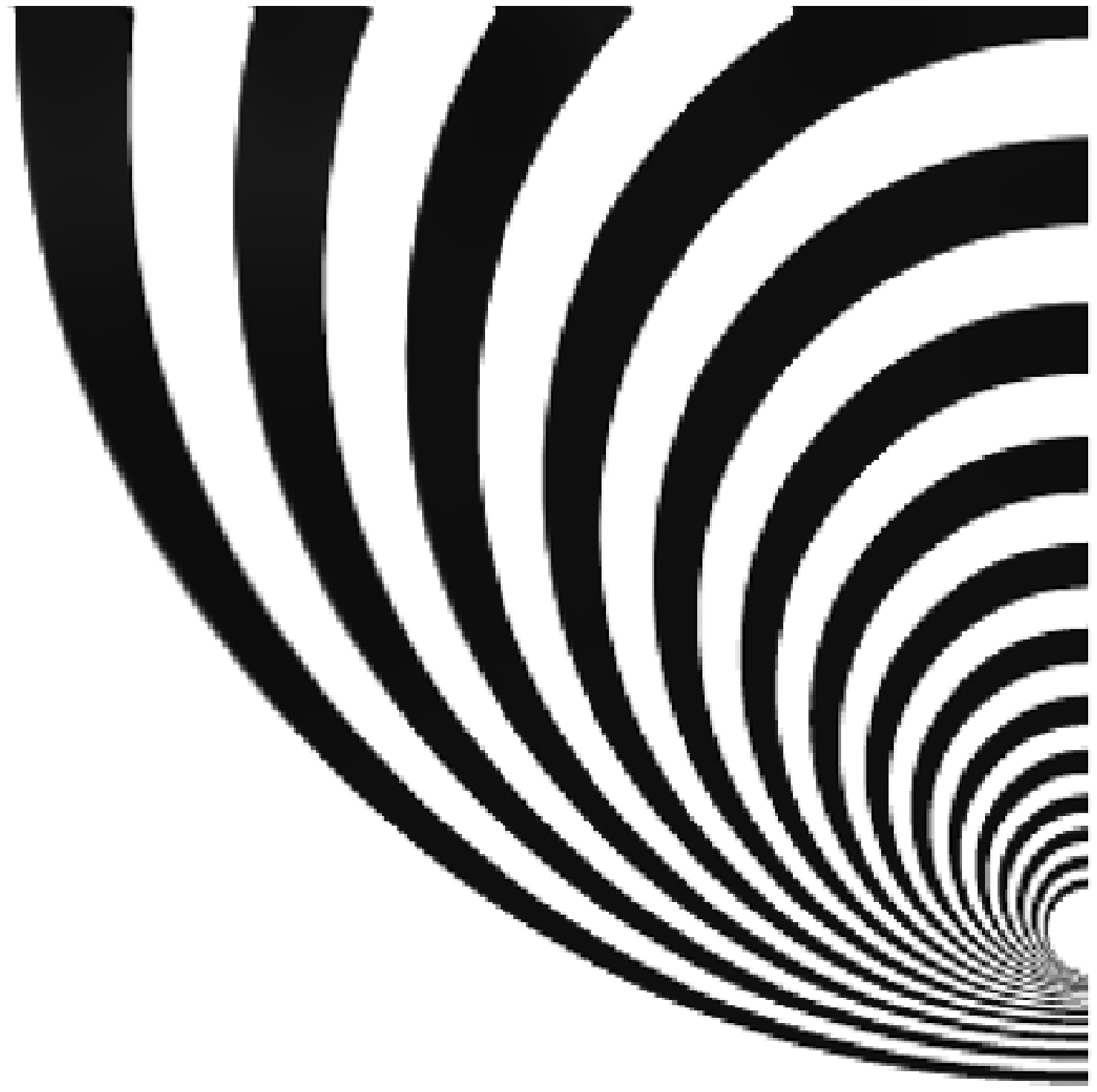}} &
\subfloat[Decomposed layer 2 - proposed]{\includegraphics[height=1.1in]{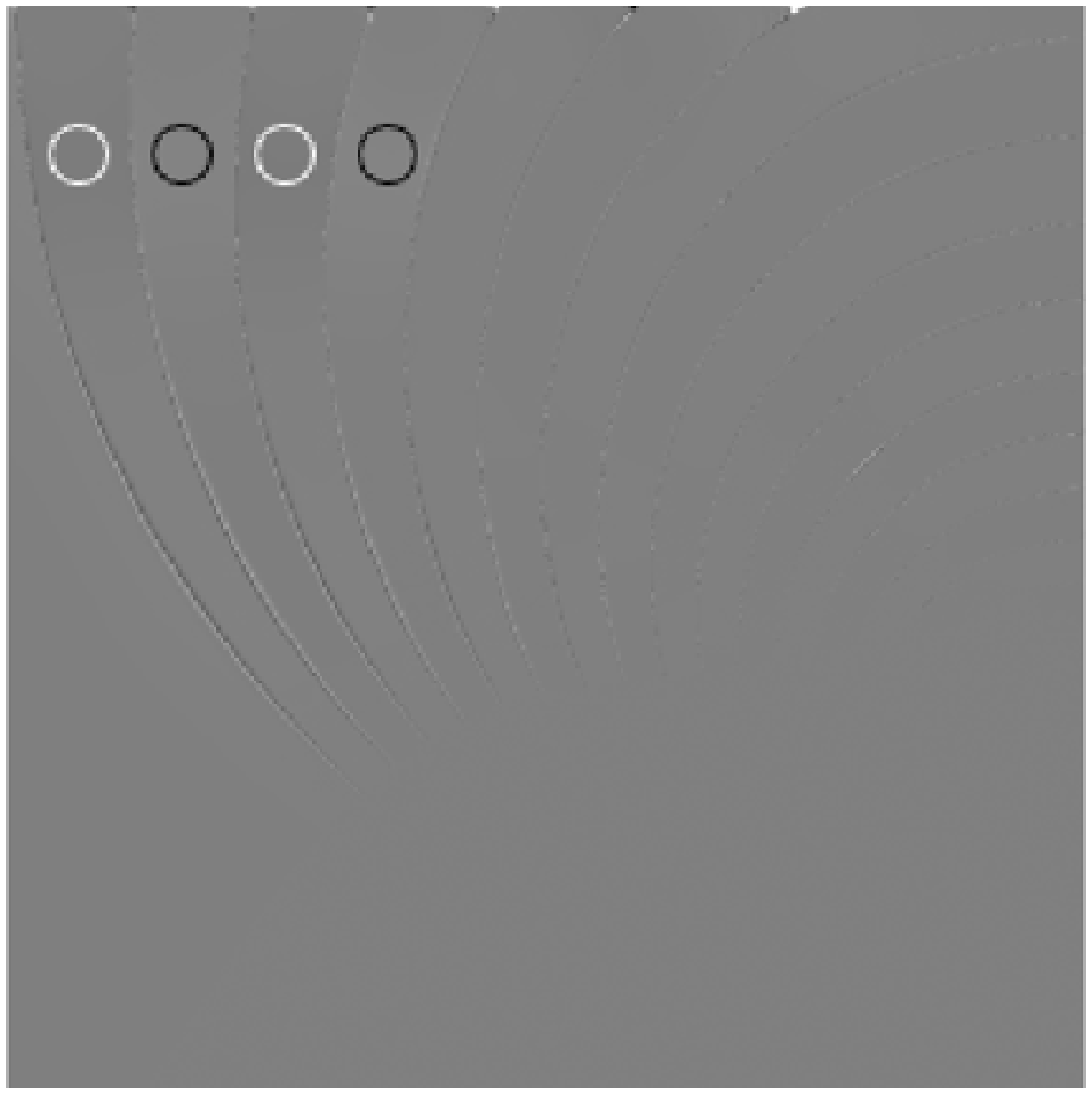}} \\

\subfloat[$t_s=0.6$, layer 1]{\includegraphics[height=1.1in]{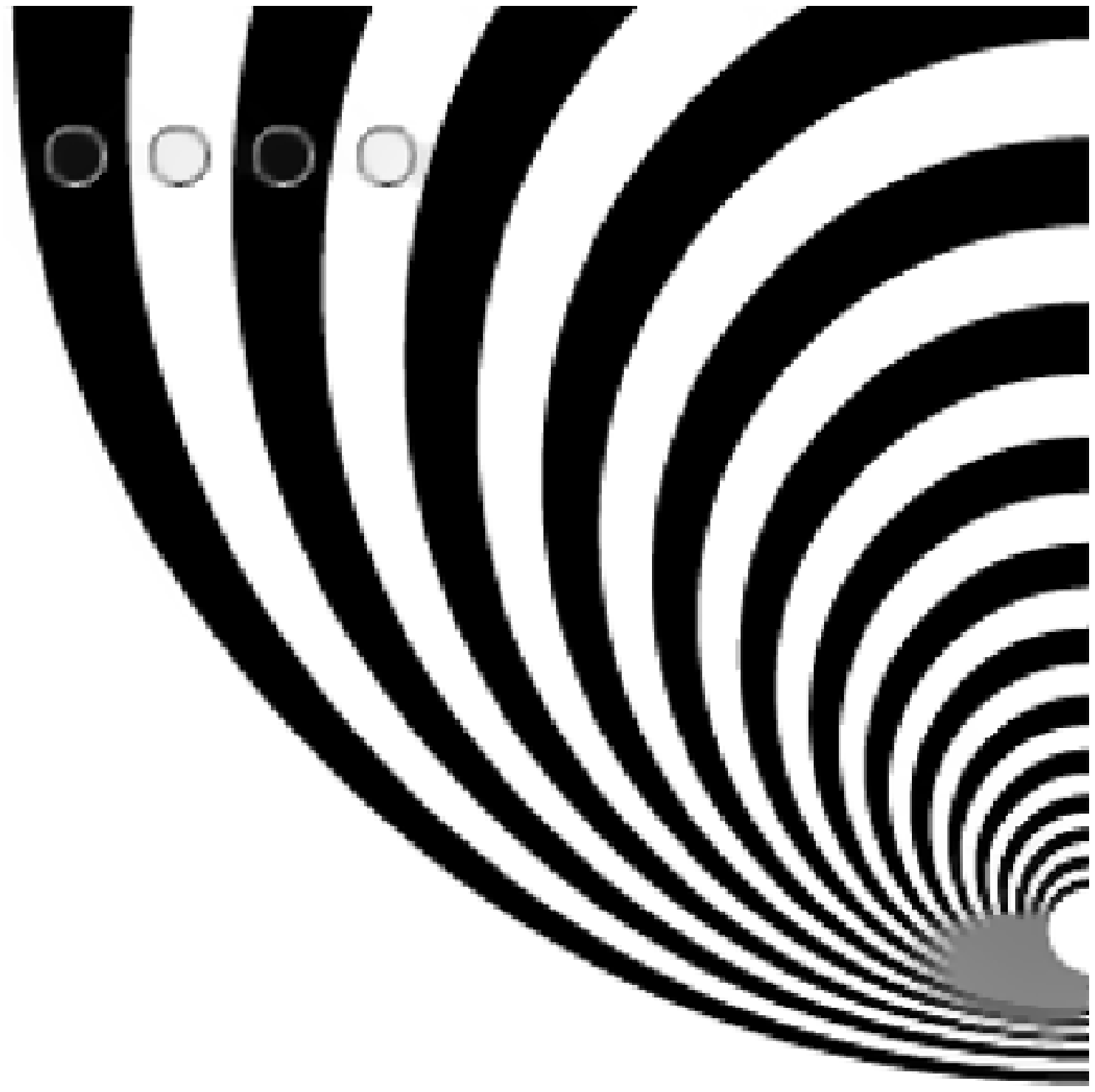}} &
\subfloat[$t_s=0.6$, layer 2]{\includegraphics[height=1.1in]{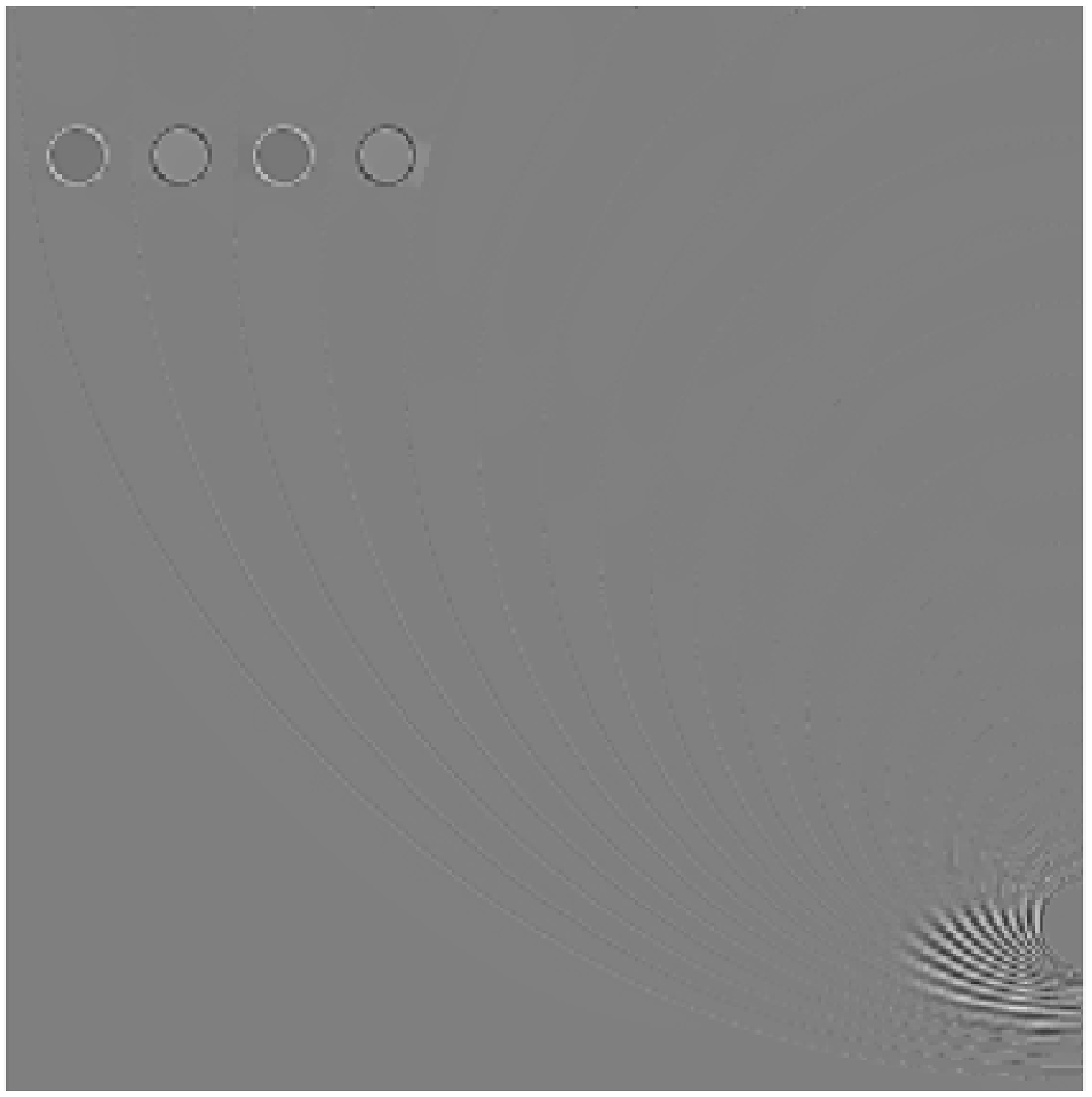}} \\
\subfloat[TV-G, layer 1]{\includegraphics[height=1.1in]{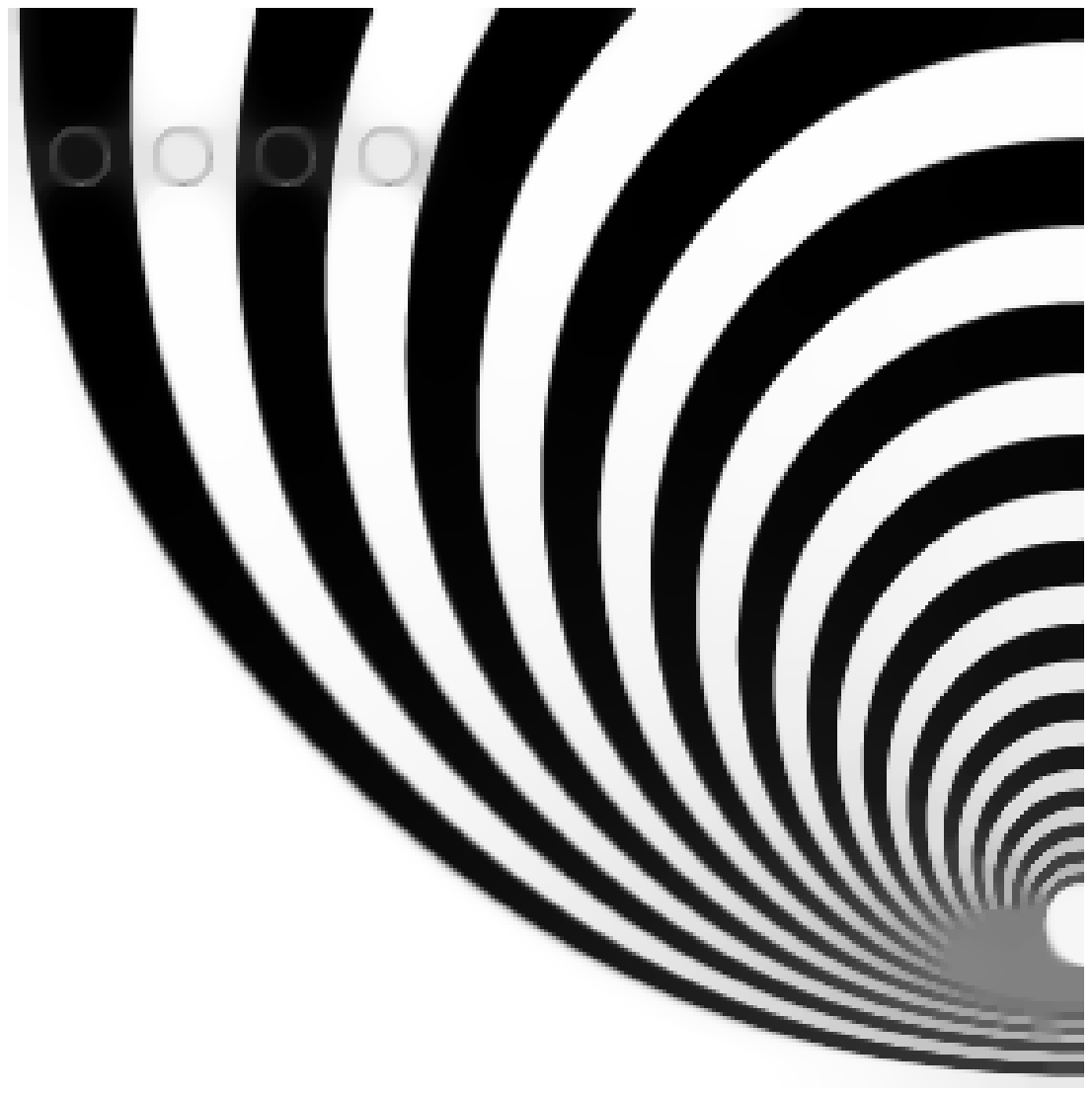}} &
\subfloat[TV-G, layer 2]{\includegraphics[height=1.1in]{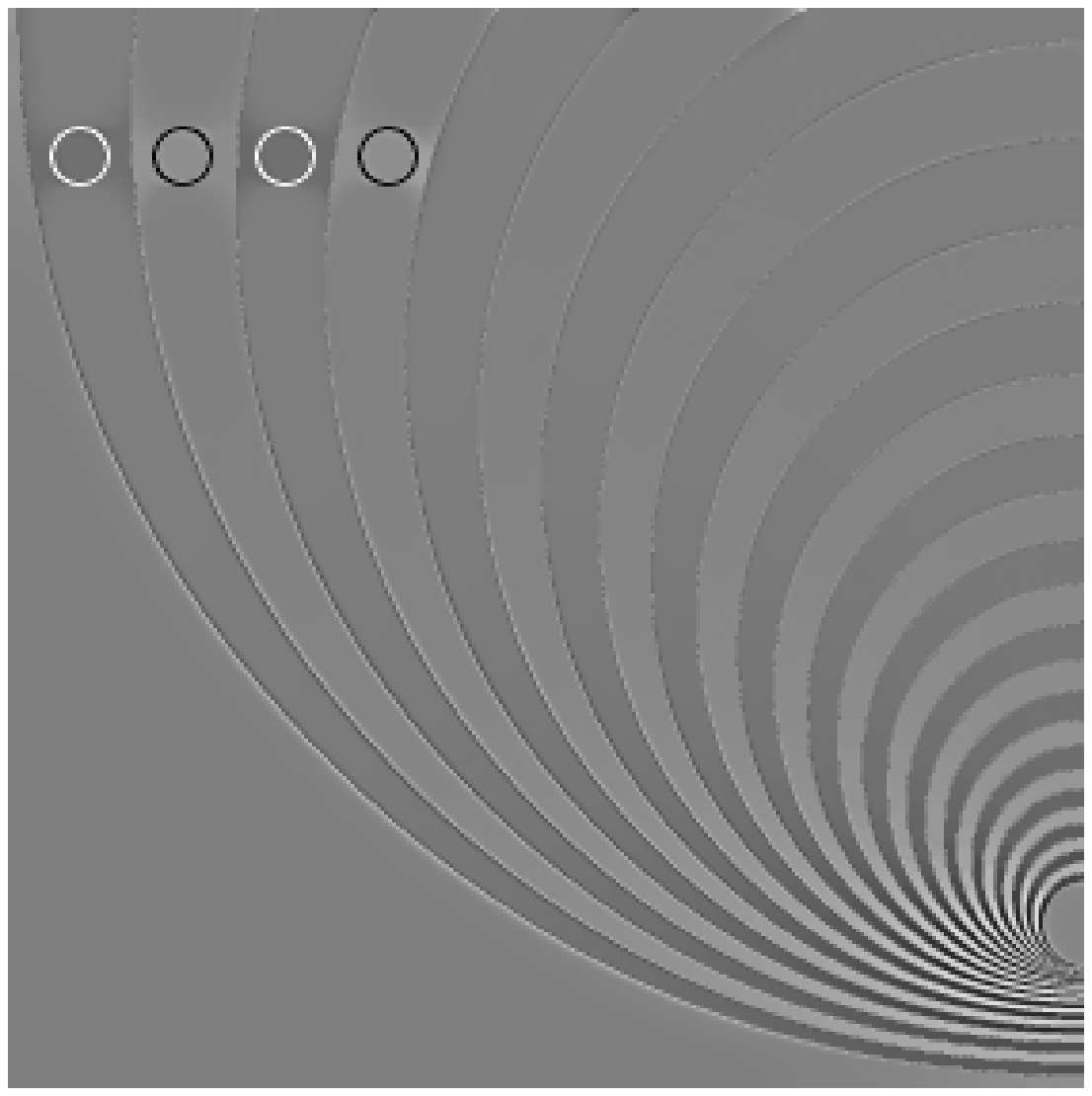}} \\
\subfloat[RGF, layer 1]{\includegraphics[height=1.1in]{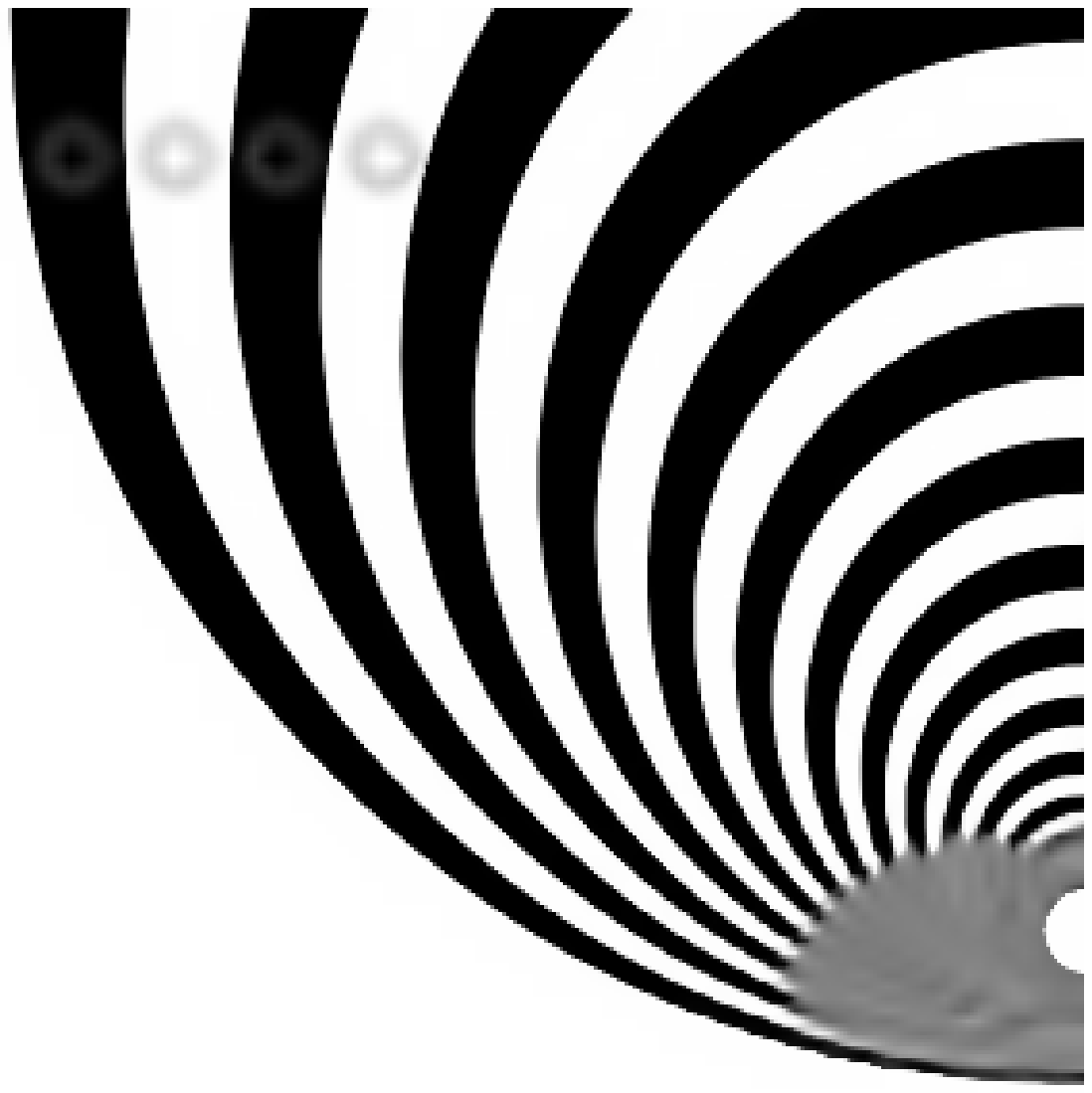}} &
\subfloat[RGF, layer 2]{\includegraphics[height=1.1in]{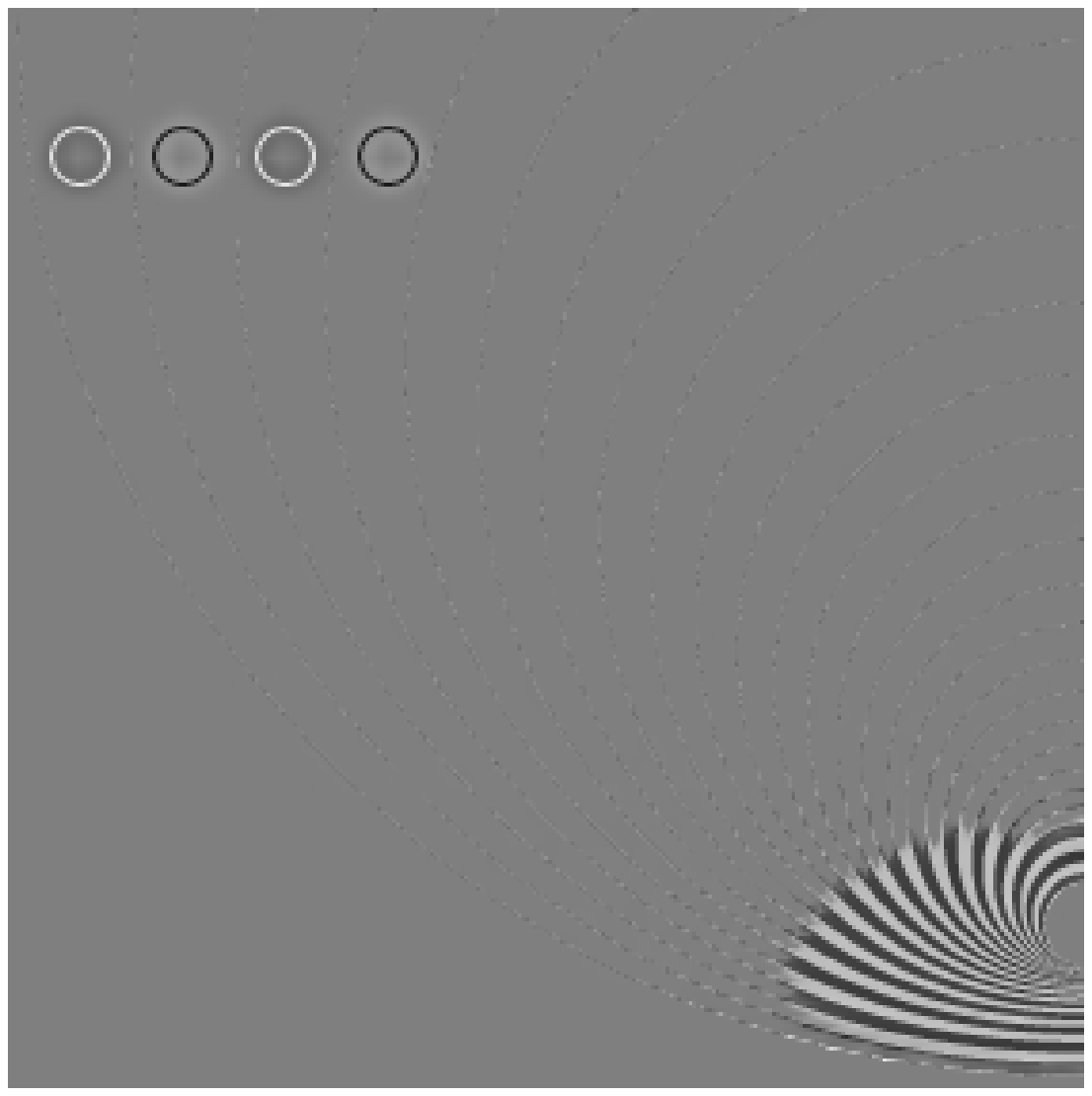}} \\
\end{tabular}
\caption {Decomposition of concentric circles image using a separation surface and comparison to TV-G and RGF.}
\label{fig:conc_result}
\end{figure}

We now show the general approach of separation surface or stratum extraction. We use this general approach when desired texture is more complex and does not linearly change in the image space. In this case we will perform the exact same actions to find the separation plane, but instead of fitting a plane to the texture filtered max. time data, we fit a surface.
We present here a simple method for surface fitting, using local linear regression, in Fig \ref{fig:conc_result}. In this example, we added 4 disks to an image of concentric circles with changing radius. We can see in the separation band image that the disks and the fine scales circles are at the same spectral scale, however, their spatial location in the image is different, and so, using a separation surface, the two textures are perfectly separated. We can see a comparison to separation at a specific time (scale) for the whole image, as was done in our previous work \cite{horesh2015multiscale}. In this case, while not all the disks features are captured in layer 2, the circles features are already contained in that layer, meaning that separating at any certain scale can not  bring to a perfect separation. More comparisons are to TV-G and to the state of the art RGF \cite{zhang2014rolling}, which both fail  getting good separating result.
Another example can be seen in the algorithm description (Fig. \ref{fig:algorithm}), where in the zebra image, the stripes were very nicely separated from its structure.
The separation surface can be optimally found  in different manners, according to the texture's features. One of them is to use the Gabor filters in order to find a separation surface for textures with definite texture orientation.
Another option is to use the Gaussian Mixture Model in order to differentiate textures with different distribution in space and time, and adapt the separation surface to it.



\section{Applications}
\label{sec:app}

 \subsection{Texture Manipulation}
Following an efficient texture separation, we can manipulate the different image layers in order to create sub-images, with enhanced texture or reduced texture. For example, in Fig \ref{fig:faces_enhance} we separate the faces only, maintaining the sharp stones' borders, and then, manipulate their level in the output image to get enhanced  or depressed facial features. In Fig. \ref{fig:chess_enhance} another texture manipulation example is shown, where the wood texture in the game board image is depressed and enhanced to a desired level. In the zebra image in Fig. \ref{fig:zebra_enhance} the stripes were extracted as depicted in Fig. \ref{fig:algorithm}, then by using a mask of the zebra itself, the stripes were enhanced and then inverted, so that the brown and white colors were replaced. Note that because of the stratum definition, including mostly the stripes, the other zebra features, such as nose and eyes maintained their color.

\begin{figure}
\centering
\begin{tabular}{cc}
\subfloat[Input Image]{\includegraphics[width=1.4in]{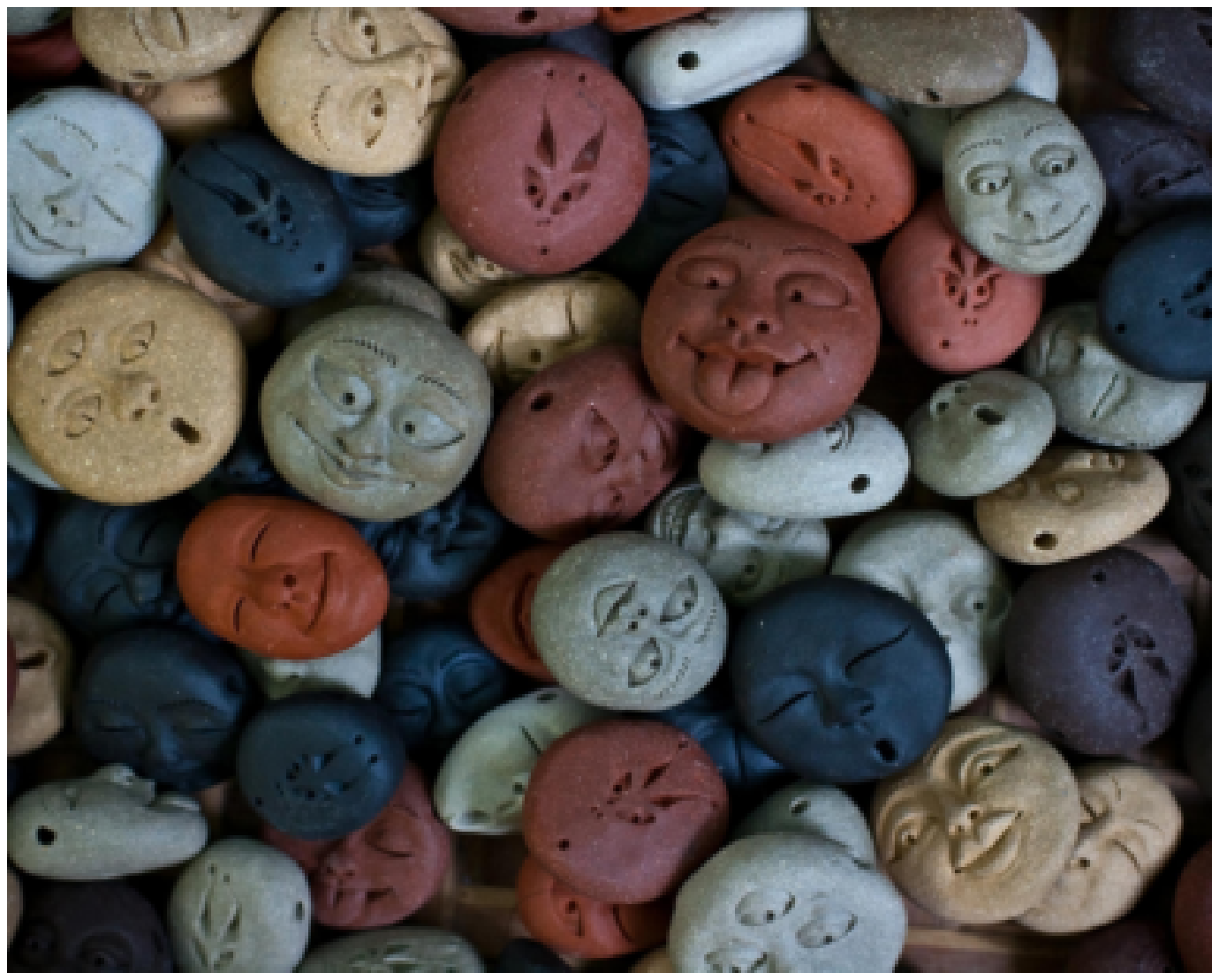}} &
\subfloat[Desired Texture]{\includegraphics[width=1.4in]{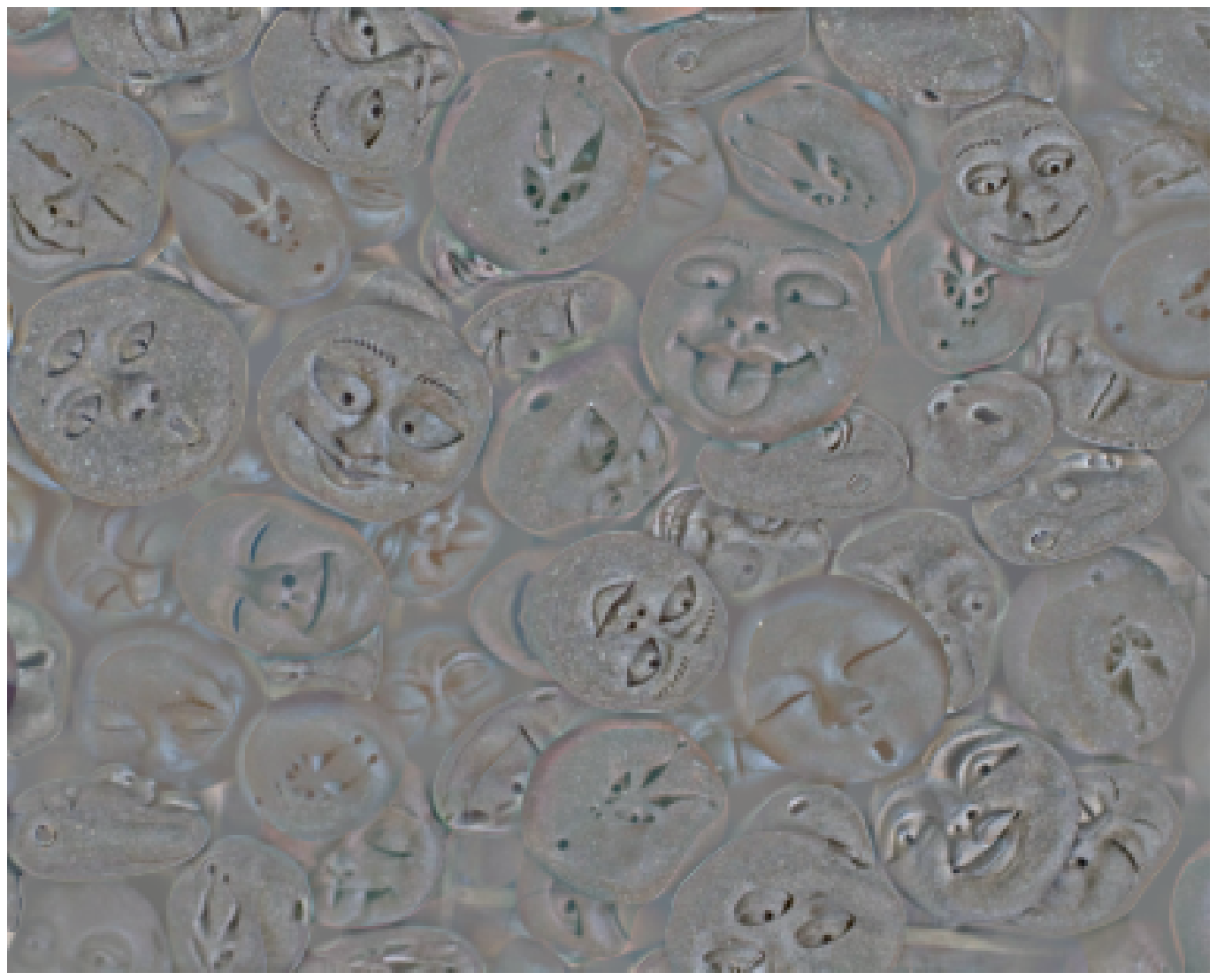}} \\
\subfloat[Attenuated texture]{\includegraphics[width=1.4in]{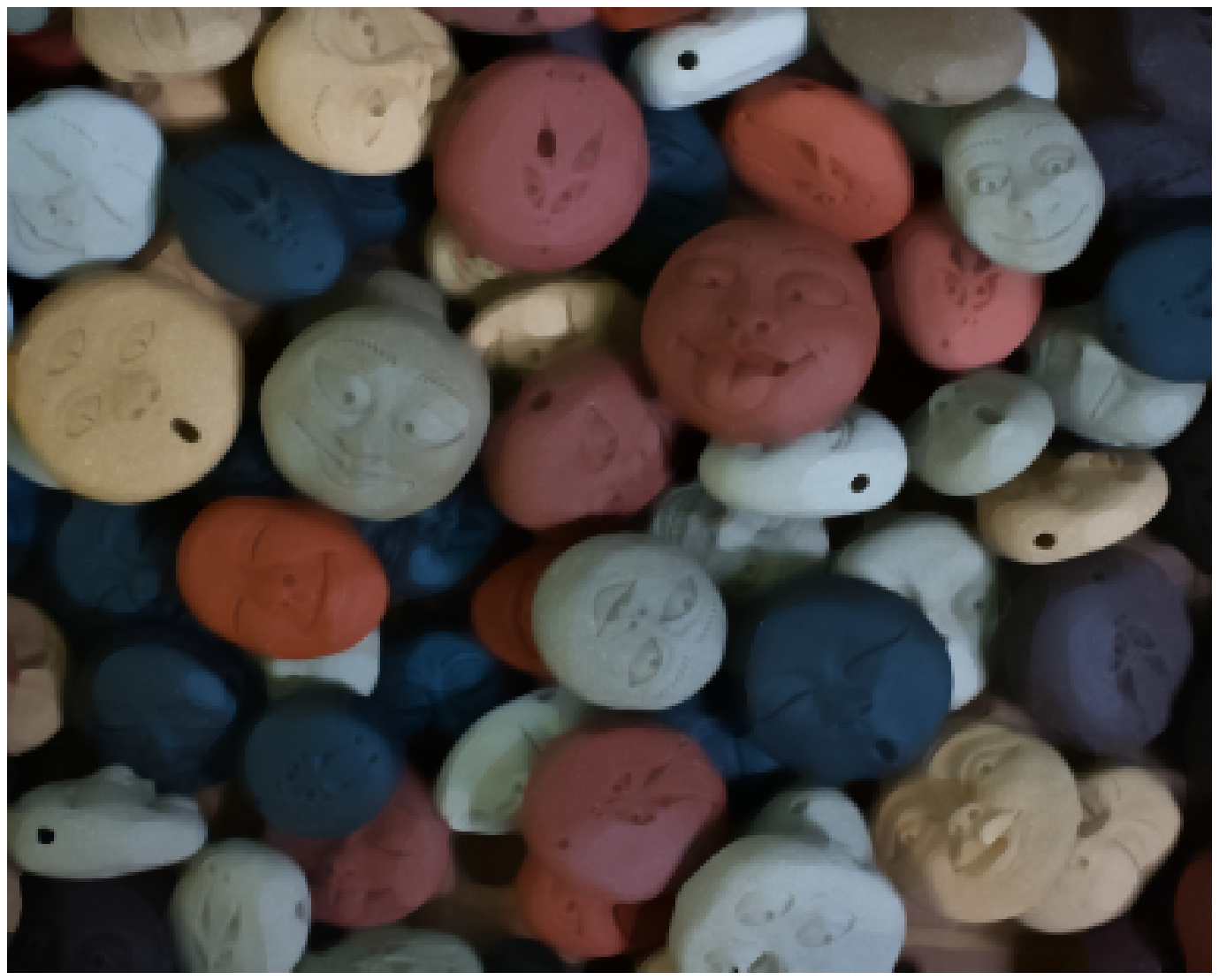}} &
\subfloat[Enhanced texture]{\includegraphics[width=1.4in]{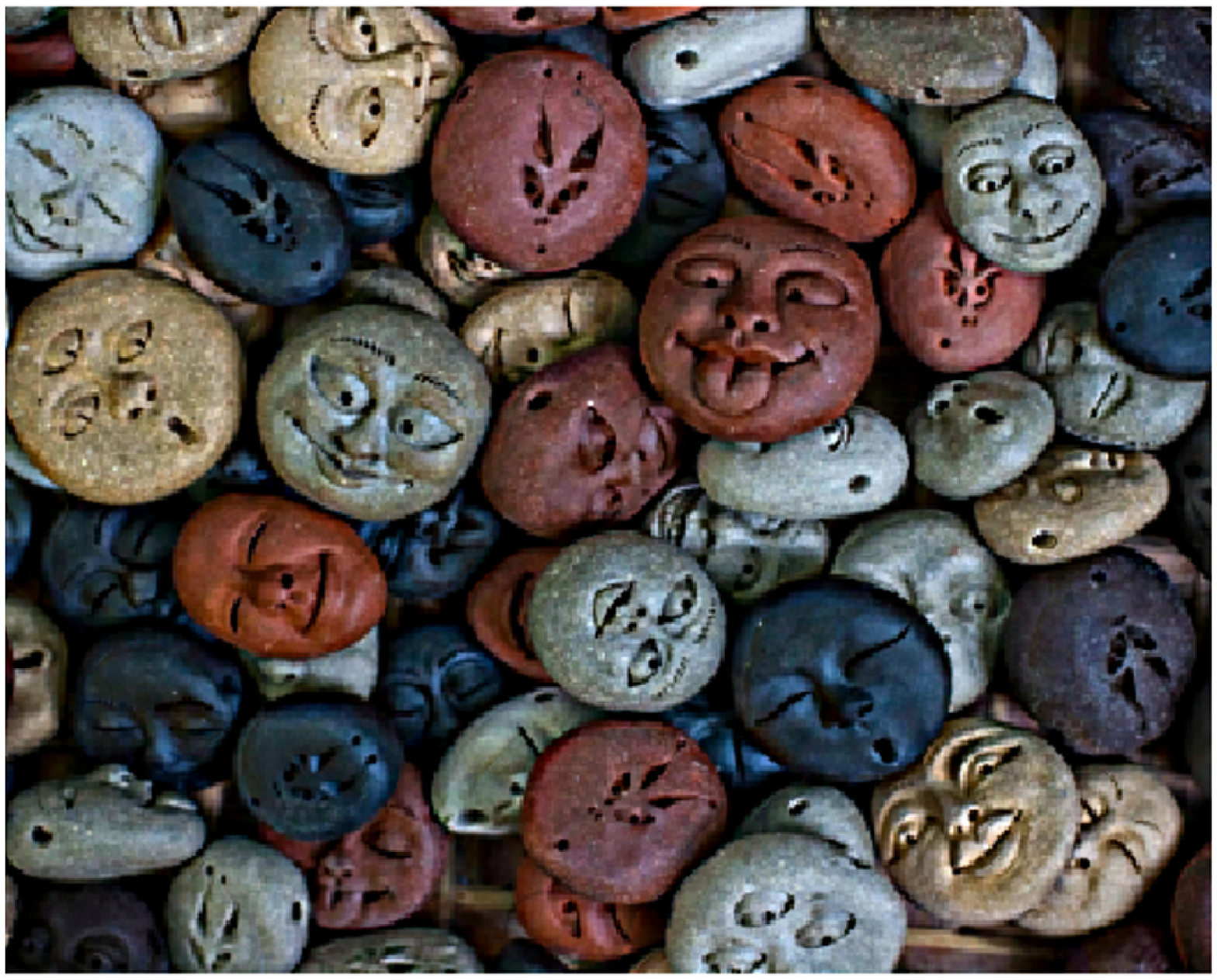}} \\
\end{tabular}
\caption {Example of manipulation of faces' features  in the image of Faces on stones}
\label{fig:faces_enhance}
\end{figure}

\begin{figure}
\centering
\begin{tabular}{cc}
\subfloat[Input Image]{\includegraphics[width=1.4in]{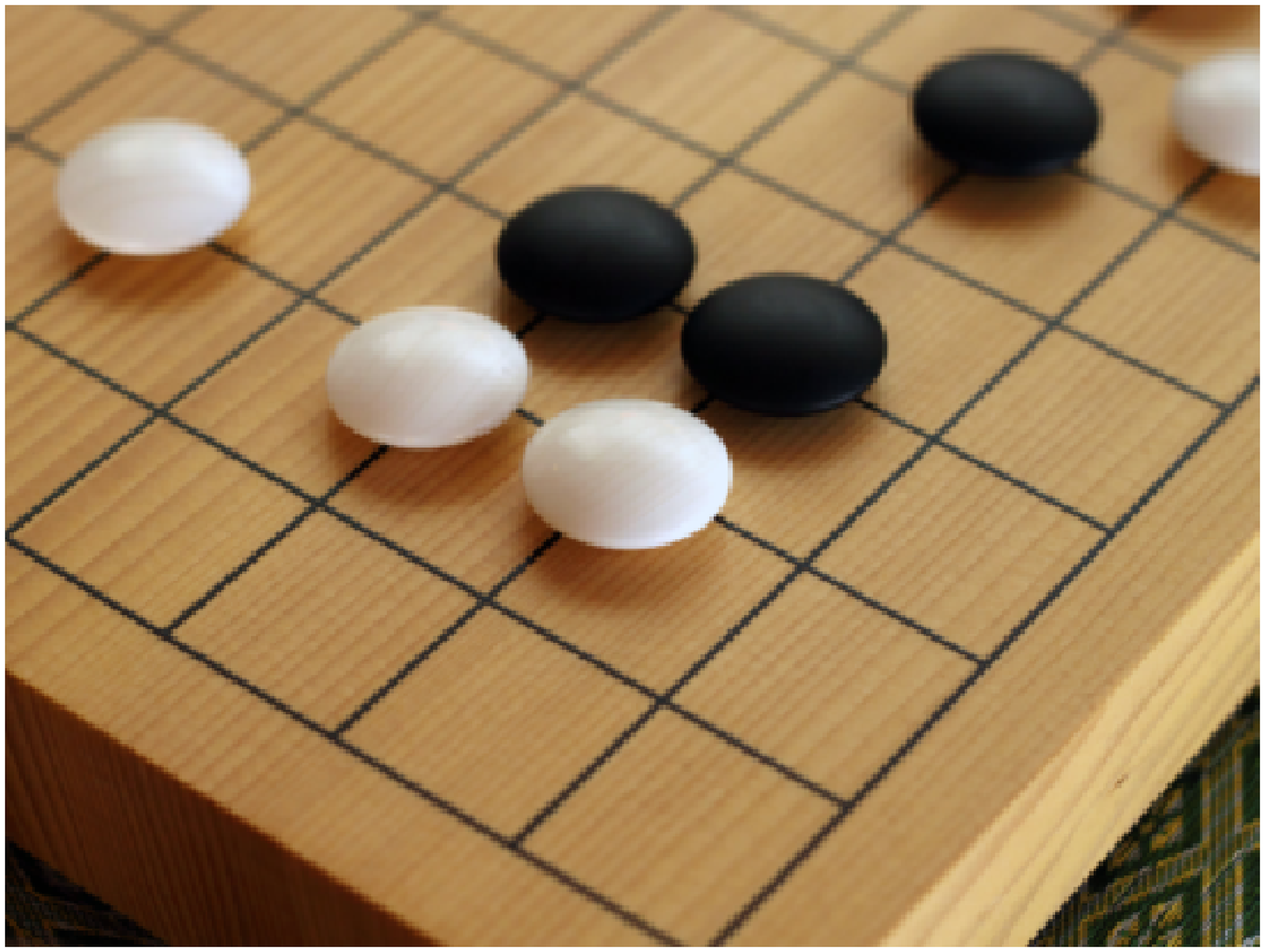}} &
\subfloat[Desired texture]{\includegraphics[width=1.4in]{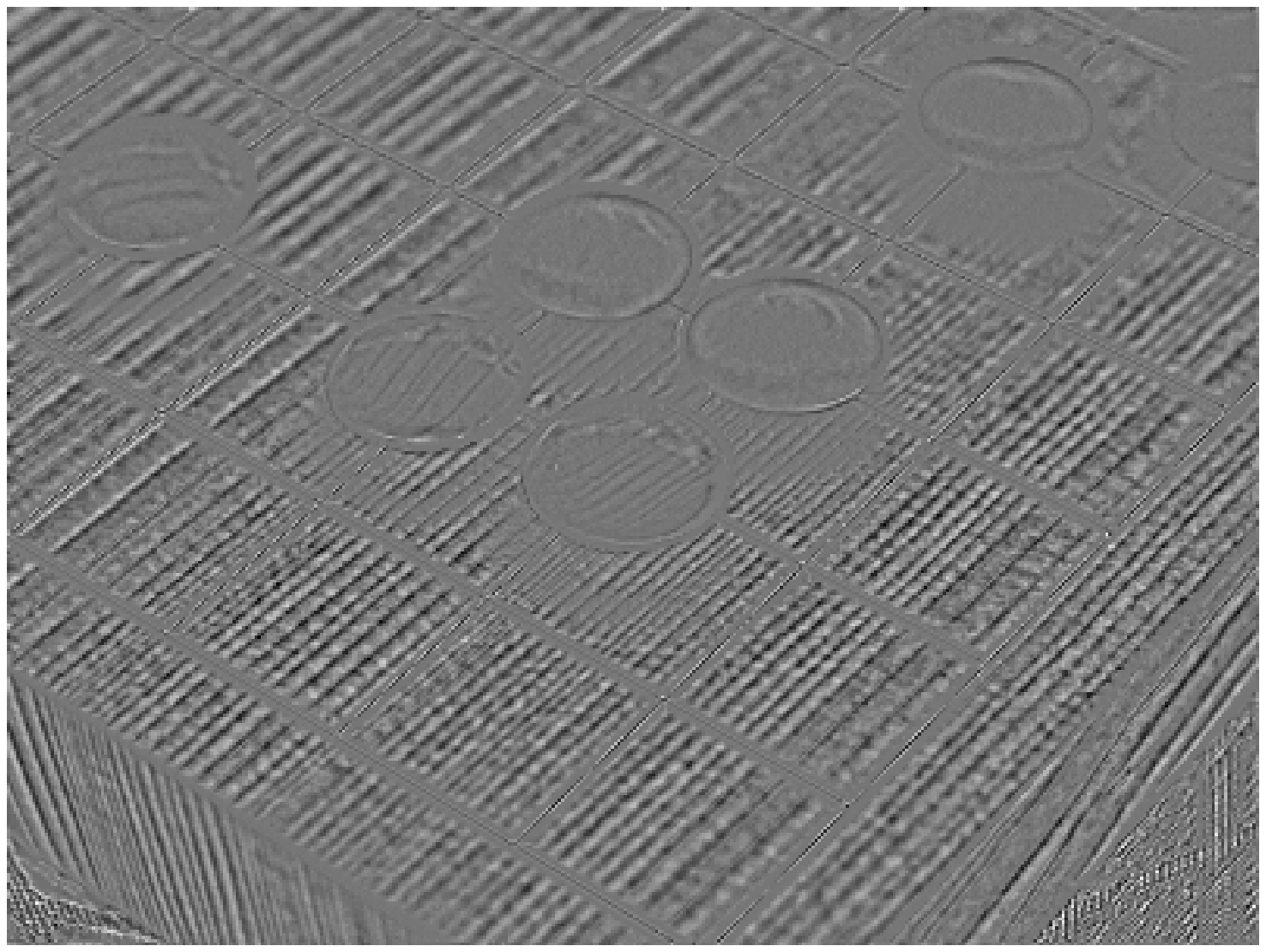}} \\
\subfloat[Attenuated texture]{\includegraphics[width=1.4in]{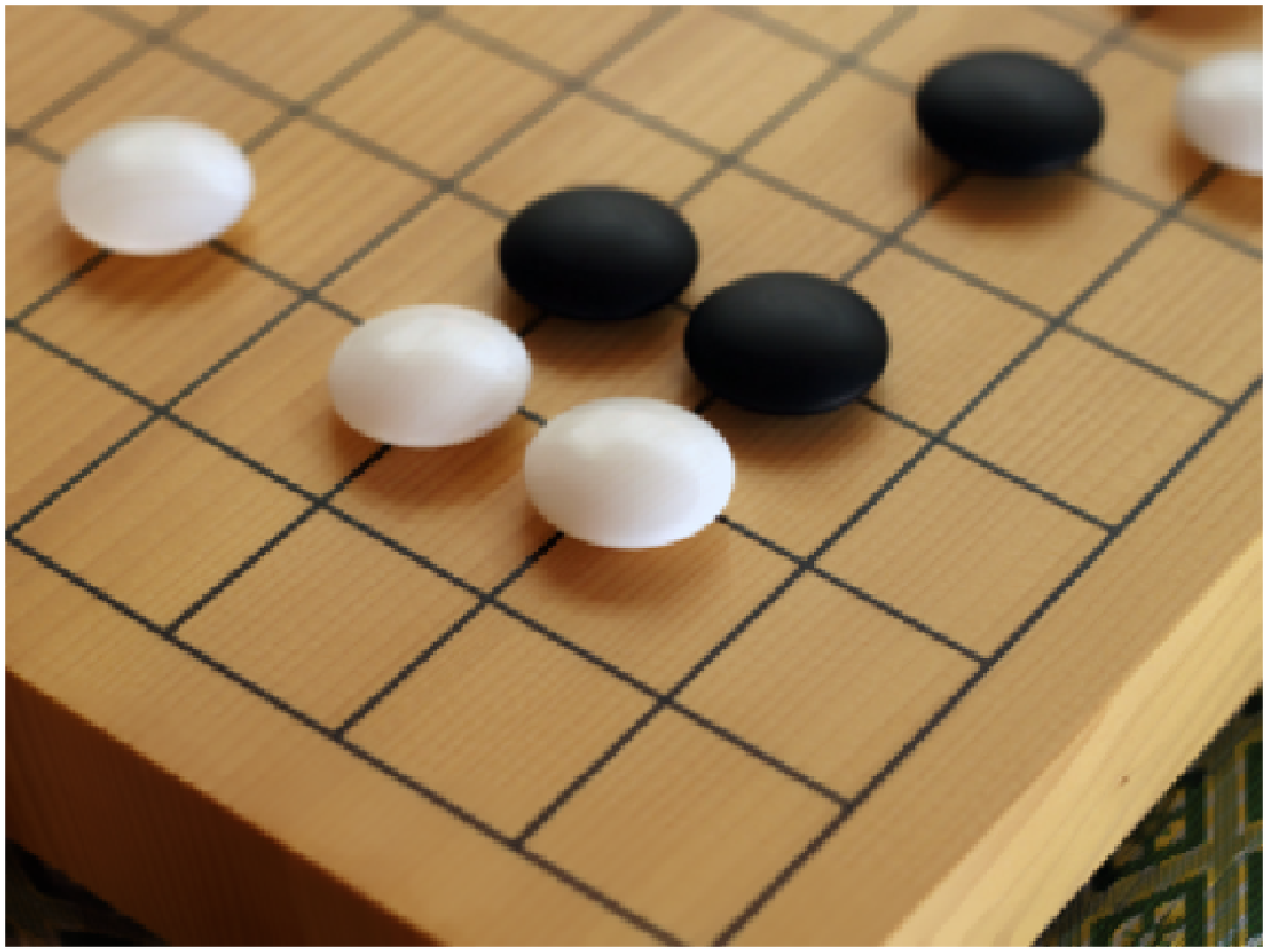}} &
\subfloat[Enhanced texture]{\includegraphics[width=1.4in]{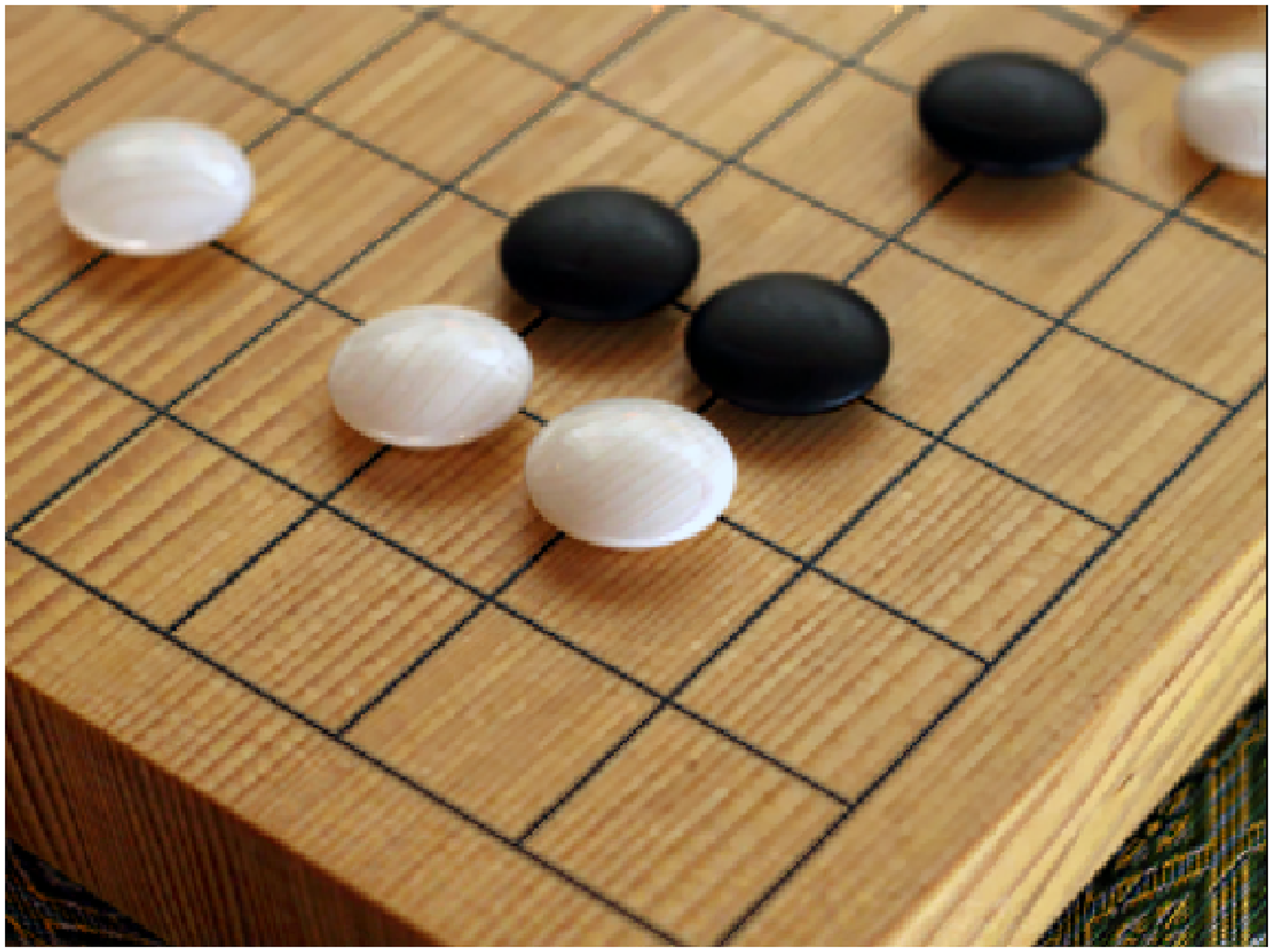}} \\
\end{tabular}
\caption { Example of manipulation of the wood texture in the game board image }
\label{fig:chess_enhance}
\end{figure}

\begin{figure}
\vspace{-10pt}
\centering
\begin{tabular}{ccc}
\subfloat[Input image]{\includegraphics[width=1in]{figs/zebra_in}} &
\subfloat[Enhanced stripes]{\includegraphics[width=1in]{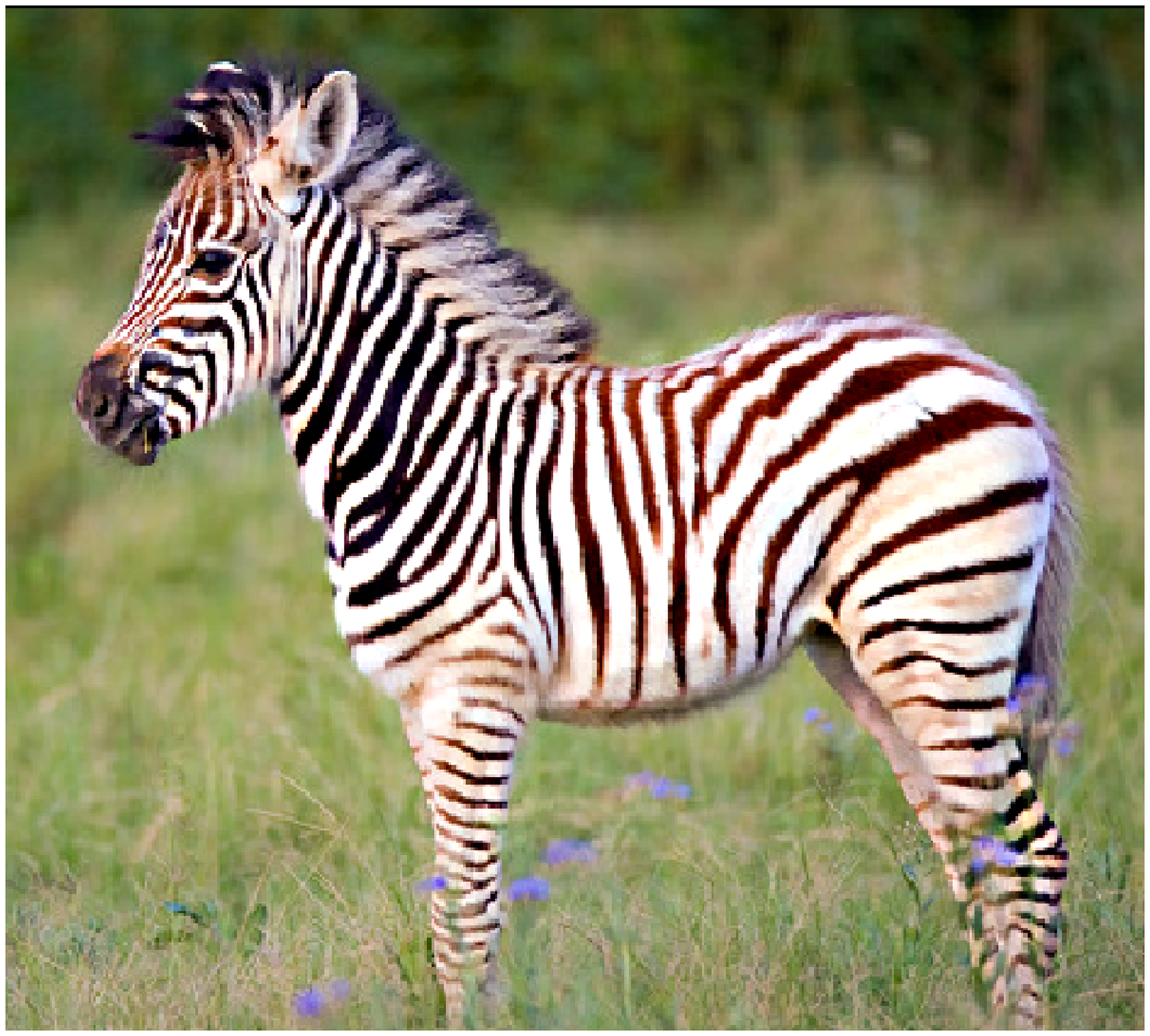}}  &
\subfloat[Inverted stripes]{\includegraphics[width=1in]{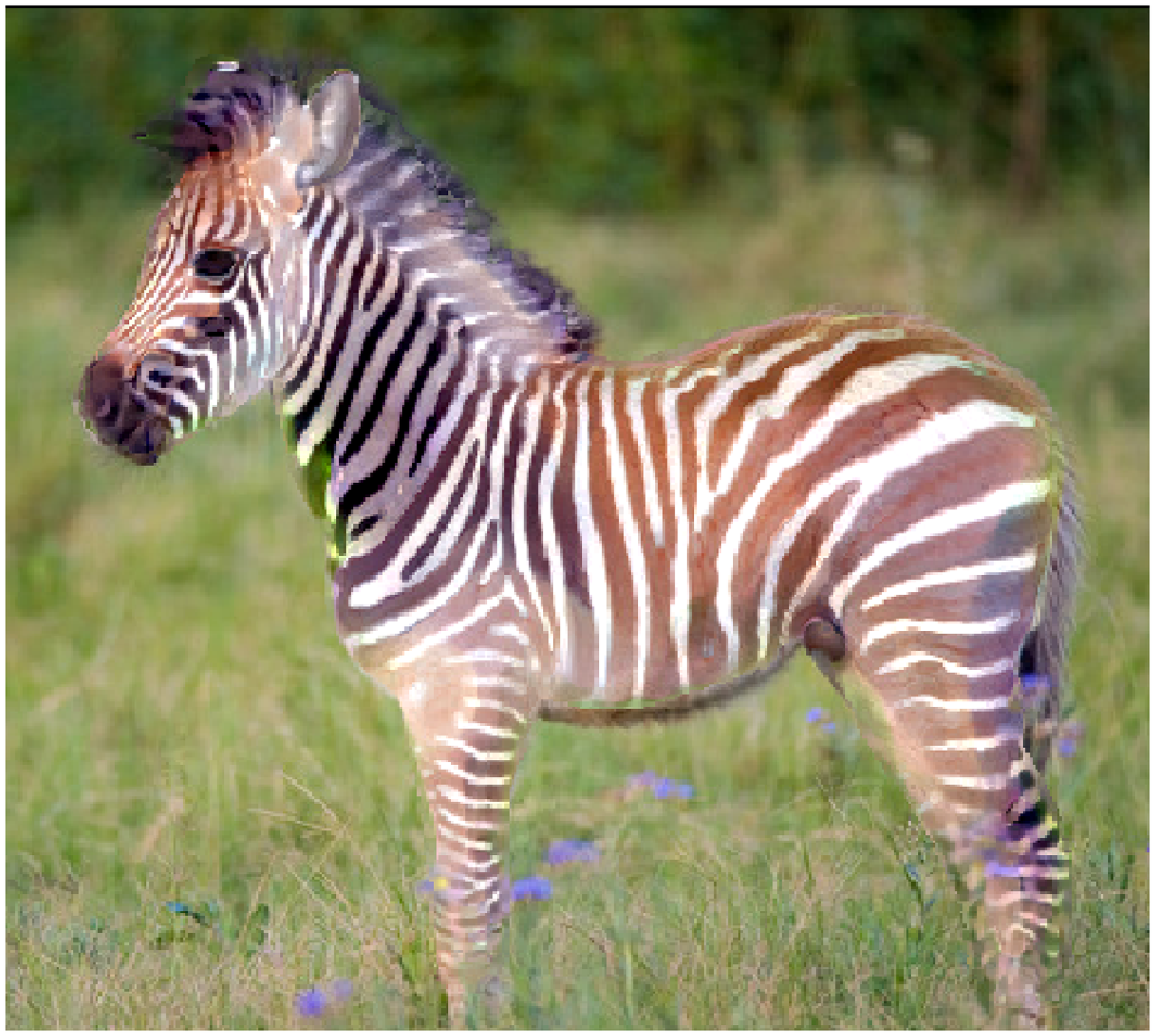}} \\
\end{tabular}
\caption {Stripes extraction and manipulation in the Zebra image.}
\label{fig:zebra_enhance}
\end{figure}

\subsection{Texture Donation}
Another application is texture donation, in which a degraded image is enhanced and its fine texture is recovered by matching a prior out of a set, using our scale-orientation descriptor, to get a perfect patch match and visually good recovery of the image.
In Fig. \ref{fig:Hair_donation} we can see a hair example, on the left image a degraded hair sample is shown, in the middle is the matching hair texture donor, selected out of many hair priors.  Its mirror image was taken for a texture match. On the right, the recovered hair image, composed of the degraded texture patch and decomposed fine scale texture of the donor patch. We can see the recovered patch inside the original hair image to see how natural it looks.

\begin{figure}
\vspace{-10pt}
\centering
\includegraphics[width=3.4in]{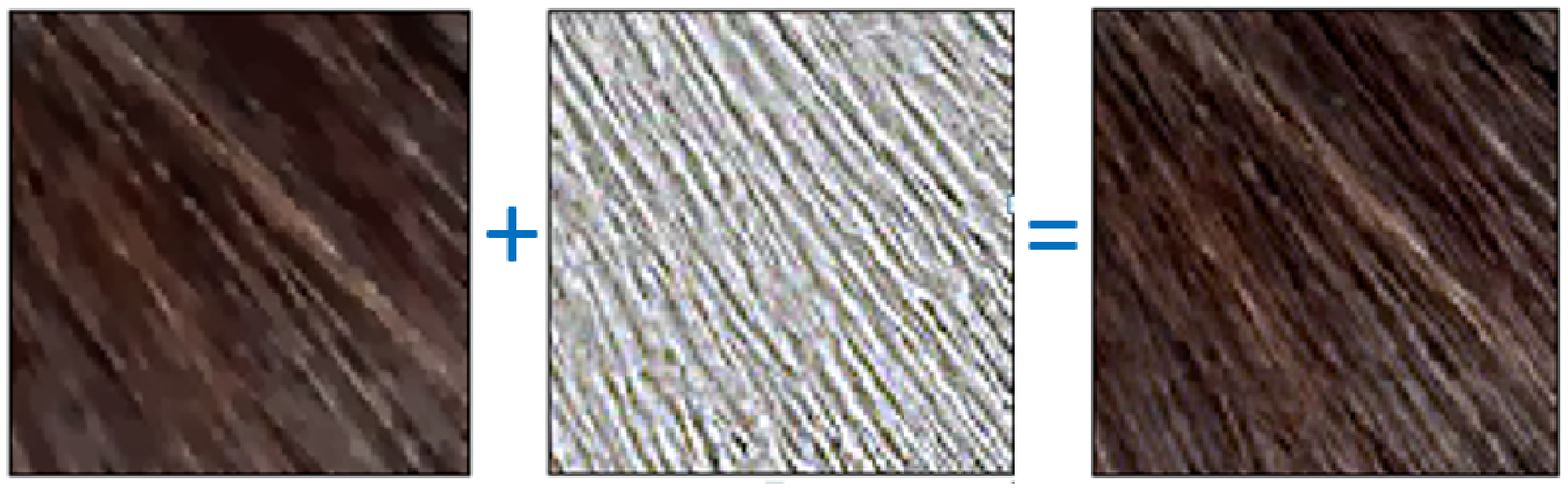}
\centering
\begin{tabular}{ccc}
\subfloat[Degraded]{\includegraphics[width=1in]{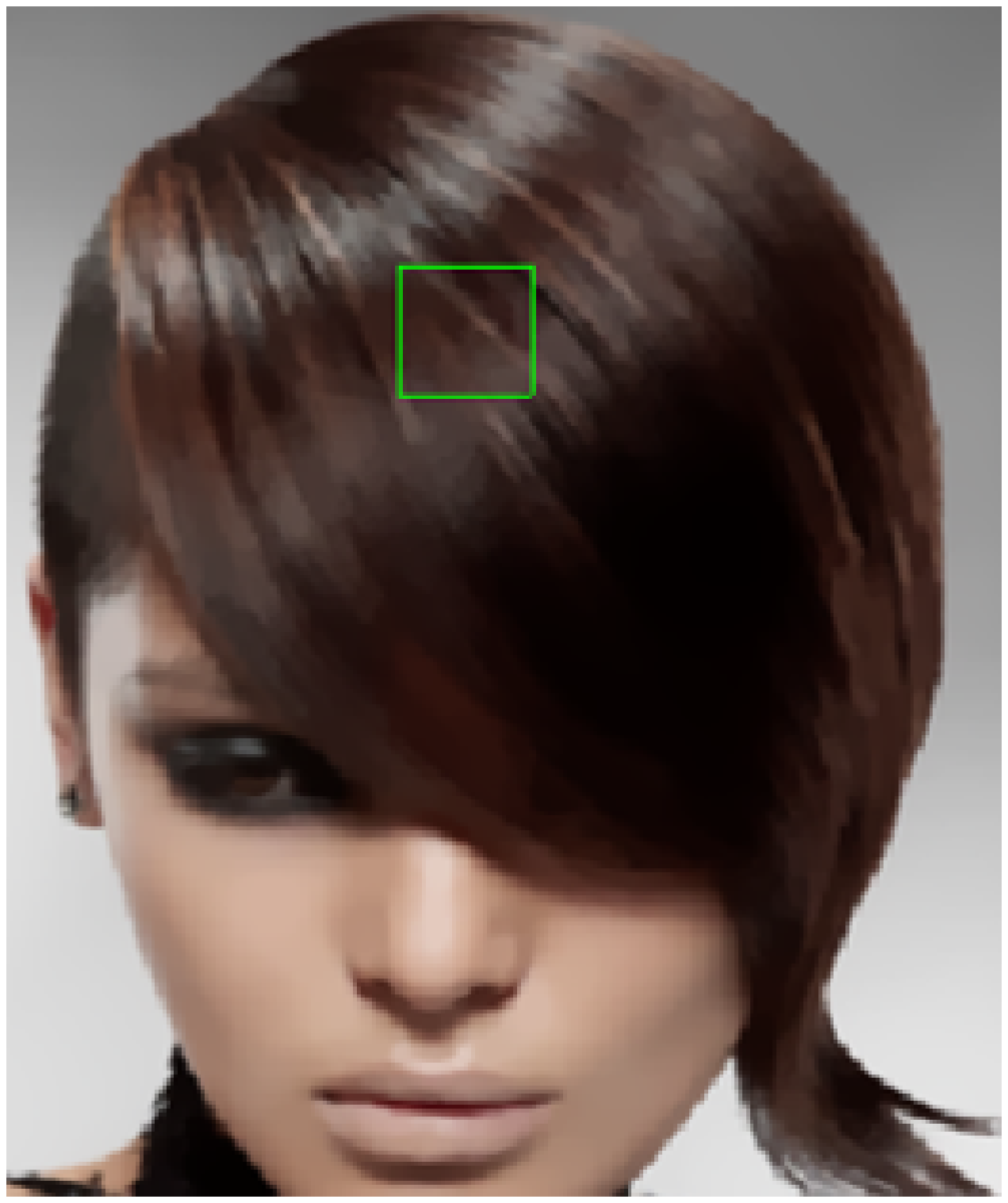}} &
\subfloat[Donor]{\includegraphics[width=1in]{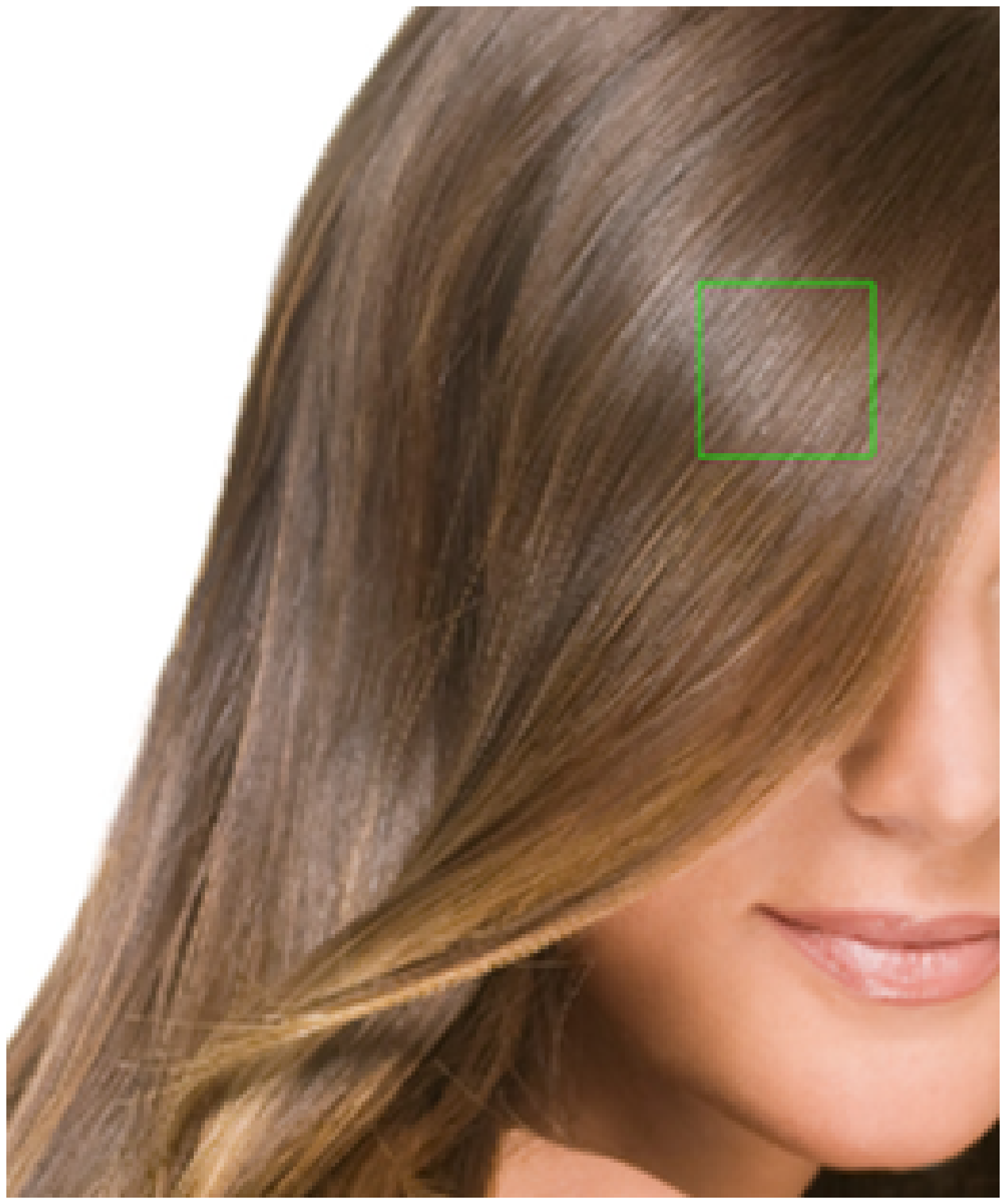}}  & 
\subfloat[Recovered]{\includegraphics[width=1in]{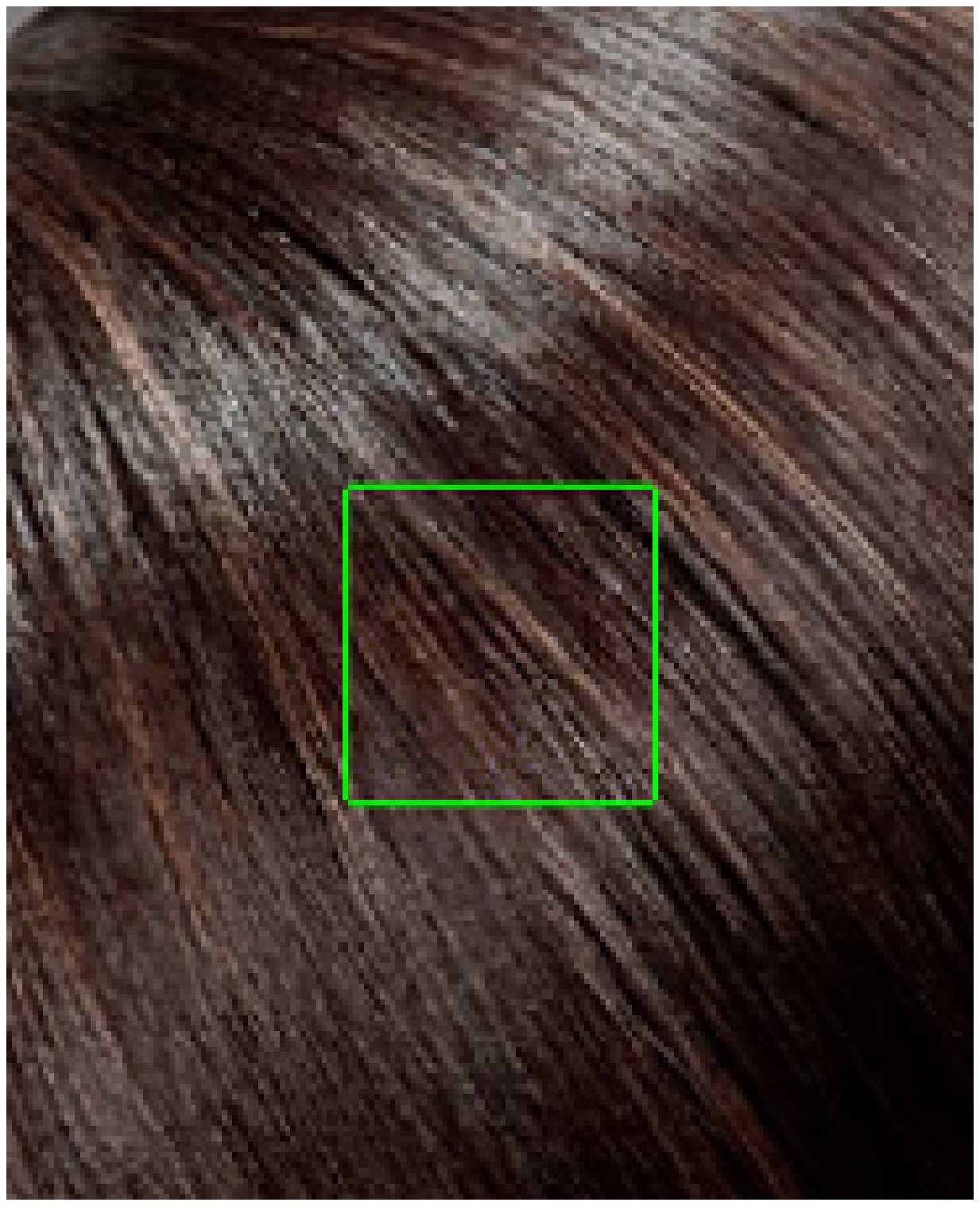}} \\
\end{tabular}
\caption {Hair texture donation example.}
\label{fig:Hair_donation}
\end{figure}
\vspace{-10pt}

\section{Conclusion}
\label{sec:conc}
A novel concept of separation surface was introduced, depicting a  stratum of a desired texture, not necessarily homogeneous in space and scale. The surface was found using regression of the maximal responses in the spectral TV domain. The surface can be fitted using methods other than maximal response, such as the Gabor filters for orientated textures, or the Gaussian mixture model for a mixture of textures with different distribution in space and scale. Image decomposition using a separation surface can be very beneficial in cases where the texture varies within the image, while preserving its characteristic. It can separate mixed textures in a highly accurate manner, compared to state-of-the-art methods. An application of texture manipulation was presented, in which the selected textures in the image can be attenuated, enhanced or even inverted, in a naturally looking formation. In a future work, we plan to examine additional methods to automatically form the surface. We would also like to broaden the use of
 multiscale decomposition to other applications.


%




%


\ifCLASSOPTIONcaptionsoff
  \newpage
\fi



\bibliographystyle{IEEEtranS}

\bibliography{./IEEE_bib}

%


%

\begin{IEEEbiography}{Dikla Horesh} received the B.Sc. degree in Biomedical engineering from the Technion, Israel, in 2006. She is currently working towards the M.Sc. degree  in Electrical engineering from the Technion, Israel.
\end{IEEEbiography}

\begin{IEEEbiography}{Guy Gilboa} is a faculty member at the Electrical Engineering Department, Technion, Israel, since 2013.
\end{IEEEbiography}







\end{document}